%% file: arxiv.tex
\documentclass[nohyperref]{article}

\usepackage{microtype}
\usepackage{graphicx}
\usepackage{subfigure}
\usepackage{booktabs} 

\usepackage{hyperref}

\usepackage[accepted]{icml2023}

\usepackage{amsmath}
\usepackage{amssymb}
\usepackage{mathtools}
\usepackage{amsthm}

\usepackage[capitalize,noabbrev]{cleveref}

\theoremstyle{plain}
\newtheorem{theorem}{Theorem}[section]

\newtheorem{lemma}[theorem]{Lemma}

\theoremstyle{definition}
\newtheorem{definition}[theorem]{Definition}
\newtheorem{assumption}[theorem]{Assumption}
\theoremstyle{remark}

\input{math.tex}

\usepackage{color}
\usepackage{comment}
\DeclareMathOperator{\asto}{\xrightarrow{\text{a.s.}}}
\DeclareMathOperator\supp{supp}
\usepackage{cancel}

\usepackage{tikz}
\usepackage{pgfplots}
\usetikzlibrary{matrix}
\usepgfplotslibrary{groupplots}
\pgfplotsset{compat=newest}
\pgfplotsset{width=7.5cm,compat=1.12}
\usepgfplotslibrary{fillbetween}

\usepackage[textsize=tiny]{todonotes}

\icmltitlerunning{Optimizing Orthogonalized Tensor Deflation via Random Tensor Theory}

\begin{document}

\onecolumn
\icmltitle{Optimizing Orthogonalized Tensor Deflation via Random Tensor Theory}




\begin{icmlauthorlist}
\icmlauthor{Mohamed El Amine Seddik}{tii}
\icmlauthor{Mohammed Mahfoud}{tum}
\icmlauthor{Merouane Debbah}{tii}
\end{icmlauthorlist}

\icmlaffiliation{tii}{Technology Innovation Institute, Abu Dhabi, United Arab Emirates}
\icmlaffiliation{tum}{Technical University of Munich, Munich, Germany}

\icmlcorrespondingauthor{M. Seddik}{mohamed.seddik@tii.ae}
\icmlcorrespondingauthor{M. Mahfoud}{mo.mahfoud@tum.de}
\icmlcorrespondingauthor{M. Debbah}{merouane.debbah@tii.ae}

\icmlkeywords{Machine Learning, ICML}

\vskip 0.3in



\printAffiliationsAndNotice{} 

\begin{abstract}
This paper tackles the problem of recovering a low-rank signal tensor with possibly correlated components from a random noisy tensor, or so-called \textit{spiked tensor model}. 
When the underlying components are orthogonal, they can be recovered efficiently using \textit{tensor deflation} which consists in successive rank-one approximations, while non-orthogonal components may alter the tensor deflation mechanism, thereby preventing efficient recovery. 
Relying on recently developed random tensor tools, this paper deals precisely with the non-orthogonal case by deriving an asymptotic analysis of a \textit{parameterized} deflation procedure performed on an order-three and rank-two spiked tensor. 
Based on this analysis, an efficient tensor deflation algorithm is proposed by optimizing the parameter introduced in the deflation mechanism, which in turn is proven to be optimal by construction for the studied tensor model.
The same ideas could be extended to more general low-rank tensor models, e.g., higher ranks and orders, leading to more efficient tensor methods with broader impact on machine learning and beyond.
\end{abstract}

\section{Introduction}
Tensor methods have been proven to be a powerful and versatile tool in machine learning in both providing a rich framework for theoretical analysis and motivating the reformulation of existing problems in higher dimensions \cite{rabanser2017introduction, sidiropoulos2017tensor}, which often results in the ability to develop better performing algorithms or massively accelerate existing ones \cite{fawzi2022discovering}. One of the most fundamental problems in machine learning is retrieving low-rank structures from high-dimensional data, in which tensor methods are particularly successful, e.g., learning Gaussian mixtures in an unsupervised setting \cite{anandkumar2014tensor}, which can be seen as a natural generalization of the standard principal component analysis (PCA) to higher order tensors \cite{zare2018extension}. Another surprising area where tensors methods have been shown to be efficient, is the reconstruction of training samples from a single gradient query of a neural network in a federated learning context \cite{wang2022reconstructing}.
\begin{figure}[t!]
    \centering
    \input{figs/intro_alignments.tex}
    \vspace{-.4cm}
    \caption{Orthogonalized deflation (see §\ref{sec_tensor_deflation} for details) applied on the rank-two tensor $\sum_{i=1}^2 \beta_i x_i^{\otimes 3}$ which yields the signal estimates $u_1$ and $u_2$, at first and second deflation steps respectively. The signals are successfully recovered in the orthogonal case $\langle x_1, x_2 \rangle=0$ (left), while the estimation is altered if the signal components are correlated $\langle x_1, x_2 \rangle=0.5$ (right). $\beta_1$ is fixed while varying $\beta_2$.} 
    \label{fig_illustration}
\end{figure}
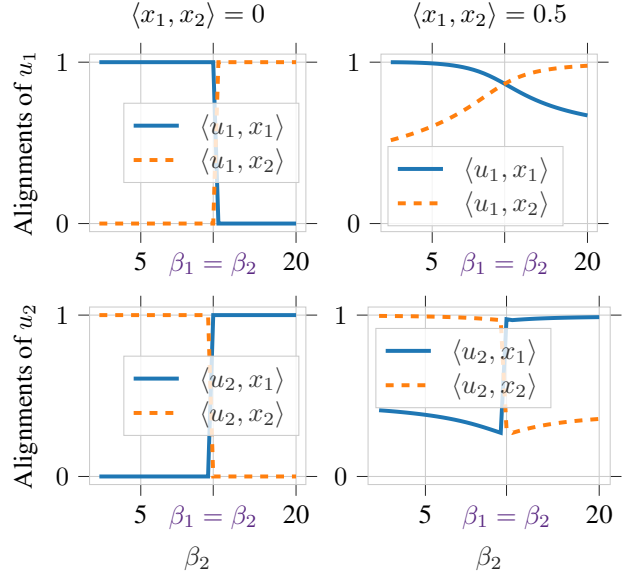

As a first step towards understanding the behavior of tensor methods, \citet{montanari2014statistical} introduced the concept of tensor PCA by studying the so-called \textit{spiked tensor model} of the form $\beta x^{\otimes d} + \gW/\sqrt p$ where $x\in \sR^p$ is a high-dimensional vector of unit norm which represents the (rank-one) signal of interest, $\gW$ is a symmetric Gaussian random noise tensor of order $d$, and $\beta \geq 0$ is a parameter controlling the signal strength. This statistical model raises many fundamental questions which mainly concern the theoretical and algorithmic guarantees that ensure the efficient recovery of the hidden signal $x$. 

A flurry of works focused on addressing these questions \cite{perry2016statistical, lesieur2017statistical, jagannath2020statistical, chen2021phase, goulart2021random, auddy2022estimating, arous2021long}. The first main result was for tensors of order $d\geq3$, for which it was shown that there exists a statistical threshold $\beta_{stat}$ of $O(1)$ in the tensor dimension, which defines the information-theoretic limit above which signal recovery is possible, with the maximum likelihood estimator (MLE), and below which signal recovery is impossible.


While recovery above $\beta_{stat}$ is theoretically possible from an information-theoretic standpoint, solving the underlying MLE problem remains NP-hard in the worst case \cite{hillar2013most}. Indeed, \citet{montanari2014statistical} suggests through heuristics that there exists an algorithmic threshold $\beta_{algo} = O(p^{\frac{d-2}{4}})$ beyond which recovery is possible with a polynomial time algorithm, therefore implying the existence of a theoretical-algorithmic spectral gap where no polynomial time algorithm has been proven to be efficient in signal recovery. These suggestions were rigorously proven by \cite{lesieur2017statistical, jagannath2020statistical, chen2021phase, huang2022power} and generalized to non-symmetric tensors by \cite{arous2021long, seddik2021random, auddy2022estimating}. 

From a practical standpoint, to be able to unleash the full potential of tensor methods, higher (beyond rank-one) low-rank signal reconstruction problems need to be considered, thus motivating the study of \textit{low-rank spiked tensor models}. In particular, and in a more realistic scenario, one would be interested in extracting low-rank hidden structures from random noise, for which the model can naturally be extended to $\sum_{i=1}^{r} \beta_i x_i^{\otimes d} + \gW/\sqrt p$ where $r$ is the rank of the signal of interest.
In this line of work, \citet{chen2021phase} proves that the asymptotic behavior of a low-rank spiked tensor model with orthogonal signal components, i.e., $\langle x_i, x_j \rangle = 0$ for $i\neq j \in [r]$, can be understood from the analysis of a rank-one model. Moreover, \citet{da2015rank, da2015iterative} show that estimating a higher rank signal boils down to performing successive rank-one approximations using iterative tensor deflation. While the latter result provides a more tractable approach to low-rank signal recovery in the orthogonal case, it may fail in signal reconstruction in the non-orthogonal case \cite{seddikaccuracy}. 

Other enhanced deflation techniques rely on orthogonal projections \cite{mackey2008deflation} while exhibiting the same alteration as the standard deflation in the non-orthogonal case, as illustrated in Figure \ref{fig_illustration}. The latter depicts signal recovery, in terms of alignments, from a rank-two tensor $\sum_{i=1}^2 \beta_i x_i^{\otimes 3}$ using an orthogonalized deflation \cite{mackey2008deflation}. It is clearly observed that, when $\beta_1 \approx \beta_2$, the non-orthogonality of $x_1$ and $x_2$ prevents efficient recovery. We highlight the fact that measuring alignments is a concrete performance measure of the recovery quality. Indeed, given that in high dimension, the probability that two random vectors $u$, $v$ are orthogonal, i.e., $\langle u, v\rangle = 0$, showcases the difficulty of obtaining high estimation accuracies in this setting.

\paragraph{Key contributions:} Aiming to understand the interplay of the orthogonalized tensor deflation and improve its efficiency, our key contributions can be summarized as follows:
\begin{enumerate}
    \item We consider a slightly different orhogonalized tensor deflation algorithm by introducing a parameter $\gamma$ as described in §\ref{sec_tensor_deflation}. In particular, $\gamma=1$ corresponds to the classical orhogonalized deflation \cite{mackey2008deflation}.
    \item We carry out a random tensor theory (RTT) analysis of the considered deflation method applied on a \textit{rank-two asymmetric spiked tensor model} defined in §\ref{sec_spiked_model}. Spiked models are more general and offer many theoretical advantages then considering noiseless low-rank models. For instance, subtracting a best rank-one approximation of a noiseless low-rank tensor may increase its rank \cite{stegeman2010subtracting}, while spiked models do not suffer from such limitation since their noise components is full rank \cite{strassen1983rank}.
    \item Based on our theoretical analysis, we optimize the parameter $\gamma$ introduced in the deflation mechanism which allows us to design a more efficient tensor deflation algorithm (see §\ref{sec_improved}).
\end{enumerate}

\section{Notations and Background}
The set $\{1, \ldots, n\}$ is denoted by $[n]$. The unit sphere in $\sR^p$ is denoted by $\sS^{p-1}$. The Dirac measure at some real value $x$ is denoted by $\delta_x$. The support of a measure $\mu$ is denoted by $\supp(\mu)$. The inner-produce between two vectors $u,v$ is denoted by $\langle u, v \rangle = \sum_i u_i v_i$. The imaginary part of a complex number $z$ is denoted by $\Im[z]$. The set of eigenvalues of a matrix $\mM$ is denoted by $\spec(\mM)$. The almost sure converges is denoted by the arrow $\asto$. The notation $a_n\asymp b_n$ means that $a_n$ and $b_n$ converge to the same limit as $n\to \infty$.

\subsection{Tensor Notations and Contractions}\label{sec_tensor_notations}
In this section, we provide the main tensor notations and definitions used throughout the paper, which we recommend to follow carefully for better understanding of our paper. 

\textbf{Three-order tensors:} The set of three-order tensors of size $p$ is denoted $\sR^{p\times p \times p}$. The scalar $T_{ijk}$ or $[\gT]_{ijk}$ denotes the $(i,j,k)$ entry of a tensor $\gT\in \sR^{p\times p \times p}$. In the remainder, we will mainly consider tensors from $\sR^{p\times p \times p}$, and for brevity, we may omit the notation $\gT \in \sR^{p\times p \times p}$.

\textbf{Rank-$r$ tensors:} A tensor $\gT$ is said to be of rank-one if it can be represented as the outer product of three real-valued vectors $x,y,z\in \sR^p$. In this case, we write $\gT = x\otimes y \otimes z$, where the outer product is defined such that $[x\otimes y \otimes z]_{ijk} = x_i y_j z_k$. More generally, a tensor $\gT$ is of rank-$r$, for some integer $r$, if it can be expressed as the sum of $r$ rank-one terms, written as $\gT = \sum_{i=1}^r x_i \otimes y_i \otimes z_i$, where $x_i, y_i, z_i\in \sR^p$ for all $i\in [r]$. To maintain consistency, we will adhere to the convention of using $x_i$ or $u_i$ to represent the components of the first mode, $y_i$ or $v_i$ to represent the components of the second mode, and $z_i$ or $w_i$ to represent the components of the third mode throughout the paper.

\textbf{Tensor contractions:} The first mode contraction of a tensor $\gT$ with a vector $x$ results in a matrix denoted $\gT(x,\cdot, \cdot)$ with entries $[\gT(x,\cdot, \cdot)]_{jk} = \sum_{i=1}^p x_i T_{ijk}$. Similarly, $\gT(\cdot, y, \cdot)$ and $\gT(\cdot, \cdot, z)$ denote the second and third mode contractions of $\gT$ with vectors $y$ and $z$ respectively. We will sometimes denote these contractions by $\gT(x)$, $\gT(y)$ and $\gT(z)$ if there is no ambiguity. 
The contraction of $\gT$ on two vectors $x,y$ is a vector denoted $\gT(x,y,\cdot)$ with entries $[\gT(x,y,\cdot)]_k = \sum_{ij} x_i y_j T_{ijk}$. The contraction of $\gT$ on three vectors $x,y,z$ is a scalar denoted $\gT(x,y,z) = \sum_{ijk} x_i y_j z_k T_{ijk}$. 
The first mode contraction of $\gT$ with a matrix $\mM\in \sR^{p\times p}$ results in a tensor denoted $\gT\times_1 \mM$ with entries $[\gT\times_1 \mM]_{ijk} = \sum_{i'=1}^p M_{ii'}T_{i'jk}$. Similarly, $\gT\times_2\mN$ and $\gT\times_3\mP$ denote the second and third modes tensor-matrix contractions of the tensor $\gT$ with the matrices $\mN$ and $\mP$ respectively. The notation $u\otimes \mM$ stands for the tensor with entries $u_i M_{jk}$.

\textbf{Tensor norms:} The Frobenius norm of a tensor $\gT$ is denoted $\Vert \gT\Vert_F$ with $\Vert \gT\Vert_F^2 = \sum_{ijk} T_{ijk}^2$. The spectral norm of $\gT$ is $\Vert \gT \Vert = \sup_{u,v,w\in \sS^{p-1}} \vert \gT(u,v,w) \vert $.

\textbf{Best rank-one approximation and tensor power iteration:} A best rank-one approximation of $\gT$ corresponds to a rank-one tensor $\lambda u\otimes v \otimes w$, where $\lambda>0$ and $u,v,w$ are unitary vectors, that minimizes the square loss $\Vert \gT - \lambda u\otimes v \otimes w\Vert_F^2$. The latter generalizes to tensors the concept of singular value and vectors \cite{lim2005singular} and the scalar $\lambda$ coincides with the spectral norm of $\gT$. In particular, the quadruple $(\lambda , u, v, w)$ satisfies the following identities 
\begin{equation}\label{eq_singular_value_vectors_identities}
    \gT(\cdot, v, w) = \lambda u, \quad \gT( u, \cdot, w) = \lambda v,\quad
    \gT( u,v, \cdot) =\lambda w, \quad \lambda = \gT(u, v, w).
\end{equation}

Such a best rank-one approximation can be computed via \textit{tensor power iteration} which consists in iterating
\begin{align*}
    u \leftarrow \frac{\gT(\cdot, v, w)}{ \Vert \gT(\cdot, v, w)\Vert }\quad
    v \leftarrow \frac{\gT(u, \cdot, w)}{ \Vert \gT(u, \cdot, w)\Vert }\quad
    w \leftarrow \frac{\gT(u, v, \cdot)}{ \Vert \gT(u, v, \cdot)\Vert }
\end{align*}
starting from some initialization \cite{anandkumar2014tensor}.

\subsection{Random Matrix Theory}
In this section, we provide some necessary tools from random matrix theory (RMT) which are at the core of our main results. Specifically, we will consider the \textit{resolvent} formalism \cite{hachem2007deterministic} which allows one to characterize the spectral behavior of large symmetric random matrices. Given a symmetric matrix $\mS\in \sR^{n\times n}$, the resolvent of $\mS$ is defined as $\mR(z) = \left( \mS - z \mI_n \right)^{-1}$ for $z\in \sC\setminus \spec(\mS)$.

In essence, RMT focuses on describing the distribution of eigenvalues of large random matrices. Typically, under certain technical assumptions on some random matrix $\mS\in \sR^{n\times n}$ with eigenvalues $\lambda_1, \ldots, \lambda_n$, the \textit{empirical spectral measure} of $\mS$, defined as $\hat \mu = \frac1n \sum_{i=1}^n \delta_{\lambda_i}$, converges in the weak sense \cite{van1996weak} to some deterministic probability measure $\mu$ as $n\to \infty$ and RMT aims at describing such $\mu$. To this end, one of the widely considered approaches relies on the \textit{Stieltjes transform} \cite{tao2012topics}. Given a probability measure $\mu$, the Stieltjes transform of $\mu$ is defined as $g_\mu(z) = \int \frac{d\mu(\lambda)}{\lambda - z}$ with $z\in \sC\setminus \supp(\mu)$, and the inverse formula allows one to describe the density function of $\mu$ as $\mu(dx) = \frac{1}{\pi} \lim_{\varepsilon\to 0} \Im[g_\mu(x+i\varepsilon)]$. 

The Stieltjes transform of the empirical spectral measure, $\hat\mu$, is closely related to the resolvent of $\mS$ through the normalized trace operator. In fact, $g_{\hat\mu}(z) = \frac1n \Tr \mR(z)$ and the \textit{almost sure} convergence of $g_{\hat\mu}(z)$ to some deterministic Stieltjes transform $g(z)$ is equivalent to the weak convergence between the underlying probability measures \cite{tao2012topics}. Our analysis relies on estimating quantities involving $\frac1n \Tr \mR(z)$, making the use of the resolvent approach a natural choice.

\section{Model \& Main Results}
\subsection{Rank-two Spiked Tensor Model}\label{sec_spiked_model}
We consider the following rank-two spiked tensor model
\begin{align}\label{eq_first_deflation_model}
    \gT_1 \equiv \gS + \frac{1}{\sqrt n} \gW \in \sR^{p\times p \times p},
\end{align}
where $\gS = \sum_{i=1}^2 \beta_i x_i \otimes y_i \otimes z_i$, $\beta_i\geq 0$ correspond to the signal-to-noise ratios (SNRs), $x_i, y_i, z_i \in \sS^{p-1}$ are the signal components, $\gW$ is a random tensor with standard Gaussian i.i.d.\@ entries, i.e., $W_{ijk} \sim \gN(0, 1)$, and $n = 3p$.

We further consider that the between signal components alignments are uniform across the modes, i.e.
\begin{align}\label{assum_alpha}
    \alpha \equiv \langle x_1, x_2 \rangle = \langle y_1, y_2 \rangle = \langle z_1, z_2 \rangle.
\end{align}
The parameter $\alpha$ will therefore control the correlation between the rank-two terms. Specifically, $\alpha=0$ corresponds to the orthogonal case while $\alpha > 0$ models the correlated case. In the following we denote $\alpha_{ij} = \alpha$ if $i\neq j$ and $1$ otherwise. Our results can be easily extended to higher ranks and high order tensors but we consider the above model and restricted assumption in \eqref{assum_alpha} for the sake of simplicity.

\subsection{RTT Analysis of Orthogonalized Tensor Deflation}\label{sec_tensor_deflation}
In order to recover the signal components, we first consider a best rank-one approximation $\hat\lambda_1 \hat u_1 \otimes \hat v_1 \otimes\hat w_1$ of the tensor $\gT_1$ as a \textit{first deflation step}.
Given the vector $\hat u_1\in \sS^{p-1}$, the \textit{second deflation step} consists in performing a best rank-one approximation $\hat \lambda_2 \hat u_2 \otimes \hat v_2 \otimes\hat w_2$ of the following tensor
\begin{align}\label{eq_model_second_deflation}
    \gT_2 \equiv \gT_1 \times_1 \left( \mI_p - \gamma \hat u_1 \hat u_1^\top \right) = \gT_1 - \gamma \hat u_1 \otimes \gT_1(\hat u_1),
\end{align}
for some parameter $\gamma \in [0, 1]$ which we will optimize based on our theoretical analysis. In particular, when $\gamma=1$, the tensor $\gT_2$ is obtained as the orthogonal projection of the first mode of the tensor $\gT_1$ on the hyperplane defined by the plane normal vector $\hat u_1$, which corresponds to the classical \textit{orthogonalized deflation} \cite{mackey2008deflation}.

Moreover, as we recalled in \eqref{eq_singular_value_vectors_identities}, the best rank-one approximations $\hat\lambda_1\hat u_1 \otimes\hat v_1 \otimes\hat w_1$ and $\hat\lambda_2\hat u_2 \otimes\hat v_2 \otimes\hat w_2$ of $\gT_1$ and $\gT_2$ respectively satisfy the following identities, for $i\in [2]$
\begin{equation}\label{eq_singular_values_vectors}
    \gT_i(\cdot,\hat v_i,\hat w_i) =\hat \lambda_i\hat u_i, \quad \gT_i(\hat u_i, \cdot,\hat w_i) =\hat \lambda_i\hat v_i,\quad
    \gT_i(\hat u_i,\hat v_i, \cdot) =\hat \lambda_i\hat w_i, \quad \hat\lambda_i = \gT_i(\hat u_i,\hat v_i,\hat w_i).
\end{equation}
In the remainder, we compute $\hat\lambda_1\hat u_1 \otimes\hat v_1 \otimes\hat w_1$ and $\hat\lambda_2\hat u_2 \otimes\hat v_2 \otimes\hat w_2$ in all our simulations via tensor power iteration initialized with tensor SVD \cite{auddy2022estimating}, which has been proven to converge in polynomial time for $\beta_i\geq O(p^{3/2})$ in the orthogonal case $\alpha=0$. Moreover, for each $i\in [2]$, denote the following alignments as
\begin{equation}
    \begin{split}
        \hat\rho_{1i} &\equiv \vert \langle \hat u_1, x_i \rangle \vert \asymp \vert \langle \hat v_1, y_i \rangle \vert \asymp \vert \langle \hat w_1, z_i \rangle \vert,\\
    \hat\theta_{2i} &\equiv \vert \langle \hat u_2, x_i \rangle\vert, \quad \hat\rho_{2i} \equiv \vert \langle \hat v_2, y_i \rangle\vert \asymp \vert \langle \hat w_2, z_i \rangle\vert,\\
    \hat \kappa &\equiv \vert \langle \hat u_1, \hat u_2 \rangle\vert,\quad \hat \eta \equiv \vert \langle \hat v_1, \hat v_2 \rangle\vert \asymp \vert \langle \hat w_1, \hat w_2 \rangle\vert.
    \end{split}
\end{equation}
The equivalences $\vert \langle \hat u_1, x_i \rangle \vert \asymp \vert \langle \hat v_1, y_i \rangle \vert \asymp \vert \langle \hat w_1, z_i \rangle \vert$, $\vert \langle \hat v_2, y_i \rangle\vert \asymp  \vert \langle \hat w_2, z_i \rangle\vert$ and $\vert \langle \hat v_1, \hat v_2 \rangle\vert \asymp \vert \langle \hat w_1, \hat w_2 \rangle\vert$ are a consequence of our assumption in \eqref{assum_alpha} and since all the mode dimensions of $\gT_1$ are equal. Moreover, $ \hat \theta_{2i} \,\cancel{\asymp} \,\hat \rho_{2i} $ and $\hat \kappa\, \cancel{\asymp} \,\hat \eta$ since the projection in \eqref{eq_model_second_deflation} is applied only on the first mode.

In order to decipher the asymptotic behavior of the considered deflation method as $n\to \infty$, our main goal is to compute the asymptotic expected values of the singular values $\hat \lambda_i$ and the alignments $\hat \rho_{1i}, \hat \theta_{2i}, \hat \rho_{2i} , \hat \kappa,\hat \eta$. Indeed, using concentration arguments one can show that these quantities tend to concentrate around their expected values as $n$ grows large with variances of order $O(n^{-1})$, in the same vein as \cite{benaych2011fluctuations} which studied the fluctuations of the largest eigenvalues of large random matrices.

Moreover, we also address the problem of estimating the underlying model parameters, namely the signal-to-noise ratios $\beta_1, \beta_2$ and the correlation parameter $\alpha$ based on a single realization of $\gT_1$. This allows us to design an improved deflation algorithm in the correlated case. Our analysis relies on a recently developed random tensor theory approach from \cite{seddik2021random}. In particular, we analyze the random tensor model obtained at each deflation step by: \textbf{1)} Identifying a corresponding random matrix; \textbf{2)} Describing the limiting spectral measure of the latter and \textbf{3)} Computing the asymptotic singular value and corresponding alignments. 

\subsubsection{First Deflation Step}
We start by analyzing the random tensor model of the first deflation step, namely the tensor $\gT_1$ in \eqref{eq_first_deflation_model}.

\textbf{Corresponding Random Matrix Model:} Starting from the identities in \eqref{eq_singular_values_vectors} for $i=1$, it has been shown in \cite{seddik2021random} that the study of the random tensor $\gT_1$ and its associated singular value and vectors $(\hat\lambda_1,\hat u_1,\hat v_1,\hat w_1)$ boils down to the analysis of the following block-wise contraction random matrix of size $n\times n$ (see Appendix \ref{proof_first_deflation})
\begin{align}
    \mN \equiv \frac{1}{\sqrt n} \begin{pmatrix}
    0 & \gW(\hat w_1) & \gW(\hat v_1)\\
    \gW(\hat w_1)^\top & 0 & \gW(\hat u_1) \\
    \gW(\hat v_1)^\top & \gW(\hat u_1)^\top  & 0
    \end{pmatrix}.
\end{align}

\textbf{Limiting Spectral Measure:} In fact, the characterization of the limits of $\lambda_1$ and the alignments $\hat \rho_{1i}$ for $i\in [2]$ when $n\to \infty$ boils down to the computation of the Stieltjes transform of the limiting spectral measure of the random matrix $\mN$, see Appendix \ref{proof_singular_value_first_deflation} for details. We henceforth need the following technical assumptions to characterize such Stieltjes transform.

\begin{assumption}\label{assump_existence_first_deflation} As $n\to \infty$, there exists a sequence of critical points $(\hat\lambda_1,\hat u_1,\hat v_1,\hat w_1)$ such that $\hat\lambda_1\asto \lambda_1 > 2 \sqrt{ \frac{2}{3} } $ and $\hat \rho_{1i} \asto \rho_{1i}  > 0$.
\end{assumption}

Under Assumption \ref{assump_existence_first_deflation}, we have the following result from \citep[Corollary 1]{seddik2021random} which characterizes the limiting spectral measure of the random matrix $\mN$.

\begin{theorem}\label{thm_semi_circle} Under Assumption \ref{assump_existence_first_deflation}, the spectral measure of $\mN$ converges weakly to a semi-circle law $\mu$ of compact support $\left[-2 \sqrt{ \frac{2}{3} }, 2 \sqrt{ \frac{2}{3} }  \right] $ and density function $\mu(dx) = \frac{3}{4\pi} \sqrt{ \left( x^2 - \frac83 \right)^+ }$. Moreover, the Stieltjes transform of $\mu$ is
\begin{align*}
    r(z) = \frac34 \left( -z + \sqrt{z^2 - \frac83} \right), \quad \text{for}\quad z > 2 \sqrt{ \frac{2}{3}}.
\end{align*}
\end{theorem}

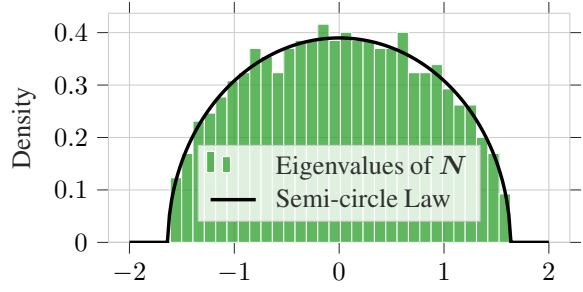
\begin{figure}[t!]
    \centering
    \input{figs/semi_circle_law.tex}
    \vspace{-.5cm}
    \caption{Histogram of the eigenvalues of $\mN$ and limiting semi-circle law from Theorem \ref{thm_semi_circle}. We considered $p=200$, $\beta_1=20$, $\beta_2=15$, $\alpha=0.8$ and one realization of $\gT_1$.}
    \label{fig_semi_circle}
\end{figure}

Figure \ref{fig_semi_circle} depicts the histogram of the eigenvalues of $\mN$ and the corresponding limiting semi-circle law as per Theorem \ref{thm_semi_circle}. Note that the spectral measure of $\mN$ is not affected by the parameters $\beta_1, \beta_2$ and $\alpha$ but some conditions are required on the latest to ensure Assumption \ref{assump_existence_first_deflation} as we will see subsequently. We also refer the reader to \cite{goulart2021random, seddik2021random} for a more detailed discussion on Assumption \ref{assump_existence_first_deflation} in the rank-one case.

\textbf{Asymptotic Singular Value and Alignments:} We now consider the computation of $\lambda_1$ to give an insight about Assumption \ref{assump_existence_first_deflation}. Given Theorem \ref{thm_semi_circle}, one can derive by taking the expectation w.r.t.\@ $\gW$ of the identity $ \hat\lambda_1 = \gT_1(\hat u_1, \hat v_1, \hat w_1) $ in \eqref{eq_singular_values_vectors}, that the limiting singular value $\lambda_1$ satisfies the following equation (see Appendix \ref{proof_singular_value_first_deflation} for the derivations)
\begin{align}
    \lambda_1 + r(\lambda_1) = \sum_{i=1}^2 \beta_i \rho_{1i}^3.
\end{align}
Therefore, since the Stieltjes transform $r$ has to be evaluated in $\lambda_1$, the latter must lie outside the support of $\mu$ which is ensured by Assumption \ref{assump_existence_first_deflation} if the signal strengths $\beta_1$ or $\beta_2$ are sufficiently high. In the case $\alpha=0$, it was shown in \citep[Corollary 3]{seddik2021random} that $\max\{\beta_1, \beta_2\}$ must be greater than $\frac{2\sqrt 3}{3}$ to ensure $\lambda_1 > 2 \sqrt{ \frac{2}{3}}$. Besides, note that when $\lambda_1 \leq 2 \sqrt{ \frac{2}{3}}$, i.e., $\lambda_1$ lies inside the support of $\mu$, it basically corresponds to the case where the tensor $\gT_1$ is indistinguishable from its noise counterpart $\gW$, and hence recovering the signal components is information-theoretically impossible \cite{montanari2014statistical,lesieur2017statistical, jagannath2020statistical, goulart2021random, seddik2021random}. 

Taking the expectation w.r.t.\@ $\gW$ of the remaining identities in \eqref{eq_singular_values_vectors} for $i=1$, projected on the signal components $x_i, y_i, z_i$ for $i\in [2]$ as detailed in Appendix \ref{proof_alignments_first_deflation}, allows us to obtain the following result which characterizes the asymptotic behavior of the first deflation step. 

\begin{theorem}\label{thm_first_deflation_step}
    Under Assumption \ref{assump_existence_first_deflation}, the limiting singular value $\lambda_1$ and corresponding alignments $\rho_{1i}$ for $i\in [2]$ of the first deflation step satisfy the following system of equations
    \begin{align}\label{eq_system_first_deflation}
        \begin{cases}
        f_r(\lambda_1) = \sum_{i=1}^2 \beta_i \rho_{1i}^3,\\
        h_r(\lambda_1) \rho_{1j} = \sum_{i=1}^2 \beta_i \alpha_{ij} \rho_{1i}^2\quad \text{for}\quad j \in [2],
    \end{cases}
    \end{align}
    where we recall $\alpha_{ij} = \alpha$ if $i\neq j$ and $1$ otherwise, and we denoted $f_r(z) = z + r(z)$ and $h_r(z) = - \frac{1}{r(z)}$.
\end{theorem}
\vspace{-.5cm}
\begin{proof}
    See Appendices \ref{proof_singular_value_first_deflation} and \ref{proof_alignments_first_deflation}.
\end{proof}

\begin{figure}[t!]
    \centering
    \input{figs/first_deflation_step_singular_value_alignments.tex}
    \vspace{-.4cm}
    \caption{Simulated versus asymptotic singular value and alignments corresponding to the first deflation step as per Theorem \ref{thm_first_deflation_step}. We considered $\beta_1 = 5$, $\alpha = 0.5$, $p=100$ and varying $\beta_1\in [0, 15]$. The system of equations in \eqref{eq_system_first_deflation} is solved numerically and initialized with the simulated singular value and alignments (dotted curves) from one realization of $\gT_1$.}
    \label{fig_first_deflation}
\end{figure}
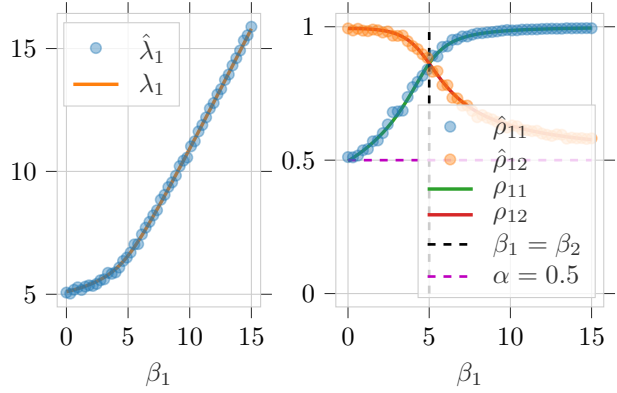

Figure \ref{fig_first_deflation} shows the simulated versus asymptotic singular value and alignments of the first deflation step as stated by Theorem \ref{thm_first_deflation_step}. Specifically, the asymptotic behavior of the first deflation step is described by a set of three polynomial equations involving $\lambda_1$ and $\rho_{1i}$ for $i\in [2]$. The existence and uniqueness of such solutions is not addressed in our present analysis and we stress out that the asymptotic curves in Figure \ref{fig_first_deflation} are obtained by solving numerically the system \eqref{eq_system_first_deflation} initialized with the simulated singular value and alignments from one realization of the random tensor $\gT_1$.

\subsubsection{Second Deflation Step}\label{sec_second_deflation}
We henceforth turn into the description of the second deflation step asymptotics.

\textbf{Corresponding Random Matrix Model:} Denote $\hat u_3 = \hat u_2 - \gamma \langle \hat u_1, \hat u_2 \rangle \hat u_1 $. We show in Appendix \ref{proof_second_deflation} that the study of the second deflation step boils down to the analysis of the following $n\times n$ block-wise contraction random matrix
 \begin{align}\label{eq_second_contraction_matrix}
    \mM \equiv \frac{1}{\sqrt n} \begin{pmatrix}
    0 & \gW(\hat w_2) & \gW(\hat v_2)\\
    \gW(\hat w_2)^\top & 0 & \gW(\hat u_3 ) \\
    \gW(\hat v_2)^\top & \gW(\hat u_3 )^\top   & 0
    \end{pmatrix},
\end{align}

\textbf{Limiting Spectral Measure:} We demonstrate that for some $\gamma \neq 1$, the limiting spectral measure of $\mM$ does not follow a semi-circle law due to the additional term $\gamma \langle\hat u_1, \hat u_2 \rangle \gW(\hat u_1)$ induced by the correlation between the singular vectors $\hat u_1$ and $\hat u_2$. Besides, when $\gamma=0$ or $\gamma=1$, the term $\gamma \langle\hat u_1, \hat u_2 \rangle \gW(\hat u_1)$ vanishes in which cases the limiting spectral measure of $\mM$ is again described by the semi-circle law in Theorem \ref{thm_semi_circle}. In fact, when $\gamma=1$, we have from the identity $\gT_2(\cdot, \hat v_2, \hat w_2) = \hat \lambda_2 \hat u_2$ in \eqref{eq_singular_values_vectors} and given $\gT_2$ in \eqref{eq_model_second_deflation}
\begin{equation}\label{eq_kappa_zero}
    \begin{split}
        &\lambda_2 \langle \hat u_1, \hat u_2 \rangle = \gT_2( \hat u_1, \hat v_2, \hat w_2 ) = \gT_1( \hat u_1, \hat v_2, \hat w_2 ) -  \underbrace{\langle \hat u_1, \hat u_1 \rangle}_{=1} \gT_1( \hat u_1, \hat v_2, \hat w_2 ) = 0,
    \end{split}
\end{equation}
which implies $\langle \hat u_1, \hat u_2 \rangle=0$ since the spectral norm of the tensor $\gT_2$ is not null, due to the presence of the noise term.

We therefore provide the result characterizing the limiting spectral measure of $\mM$ for any $\gamma\in [0, 1]$, and which in turn generalizes Theorem \ref{thm_semi_circle} to random contraction matrices of the form in \eqref{eq_second_contraction_matrix}. We start by the following definition.

\begin{definition}\label{def_limiting_measure} Let $\nu$ be the probability measure with Stieltjes transform $q(z) = a(z) + 2 b(z) $ verifying $\Im[q(z)]>0$ for $\Im[z]>0$, where $a(z)$ and $b(z)$ satisfy the following system of equations, for $z\notin \supp(\nu)$
\begin{equation}\label{eq_fixed_point}
    \begin{cases}
        \left[ 2b(z) +z \right] a(z) + \frac13 = 0,\\
        (a(z) + z - \tau b(z)) b(z) + \frac13 = 0,
    \end{cases}
\end{equation}
for some parameter $\tau\in \sR$. Moreover, the density function corresponding to $\nu$ is given by the Stieltjes inverse formula $\nu(dx) = \frac{1}{\pi} \lim_{\varepsilon\to 0}  \Im [ q(x + i\varepsilon) ]$.
\end{definition}

Similarly to the analysis of the first deflation step, we need some additional technical assumptions to describe the limiting singular value $\lambda_2$ and corresponding alignments.

\begin{assumption}\label{assump_existence_second_deflation}
    As $n\to \infty$, there exists a sequence of critical points $(\hat\lambda_2, \hat u_2, \hat v_2, \hat w_2 )$ such that, for $i\in[2]$
    \begin{align*}
        &\hat \lambda_2 \asto \lambda_2, \,\, \hat \theta_{2i} \asto \theta_{2i}, \,\, \hat \kappa \asto \kappa, \,\, \hat \rho_{2i} \asto \rho_{2i},\,\, \hat  \eta \asto \eta,
    \end{align*}
    where $\lambda_2 \notin \supp(\nu)$ with $\nu$ defined in Definition \ref{def_limiting_measure} for $\tau =\gamma\kappa^2-1+\kappa(\gamma-1)$ and suppose $\theta_{2i}, \kappa, \rho_{2i}, \eta > 0$.
\end{assumption}

We therefore have the following theorem which characterizes the limiting spectral measure of $\mM$.

\begin{theorem}\label{thm_limiting_measure_second_deflation}
    Under Assumption \ref{assump_existence_second_deflation}, the spectral measure of $\mM$ converges weakly to the probability measure $\nu$ defined in Definition \ref{def_limiting_measure} for $\tau =\gamma\kappa^2-1+\kappa(\gamma-1)$.
\end{theorem}
\vspace{-.5cm}
\begin{proof}
    See Appendix \ref{proof_limiting_measure_second_deflation}.
\end{proof}

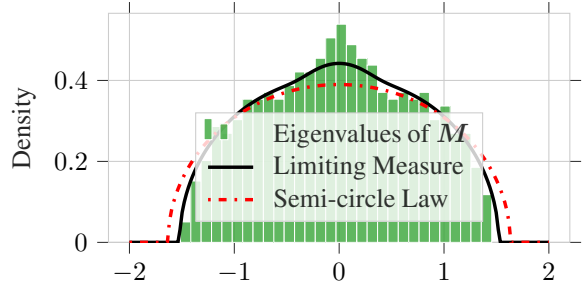
\begin{figure}[t!]
    \centering
    \input{figs/limiting_law.tex}
    \vspace{-.5cm}
    \caption{Histogram of the eigenvalues of $\mM$ and corresponding limiting spectral measure as per Theorem \ref{thm_limiting_measure_second_deflation}. We considered $p=200$, $\beta_1=20$, $\beta_2=15$, $\alpha=0.8$, $\gamma=0.85$ and one realization of $\gT_1$.}
    \label{fig_spectrum_second_deflation}
\end{figure}

In essence, if the involved alignments in the second deflation step converge asymptotically, Theorem \ref{thm_limiting_measure_second_deflation} states the convergence of the spectral measure of $\mM$ to the deterministic measure $\nu$ defined in Definition \ref{def_limiting_measure} for $\tau =\gamma\kappa^2-1+\kappa(\gamma-1)$. We particularly recall that $\kappa$ corresponds to the limit of $\langle \hat u_1, \hat u_2 \rangle$, which highlights the fact that the spectrum of $\mM$ can be deformed if the singular vectors $u_1$ and $u_2$ are correlated, i.e., if $\gamma\neq 1$. This phenomenon is depicted in Figure \ref{fig_spectrum_second_deflation} where we see that for $\gamma=0.85$, the limiting spectral measure of $\mM$ defers from the semi-circle law. In contrast, if $\gamma=1$ we have $\kappa=0$ as we saw in \eqref{eq_kappa_zero} which implies that $\tau=-1$. In this case, the limiting spectral measure $\nu$ becomes equal to $\mu$, thereby describing again a semi-circle law. This can be trivially checked from Definition \ref{def_limiting_measure} by setting $\tau=-1$ and $a(z)=b(z)$, and we therefore find $a(z) = b(z) = \frac{r(z)}{3}$ and $q(z) = r(z)$. Note that for $\gamma\in(0, 1)$, the Stieltjes transform $q(z)$ can be computed numerically by alternating the equations in \eqref{eq_fixed_point} as per Algorithm \ref{alg_fixed_point}, which can be proved to converge to a fixed point in the same vein as \cite{louart2018concentration}.

\textbf{Asymptotic Singular Value and Alignments:} As for the first deflation step, taking the expectation w.r.t.\@ $\gW$ of the identity $\hat \lambda_2 = \gT_2(\hat u_2, \hat v_2, \hat w_2)$ in \eqref{eq_singular_values_vectors} allows us to obtain the equation satisfied by $\lambda_2$, see Appendix \ref{proof_singular_value_second_deflation}, which yields
\begin{equation}
    \begin{split}
        f_q(\lambda_2) &-   \frac{\gamma\kappa \eta^2}{3} r(\lambda_1) -   2\gamma \kappa^2 b(\lambda_2) = \sum_{i=1}^2 \beta_i \theta_{2i} \rho_{2i}^2  -    \gamma \kappa \sum_{i=1}^2 \beta_i \rho_{1i} \rho_{2i}^2,
    \end{split}
\end{equation}
where $f_q(z) = z + q(z)$. Again, the limiting singular value $\lambda_2$ must lie outside the support of $\nu$, as we assumed in Assumption \ref{assump_existence_second_deflation}, since its corresponding Stieltjes transform $q(z)$ (and the function $b(\cdot)$) needs to be evaluated at $\lambda_2$. In fact, if $\lambda_2\in \supp(\nu)$, then it is information-theoretically impossible to recover the second signal term (i.e., the one with strength $\min \{ \beta_1, \beta_2\} $).

\begin{table*}[t!]
    \begin{equation}\label{eq_system_general_gamma}
    \begin{cases}
        f_q(\lambda_2) -   \frac{\gamma\kappa \eta^2}{3} r(\lambda_1) -   2\gamma \kappa^2 b(\lambda_2)= \sum_{i=1}^2 \beta_i \theta_{2i} \rho_{2i}^2  -    \gamma \kappa \sum_{i=1}^2 \beta_i \rho_{1i} \rho_{2i}^2,\\
        [f_q (\lambda_2) - a(\lambda_2)] \theta_{2j} - \gamma \rho_{1j} \left[ \frac{\eta^2}{3} r(\lambda_1) +  2\kappa b(\lambda_2)  \right]=\sum_{i=1}^2 \beta_i \alpha_{ij} \rho_{2i}^2 -  \gamma \rho_{1j} \sum_{i=1}^2 \beta_i \rho_{1i} \rho_{2i}^2 \quad \text{for} \quad j\in[2],\\
        \left[ \lambda_2 + 2(1 - {\gamma}) b(\lambda_2)  \right] \kappa= (1 - {\gamma}) \left[\sum_{i=1}^2 \beta_i \rho_{1i} \rho_{2i}^2 - \frac{\eta^2}{3} r(\lambda_1) \right],\\
        \left[ f_q(\lambda_2) - (1 + \gamma  \kappa^2) b(\lambda_2)\right]  \rho_{2j} =  \sum_{i=1}^2 \beta_i \theta_{2i} \rho_{2i} \alpha_{ij}  -  \gamma  \kappa  \left[ \sum_{i=1}^2 \beta_i \rho_{1i} \rho_{2i} \alpha_{ij}  - \frac{\rho_{1j} \eta}{3} r(\lambda_1) \right]\,\text{for}\quad j\in [2], \\
        \left[ \lambda_2 + a(\lambda_2) + (1 - \gamma \kappa^2) b(\lambda_2) - \frac{\gamma \kappa }{3} r(\lambda_1) \right] \eta   = \sum_{i=1}^2 \beta_i \theta_{2i} \rho_{1i} \rho_{2i} - {\gamma}\kappa  \sum_{i=1}^2 \beta_i \rho_{1i}^2 \rho_{2i}.
    \end{cases}
\end{equation}
\vspace{-.7cm}
\end{table*}

Moreover, taking the expectation w.r.t.\@ $\gW$ of the remaining identities in \eqref{eq_singular_values_vectors} for $i=2$, projected on the signal components $x_i, y_i, z_i$ for $i\in [2]$ and the first singular vectors $u_1, v_1, w_1$, allows us to derive the result characterizing the behavior of the second deflation step.

\begin{theorem}\label{thm_second_deflation} Under Assumption \ref{assump_existence_second_deflation}, the limiting singular value $\lambda_2$ and corresponding alignments $\theta_{2i}, \rho_{2i}$ for $i\in[2]$ and $\kappa, \eta$ of the second deflation step satisfy the system of equations in \eqref{eq_system_general_gamma}.
\end{theorem}
\vspace{-.5cm}
\begin{proof}
    See Appendices \ref{proof_singular_value_second_deflation} and \ref{proof_alignments_second_deflation}.
\end{proof}

Figure \ref{fig_alignments_second_deflation} depicts the simulated singular value and alignments of the second deflation step and their asymptotic counterparts as given by Theorem \ref{thm_second_deflation}. In essence, the asymptotic behavior of the second deflation step is described by a set of seven polynomial equations in $\lambda_2$ and the alignments $\theta_{2i}, \rho_{2i}, \kappa$ and $\eta$. Again, we do not address the existence and uniqueness of such solutions, and we solve the system in \eqref{eq_system_general_gamma} numerically starting from the simulated singular value and alignments from one realization of $\gT_1$. Contrarily to the first deflation step, we highlight that the Stieltjes transform $q(z)$ depends on the alignment $\kappa$. Therefore, we alternate solving the system in \eqref{eq_system_general_gamma} with the fixed point equations in \eqref{eq_fixed_point} for $\tau = \gamma \kappa^2 - 1 + \kappa(\gamma -1)$ as per Theorem \ref{thm_limiting_measure_second_deflation}.

We further stress out that for a fixed $\beta_2$ large enough and $\alpha\neq 1$, there exists a threshold for $\beta_1$ below which it is information-theoretically impossible to recover the second signal component. This can be visible from Figure \ref{fig_alignments_second_deflation} for $\beta_1 \approx 2 $, below which all the alignments are asymptotically null and the asymptotic singular value converges to the right edge of the distribution $\nu$. Precisely, this corresponds to the scenario where Assumption \ref{assump_existence_second_deflation} is not verified. Moreover, not that there might also exist a theoretical-algorithmic spectral gap, that needs to be determined for the present deflation procedure, in the same vein as in \cite{montanari2014statistical} for the rank-one case.

\begin{figure*}[t!]
    \centering
    \input{figs/optimization_of_gamma_alpha_0.6.tex}
    \vspace{-.8cm}
    \caption{Asymptotic alignments of the second deflation varying the hyper-parameter $\gamma$. We considered $\beta_1=10$, $\beta_2 = 8$ and $\alpha=0.6$.}
    \label{fig_varying_gamma}
\end{figure*}
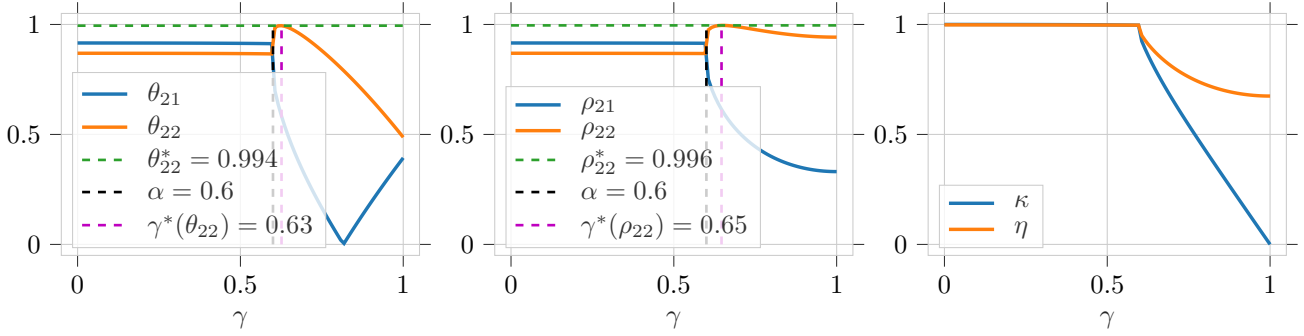

\textbf{Case $\gamma=1$:} As we discussed earlier, in the case $\gamma=1$ the limiting spectral measure $\nu$ becomes equal to the semi-circle law $\mu$ described in the first deflation step. Moreover, the system of equations in \eqref{eq_system_general_gamma} reduces to the following equations, for $j\in [2]$, which will be needed subsequently.
\begin{equation}\label{eq_system_gamma_one}
    \begin{cases}
        f_r(\lambda_2)  = \sum_{i=1}^2 \beta_i \theta_{2i} \rho_{2i}^2  \\
        h_r (\lambda_2) \theta_{2j} -  \frac{\eta^2}{3} r(\lambda_1) \rho_{1j}  =\sum_{i=1}^2 \beta_i \alpha_{ij} \rho_{2i}^2 -   \rho_{1j} \sum_{i=1}^2 \beta_i \rho_{1i} \rho_{2i}^2\\
         h_r(\lambda_2)   \rho_{2j} =  \sum_{i=1}^2 \beta_i \theta_{2i} \rho_{2i} \alpha_{ij}  \\
        \left[ \lambda_2 + \frac23 r(\lambda_2)  \right] \eta   = \sum_{i=1}^2 \beta_i \theta_{2i} \rho_{1i} \rho_{2i} 
    \end{cases}
\end{equation}

\subsubsection{Model Parameters Estimation}\label{sec_model_estimation}
In this section, we discuss the problem of estimating the underlying model parameters, namely the SNRs and the signal components correlation $\vbeta \equiv (\beta_1, \beta_2, \alpha)\in \sR^3$, and the alignments $\vrho \equiv (\rho_{1i}, \rho_{2i}, \theta_{2i}\mid i\in [2])\in \sR^6$ from one realization of the random tensor $\gT_1$. Indeed, this will allow us to design an improved deflation algorithm by optimizing the parameter $\gamma$ introduced in the second deflation step. Further denoting $\vlambda \equiv (\lambda_1, \lambda_2, \eta)\in \sR^3$, we define the mapping $\psi:\sR^3\times \sR^3\times \sR^6 \to \sR^9$ through \eqref{eq_psi}, where the first three entries of the vector $\psi(\vbeta, \vlambda, \vrho)$ correspond to the first deflation step equations in \eqref{eq_system_first_deflation} while the remaining entries correspond to the second deflation step for $\gamma=1$ characterized by \eqref{eq_system_gamma_one}. In particular, the singular values $\lambda_1, \lambda_2$ and the corresponding alignments satisfy $\psi (\vbeta, \vlambda, \vrho) = 0$. On the other hand, given an estimate $\hat \vlambda = (\hat \lambda_1, \hat \lambda_2, \hat \eta)$ of $\vlambda$, which can be computed via tensor power iteration as discussed previously, we can solve $\psi(\cdot, \hat \vlambda, \cdot)=0$ in the variables $\vbeta$ and $\vrho$ while fixing $\hat \vlambda$, which allows us to estimate the model parameters $\hat\vbeta$ and $\hat\vrho$. In particular, Figure \ref{fig_model_estimation} supports this statement where we see that solving $\psi(\cdot, \hat \vlambda, \cdot)=0$ allows us to estimate $\beta_1$ and $\beta_2$ with reasonably low variance. Further details are deferred to Appendix \ref{appendix_parameters_estimation}.

\subsection{RTT-Improved Tensor Deflation Algorithm}\label{sec_improved}
\begin{figure}[t!]
    \centering
    \input{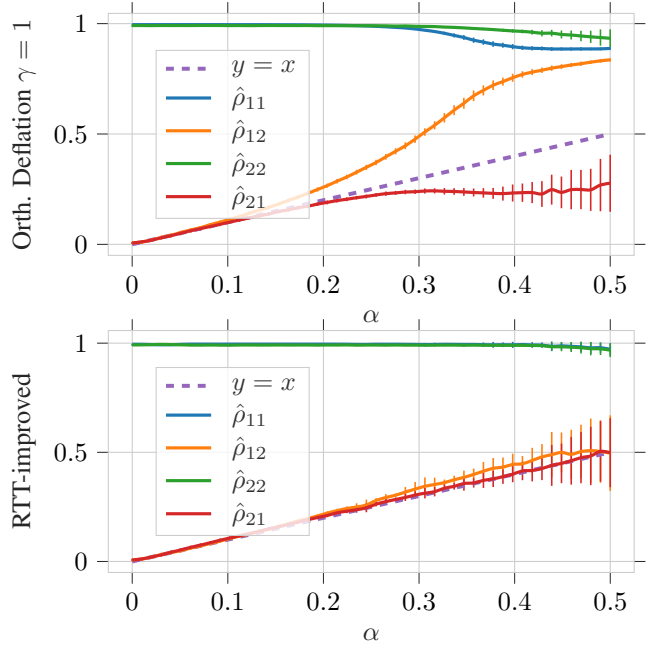}
    \vspace{-.4cm}
    \caption{Alignments of first and second deflation steps in terms of $\alpha$. (top) Performance of standard orthogonalized deflation ($\gamma=1$) and (bottom) of the RTT-improved orthogonalized deflation Algorithm \ref{alg_RTT_improved}. We considered $\beta_1=6$, $\beta_2=5.7$ and $p=150$. The curves are obtained by averaging over $200$ realizations of $\gT_1$ and we depict the means and std of the different alignments.}
    \label{fig_improved}
\end{figure}

We are now in place to describe our improved tensor deflation algorithm. Our principal insight lies in the fact that, for $\beta_1 > \beta_2$ for instance, the asymptotic alignments $\theta_{22}$ and $\rho_{22}$ at the second deflation step are concave functions of the parameter $\gamma$ as depicted in Figure \ref{fig_varying_gamma} for $\alpha=0.6$ and Figure \ref{fig_gamma_appendix} for different values of $\alpha$. Therefore, there exists an optimal value $\gamma^*$ which maximizes such alignments and which we need to tune in order to recover the signal components efficiently, given only one realization of the spiked random tensor $\gT_1$. To this end, we propose the following procedure, which is deferred in the Appendix in Algorithm \ref{alg_RTT_improved} due to space limitation:

\begin{itemize}
    \item First we perform a standard orthogonalized tensor deflation with $\gamma = 1$ which corresponds to the steps 1 and 2 of Algorithm \ref{alg_RTT_improved}.
    \item Then we estimate the underlying model parameters, i.e., $\beta_1, \beta_2$ and $\alpha$ as we discussed in Section \ref{sec_model_estimation}. This corresponds to the steps 3 and 4 of Algorithm \ref{alg_RTT_improved}.
    \item In order to find the optimal parameter $\gamma^*$ which maximizes the alignment $\hat\rho_{22}$ for instance. We update $\gamma$ as $\gamma\leftarrow \gamma - \epsilon$ for some small step size $\epsilon > 0$ and starting from $\gamma = 1$, while solving the system in \eqref{eq_system_general_gamma} to get an estimation for $\hat\rho_{22}$. We stop updating $\gamma$ until the maximum value of $\hat\rho_{22}$ is reached and we return the corresponding $\gamma$ as $\gamma^*$. Note that at each iteration, the system in \eqref{eq_system_general_gamma} is solved numerically and initialized with the previous iteration estimates. This corresponds to the steps 5-12 in Algorithm \ref{alg_RTT_improved}.
    \item We then perform orthogonalized deflation with $\gamma^*$ along the modes $1$ and $2$ which provides better estimation of the signal component denoted as $\hat \lambda_2 \hat u_2^* \otimes \hat v_2^* \otimes \hat w_2^*$. This corresponds to steps 13 and 14 of Algorithm \ref{alg_RTT_improved}.
    \item Finally in step 15 of Algorithm \ref{alg_RTT_improved}, we re-estimate the first signal component by performing a best rank-one approximation of $\gT_1 - \min\{\hat\beta_1, \hat\beta_2 \}  \hat u_2^* \otimes \hat v_2^* \otimes \hat w_2^*$ with $\hat\beta_1, \hat\beta_2$ the estimated SNRs from step 4 of Algorithm \ref{alg_RTT_improved}.
\end{itemize}

Figure \ref{fig_improved} depicts the performances of the standard orthogonalized deflation ($\gamma=1$) and our proposed RTT-improved version, while varying the signal correlation parameter $\alpha$. As theoretically anticipated, the RTT-improved algorithm recovers successively the signal components in the more challenging correlated setting (e.g., $\alpha\geq 0.3$). 

\section{Conclusion and Future Work}
We have showcased a concrete example where random tensor theory allows us to understand and even improve signal recovery from low-rank asymmetric spiked random tensors. To the best of our knowledge, this is the first time where an asymptotic characterization of the considered deflation method is carried out. We highlight that our methodology can be straightforwardly generalized to higher order and higher (low) rank tensors in the same vein as \cite{seddik2021random}, which studied order-$d$ and rank-one spiked tensor models. Such generalization will require more analytical computations and will result in more complicated systems of equations. Hence, for the sake of clarity, we limited our analysis to the simpler order-three and rank-two model.

Our approach has still many limitations and raises various open questions which we discuss as follows: (i) Our main results rely on Assumptions \ref{assump_existence_first_deflation} and \ref{assump_existence_second_deflation} which basically suppose the almost convergence of the singular values and alignments of interest. Similar Assumptions were made and discussed in \cite{goulart2021random, seddik2021random} which also relied on a RMT approach. A formal proof of these statements is still required and would make our analysis more complete. (ii) The second point concerns the existence and uniqueness of the solutions of the involved systems of equations. (iii) The third point concerns a proof of consistency of the underlying model parameters estimation. Specifically, proving a central limit theorem about the convergence of our estimates to the true parameters and the related conditions. We do not address these questions in our present analysis and we defer them to a future study.

Nevertheless, our actual results pave already a new path towards the analysis and improvement of more sophisticated tensor methods and models, by means of random tensor theory, thereby impacting tensor-based machine learning methods and many other applications which rely on tensors.

\textbf{Code and Reproducibility:} A Github repository will be provided for the reproducibility of our simulations and results, with an implementation of Algorithm \ref{alg_RTT_improved}.


\bibliography{refs}
\bibliographystyle{icml2023}


\section{Additional Simulations}
\subsection{Simulated and Asymptotic Alignments at Second Deflation Step}
Figure \ref{fig_alignments_second_deflation} depicts the simulated and asymptotic singular value and alignments of the second deflation step as described in Section \ref{sec_second_deflation}. Since Theorem \ref{thm_second_deflation} requires Assumption \ref{assump_existence_second_deflation}, the system of equations in \eqref{eq_system_general_gamma} may have many solutions in general, but when initializing it with the simulated values of $\lambda_2$ and alignments, we obtain consistent characterization of their asymptotic counterparts. As we discussed in the conclusion, the existence and uniqueness of the solutions of \eqref{eq_system_general_gamma} is left for a future study.
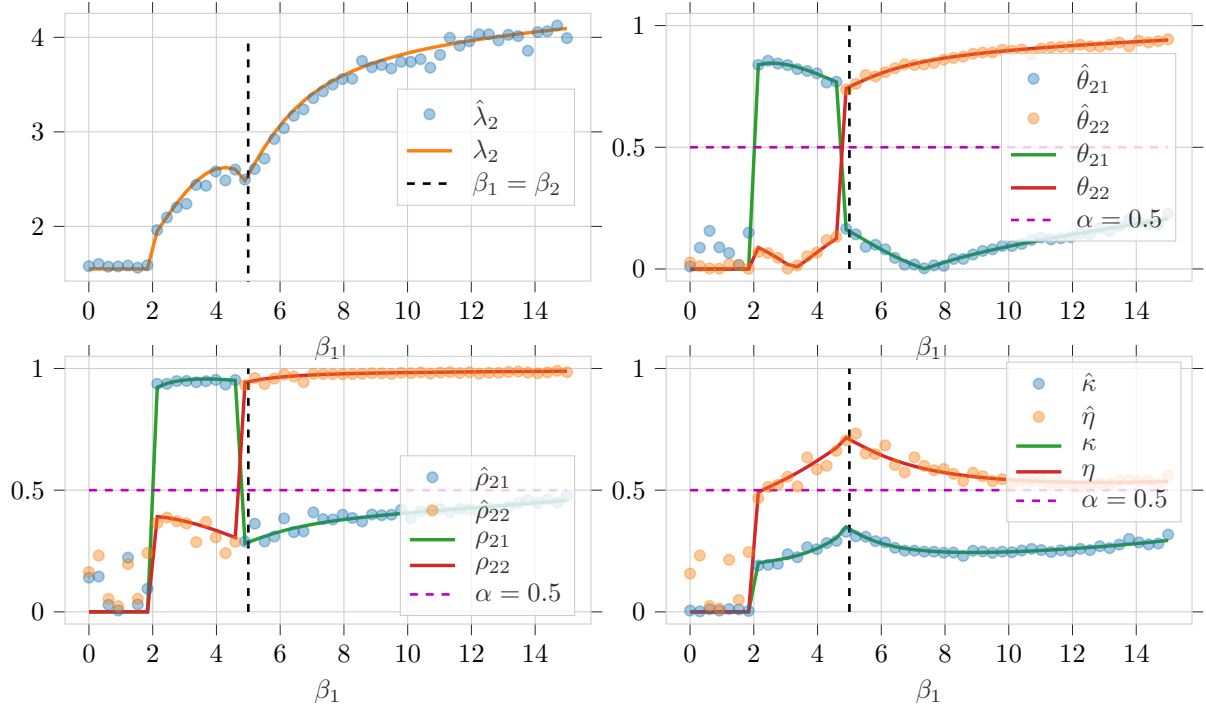
\begin{figure}[h!]
    \centering
    \input{figs/second_deflation_step_singular_value_alignments.tex}
    \caption{Simulated versus asymptotic singular value and alignments corresponding to the second deflation step as per Theorem \ref{thm_second_deflation}. We considered $\beta_1=5$, $\alpha=0.5$, $p=100$, $\gamma=0.8$ and varying $\beta_1\in [0, 15]$. The system of equations in \eqref{eq_system_general_gamma} is solved numerically and initialized with the simulated singular value and alignments (dotted curves) from one realization of $\gT_1$.}
    \label{fig_alignments_second_deflation}
\end{figure}

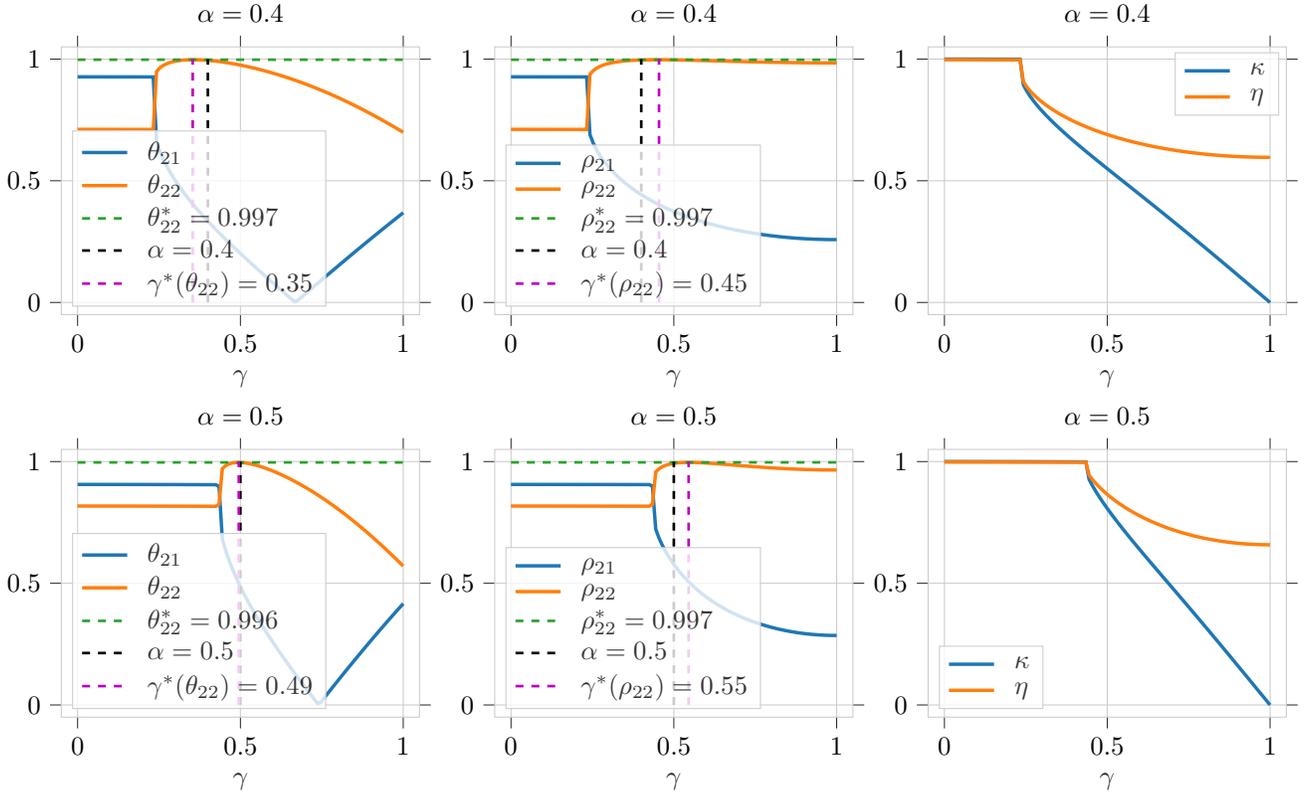
\begin{figure*}[h!]
    \centering
    \input{figs/optimization_of_gamma_alpha_0.4.tex}
    \input{figs/optimization_of_gamma_alpha_0.5.tex}
    \vspace{-.7cm}
    \caption{Asymptotic alignments of the second deflation step in terms of $\gamma$ and $\alpha$. We considered $\beta_1=10$ and $\beta_2 = 8$.}
    \label{fig_gamma_appendix}
\end{figure*}

\subsection{More on Model Parameters Estimation}\label{appendix_parameters_estimation}
In this section, we provide more discussions about the model parameters estimation described in Section \ref{sec_model_estimation}. An import aspect about such estimation, is to prove its consistency. Namely, demonstrating a CLT result which shows the concentration of $\hat \vbeta$ around the true $\vbeta$ as well as for $\hat \vrho$. We currently support this statement through simulations as depicted in Figures \ref{fig_model_estimation} and \ref{fig_alignments_estimation}. Note however that, given the concentration of $\hat \vlambda$, we believe that such consistency can be ensured with additional assumptions on the function $\psi$ in \eqref{eq_psi} and in particular the existence and uniqueness of
solution to the equation $\psi(\cdot, \hat \vlambda, \cdot)=0$. 
 \begin{equation}\label{eq_psi}
\begin{split}
     \psi: (\vbeta, \vlambda, \vrho) \mapsto \begin{pmatrix}
        f_r(\lambda_1) - \sum_{i=1}^2 \beta_i \rho_{1i}^3\\
        h_r(\lambda_1) \rho_{11} - \sum_{i=1}^2 \beta_i \alpha_{i1} \rho_{1i}^2\\
        h_r(\lambda_1) \rho_{12} - \sum_{i=1}^2 \beta_i \alpha_{i2} \rho_{1i}^2\\
        f_r(\lambda_2)  - \sum_{i=1}^2 \beta_i \theta_{2i} \rho_{2i}^2\\
        h_r (\lambda_2) \theta_{21} -  \frac{\eta^2}{3} r(\lambda_1) \rho_{11}  -\sum_{i=1}^2 \beta_i \alpha_{i1} \rho_{2i}^2 +   \rho_{11} \sum_{i=1}^2 \beta_i \rho_{1i} \rho_{2i}^2\\
        h_r (\lambda_2) \theta_{22} -  \frac{\eta^2}{3} r(\lambda_1) \rho_{12}  -\sum_{i=1}^2 \beta_i \alpha_{i2} \rho_{2i}^2 +   \rho_{12} \sum_{i=1}^2 \beta_i \rho_{1i} \rho_{2i}^2\\
        h_r(\lambda_2)   \rho_{21} -  \sum_{i=1}^2 \beta_i \theta_{2i} \rho_{2i} \alpha_{i1}\\
        h_r(\lambda_2)   \rho_{22} -  \sum_{i=1}^2 \beta_i \theta_{2i} \rho_{2i} \alpha_{i2}\\
        \left[ \lambda_2 + \frac23 r(\lambda_2)  \right] \eta   - \sum_{i=1}^2 \beta_i \theta_{2i} \rho_{1i} \rho_{2i} 
    \end{pmatrix}.
\end{split}
\end{equation}

\begin{figure}[h!]
    \centering
    \input{figs/model_estimation_mo.tex}
    \caption{Estimation of the underlying SNRs $ \beta_1$ and $\beta_2$ as described in Section \ref{sec_model_estimation}. We considered $\beta_2=5$, $\alpha=0.5$, $p=150$ and $\gamma=1$ while varying $\beta_2$. The parameters are estimated only from the singular values $\hat\lambda_1$, $\hat \lambda_2$ and the alignment between the singular vectors $\hat \eta = \langle \hat v_1, \hat v_2 \rangle$, computed via tensor power iteration. The curves are averaged over $100$ realizations of $\gT_1$.}
    \label{fig_model_estimation}
\end{figure}
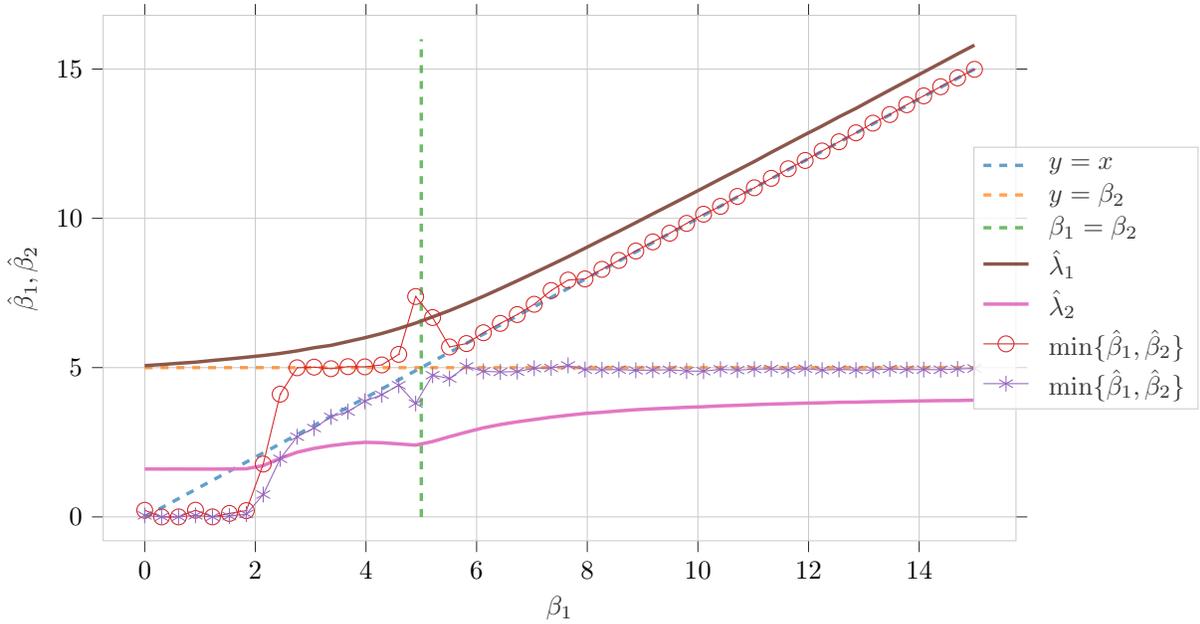

\begin{figure}[h!]
    \centering
    \input{figs/alignments_estimation_gamma_1.tex}
    \caption{Estimation of the alignments as described in Section \ref{sec_model_estimation} from one realization of the random tensor $\gT_1$. We considered $\beta_1 = 15$, $\beta_2=5$, $\gamma=1$, $p=100$ while varying $\alpha$. The curves are averaged over $100$ realizations of $\gT_1$. The hats correspond to simulations while tildes correspond to the estimated alignments as per Section \ref{sec_model_estimation}.}
    \label{fig_alignments_estimation}
\end{figure}
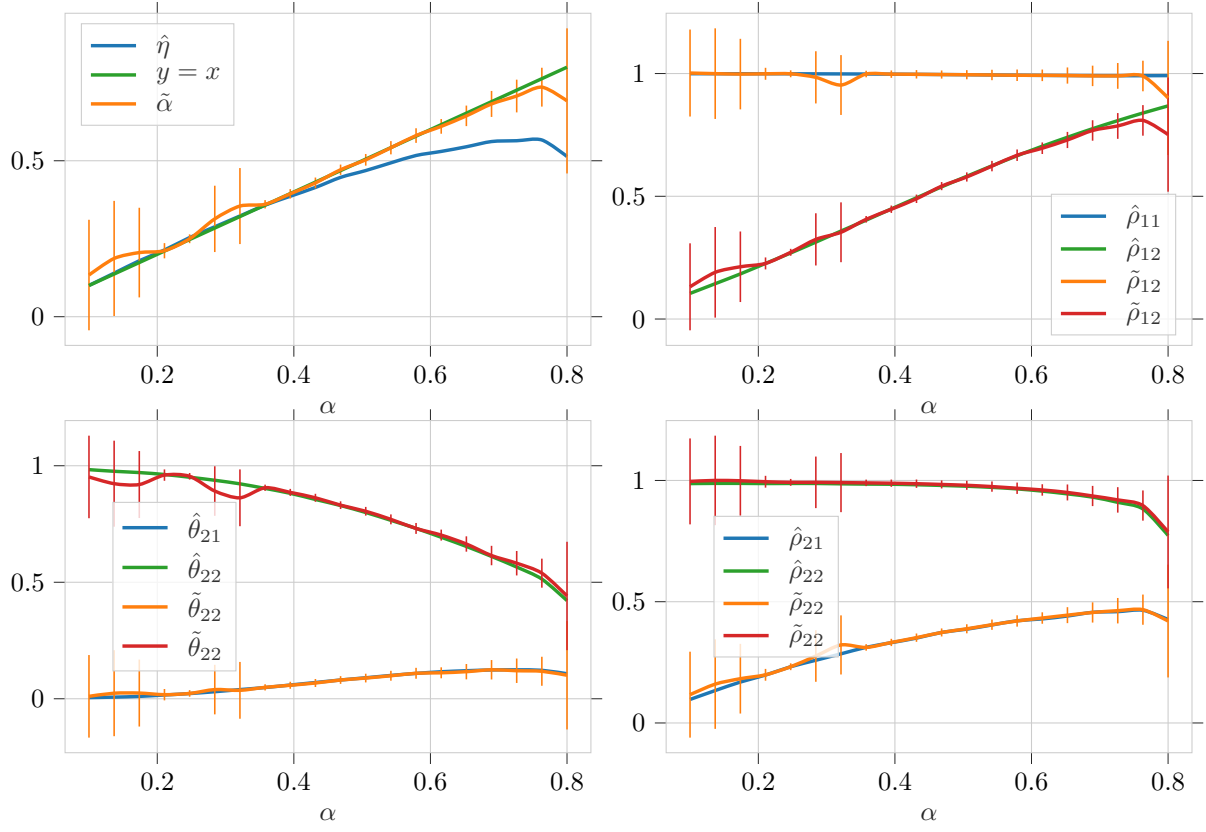

\section{Some Key Lemmas}
In this section, we recall some key lemmas that are at the heart of our analysis.
\begin{lemma}[Woodbury matrix identity]\label{lemma_woodbury} Let $\mA\in \sR^{n\times n}$, $\mB \in \sR^{r\times r}$, $\mU\in \sR^{n\times r}$ and $\mV \in \sR^{r\times n}$, we have
\begin{align*}
    \left( \mA + \mU \mB \mV \right)^{-1} = \mA^{-1} - \mA^{-1} \mU \left( \mB^{-1} + \mV \mA^{-1} \mU \right)^{-1} \mV \mA^{-1}
\end{align*}
\end{lemma}

The following perturbation lemma is wildly used in RMT. Basically, it states that the normalized trace operator is invariant through low-rank perturbations in high dimension. The notation $a=O_n(b_n)$ means that $a$ is of order $O(b_n)$ as $n\to \infty$.

\begin{lemma}[Perturbation lemma \citep{silverstein1995empirical}]\label{lemma_perturbation} Let $\mM\in \sR^{n\times n}$ and $\mP \in \sR^{n\times n}$ such that  $\Vert \mM \Vert = O_n(1)$,  $\Vert \mP \Vert = O_n(1)$ and $\rank(\mP) = O_n(1)$. For all $z\in \sC \setminus \spec(\mM + \mP) $, we have
\begin{align*}
    \frac1n \Tr \left( \mM + \mP - z \mI_n \right)^{-1} = \frac1n \Tr \left( \mM - z \mI_n \right)^{-1} + O_n(n^{-1}) 
\end{align*}
\end{lemma}
\begin{proof} Simple consequence of the Woodbury identity from Lemma \ref{lemma_woodbury} applied to the matrix $\mM + \mP$.
\end{proof}
Our analysis will particularly rely on computing expectations which we drive through the classical Stein's Lemma. 
\begin{lemma}[Stein's Lemma \cite{stein1981estimation}]\label{lemma_stein}
Let $W\sim \mathcal{N}(0, \sigma^2)$ and $f$ some continuously differentiable function having at most polynomial growth, then
\begin{align*}
    \E \left[ W f(W) \right] = \sigma^2 \E \left[ f'(W) \right]
\end{align*}
when the above expectations exist.
\end{lemma}

\section{Proofs of the main results}
We recall our considered spiked tensor model as follows
\begin{align}\label{eq_spiked_model}
    \gT_1 = \gS  + \frac{1}{\sqrt n} \gW \in \sR^{p\times p\times p} \quad \text{with} \quad \gS=\sum_{i=1}^2 \beta_i x_i \otimes y_i \otimes z_i
\end{align}
where $\Vert x_i\Vert = \Vert y_i\Vert = \Vert z_i\Vert = 1$, $\beta_i\geq 0$, $n=3p$ and $W_{ijk}\sim \gN(0, 1)$. In the remainder, if some quantity expresses as $a(n) = \sum_{i=1}^r b_i(n)$, the notation $a(n) \simeq b_j(n)$ means that $b_j(n)$ is the only contributing term in the expression of $a(n)$ as $n$ goes to infinity.
\subsection{First deflation step}\label{proof_first_deflation}
The singular vectors $u_1, v_1$ and $w_1$ of $\gT_1$ corresponding to its largest singular value $\lambda_1$ satisfy
\begin{align}\label{eq_id_1}
\gT_1(\cdot, v_1, w_1) = \lambda_1 u_1, \quad  \gT_1(u_1, \cdot, w_1) = \lambda_1 v_1, \quad \gT_1(u_1, v_1, \cdot) = \lambda_1 w_1
\end{align}
In the remainder, we will need to compute the derivatives of the singular vectors $u_1, v_1$ and $w_1$ w.r.t.\@ the entries of the noise tensor $\gW$. From \citep[Appendix B.1]{seddik2021random}, we have
\begin{align}\label{eq_deriv}
    \begin{pmatrix}
        \frac{\partial u_1}{\partial W_{ijk}}\\
        \frac{\partial v_1}{\partial W_{ijk}}\\
        \frac{\partial w_1}{\partial W_{ijk}}
    \end{pmatrix} = - \frac{1}{\sqrt n} \left( 
    \begin{pmatrix}
        0 & \gT_1(w_1) & \gT_1 (v_1)\\
        \gT_1(w_1)^\top & 0 & \gT_1(u_1)\\
        \gT_1 (v_1)^\top & \gT_1(u_1)^\top & 0
    \end{pmatrix} - \lambda_1 \mI_n
    \right)^{-1} \begin{pmatrix}
        v_{1j} w_{1k} (e_i - u_{1i} u_1)\\
        u_{1i} w_{1k} (e_j - v_{1j} v_1)\\
        u_{1i} v_{1j} (e_k - w_{1k} w_1)
    \end{pmatrix}
\end{align}
which results from deriving the identities in \eqref{eq_id_1} w.r.t.\@ the entry $W_{ijk}$ of the noise tensor $\gW$. In particular, as demonstrated in \cite{seddik2021random}, the only contributing terms in the quantities we will compute later on will depend only on traces of the resolvent matrix appearing in \eqref{eq_deriv}.

\subsubsection{Limiting Stieltjes transform}
Since the tensor $\gT_1$ is a low-rank perturbation of a random tensor $\gW$, by Lemma \ref{lemma_perturbation}, the normalized trace of the resolvent in \eqref{eq_deriv} is asymptotically equal to the normalized trace of the resolvent of the following random matrix
\begin{align}
    \mN = \frac{1}{\sqrt n} \begin{pmatrix}
    0 & \gW(w_1) & \gW(v_1)\\
    \gW(w_1)^\top & 0 & \gW(u_1) \\
    \gW(v_1)^\top & \gW(u_1)^\top  & 0
    \end{pmatrix}
\end{align}
Let $R(z) = (\mN - z \mI_n)^{-1}$ be the corresponding resolvent. We denote the different sub-blocks of $R(z)$ as
\begin{align}
    R(z) = \begin{pmatrix}
 R^{11}(z) & R^{12}(z) & R^{13}(z) \\ 
 R^{12}(z)^\top & R^{22}(z) & R^{23}(z) \\ 
 R^{13}(z)^\top & R^{23}(z)^\top & R^{33}(z)
\end{pmatrix}
\end{align}
It has been shown in \citep[Appendix B.2]{seddik2021random} that
\begin{align}
    \frac1n \Tr R^{ii}(z) \xrightarrow[n\to\infty]{} r_i(z) = \frac{r(z)}{3} \quad \text{and} \quad \frac1n \Tr R(z) \xrightarrow[n\to\infty]{} r(z)
\end{align}
with
\begin{align}
    \boxed{r(z) = \frac34 \left( -z + \sqrt{z^2 - \frac83} \right)}
\end{align}

\subsubsection{Estimation of the singular value}\label{proof_singular_value_first_deflation}

\paragraph{Estimation of $\lambda_1$:}
From the identities in \eqref{eq_id_1}, we have
\begin{align*}
    \lambda_1 = \gT_1(u_1, v_1, w_1) = \gS(u_1, v_1, w_1)  + \frac{1}{\sqrt n} \gW(u_1, v_1, w_1)
\end{align*}
and 
\begin{align*}
    &\frac{1}{\sqrt n}\E \left[ \gW(u_1, v_1, w_1) \right] = \frac{1}{\sqrt n } \sum_{ijk} \E [ u_{1i} v_{1j} w_{1k} W_{ijk}]\\
    &= \frac{1}{\sqrt n} \sum_{ijk} \E \left[ \frac{\partial u_{1i}}{ \partial W_{ijk} } v_{1j} w_{1k} \right] + \E \left[ u_{1i} \frac{\partial v_{1j} }{ \partial W_{ijk} }  w_{1k} \right] + \E \left[ u_{1i} v_{1j} \frac{\partial w_{1k} }{ \partial W_{ijk} }   \right]
\end{align*}
where the last equality is derived from Stein's Lemma and the involved derivatives express as
\begin{align*}
    \frac{\partial u_{1i}}{ \partial W_{ijk} }  \simeq \frac{-1}{\sqrt n} v_{1j} w_{1k} R^{11}_{ii}(\lambda_1), \quad \frac{\partial v_{1j}}{ \partial W_{ijk} }  \simeq \frac{-1}{\sqrt n} u_{1i} w_{1k} R^{22}_{jj}(\lambda_1), \quad\frac{\partial w_{1k}}{ \partial W_{ijk} }  \simeq \frac{-1}{\sqrt n} u_{1i} v_{1j} R^{33}_{kk}(\lambda_1)
\end{align*}
Substituting in the above sum, we get
\begin{align*}
    \frac{1}{\sqrt n}\E \left[ \gW(u_1, v_1, w_1) \right] &\simeq -\frac{1}{n} \sum_{ijk} \E [ v_{1j}^2 w_{1k}^2 R^{11}_{ii}(\lambda_1) ] -\frac{1}{n} \sum_{ijk} \E [ u_{1i}^2 w_{1k}^2 R^{22}_{jj}(\lambda_1) ] -\frac{1}{n} \sum_{ijk} \E [ u_{1i}^2 v_{1j}^2 R^{33}_{kk}(\lambda_1) ]\\
    &= - \E\left[ \frac1n \Tr R^{11}(\lambda_1) \right] - \E\left[ \frac1n \Tr R^{22}(\lambda_1) \right] - \E\left[ \frac1n \Tr R^{33}(\lambda_1) \right]\\
    &\xrightarrow[n\to\infty]{} -(r_1(\lambda_1) + r_2(\lambda_1) + r_3(\lambda_1)) = -r(\lambda_1)
\end{align*}
Therefore, we have
\begin{align}
    \lambda_1 + r(\lambda_1) = \gS(u_1, v_1, w_1) 
\end{align}
\subsubsection{Estimation of the alignments}\label{proof_alignments_first_deflation}
\paragraph{Estimation of $\langle u_1, x_s \rangle$:} Again from the first identity in \eqref{eq_id_1}, we have
\begin{align*}
    \lambda_1 \langle u_1, x_s \rangle &= \gT_1 (x_s, v_1, w_1)= \gS(x_s, v_1, w_1)   + \frac{1}{\sqrt n} \gW(x_s, v_1, w_1)
\end{align*}
And we have
\begin{align*}
    &\frac{1}{\sqrt n} \E [\gW(x_s, v_1, w_1)] = \frac{1}{\sqrt n} \sum_{ijk} \E \left[ x_{si} v_{1j} w_{1k} W_{ijk} \right]\\
    &= \frac{1}{\sqrt n} \sum_{ijk} \E \left[ x_{si} \frac{\partial v_{1j} }{ \partial W_{ijk} }  w_{1k}  \right] + \frac{1}{\sqrt n} \sum_{ijk} \E \left[ x_{si} v_{1j} \frac{\partial  w_{1k}  }{ \partial W_{ijk} }   \right]\\
    &\simeq -\frac1n \sum_{ijk} \E \left[ x_{si} u_{1i} w_{1k}^2 R^{22}_{jj}(\lambda_1) \right] -\frac1n \sum_{ijk} \E \left[ x_{si} u_{1i} v_{1j}^2 R^{33}_{kk}(\lambda_1) \right]\\
    &= - \E \left[ \langle x_s, u_1 \rangle \frac1n \Tr R^{22}(\lambda_1) \right] - \E \left[ \langle x_s, u_1 \rangle \frac1n \Tr R^{33}(\lambda_1) \right]\\
    &\xrightarrow[n\to\infty]{} -(r_2(\lambda_1) + r_3(\lambda_1)) \langle x_s, u_1 \rangle = -(r(\lambda_1) - r_1(\lambda_1) ) \langle x_s, u_1 \rangle
\end{align*}
Therefore, we have
\begin{align}
    (\lambda_1 + r(\lambda_1) - r_1(\lambda_1) ) \langle x_s, u_1 \rangle = \gS(x_s, v_1, w_1)
\end{align}
Similarly, we get
\begin{equation}
    \begin{split}
        (\lambda_1 + r(\lambda_1) - r_2(\lambda_1) ) \langle y_s, v_1 \rangle &= \gS(u_1, y_s, w_1)\\
    (\lambda_1 + r(\lambda_1) - r_3(\lambda_1) ) \langle z_s, w_1 \rangle &= \gS(u_1, v_1, z_s)
    \end{split}
\end{equation}

Finally, with our assumption in \eqref{assum_alpha} and since $\gT_1$ is cubic, the above equations reduce to the following system of equations describing the first deflation step
\begin{equation}
    \boxed{
    \begin{cases}
        f_r(\lambda_1) = \sum_{i=1}^2 \beta_i \rho_{1i}^3\\
        h_r(\lambda_1) \rho_{1j} = \sum_{i=1}^2 \beta_i \langle x_i, x_j \rangle \rho_{1i}^2\quad \text{for}\quad j \in [2]
    \end{cases}
    }
\end{equation}
where we denoted $f_r(z) = z + r(z)$ and $h_r(z) = z + \frac23 r(z) = -\frac{1}{r(z)}$.
\subsection{Second deflation step}\label{proof_second_deflation}
Given $u_1$ from the first deflation step, we consider now the following random tensor 
\begin{align}
    \gT_2 = \gT_1 \times_1 \left( \mI_N - { \gamma} u_1 u_1^\top \right) = \gT_1 - { \gamma} u_1 \otimes \gT_1(u_1, \cdot, \cdot)
\end{align}
Again the singular vectors of $\gT_2$ satisfy
\begin{align}\label{eq_id_2}
\gT_2(\cdot, v_2, w_2) = \lambda_2 u_2, \quad  \gT_2(u_2, \cdot, w_2) = \lambda_2 v_2, \quad \gT_2(u_2, v_2, \cdot) = \lambda_2 w_2
\end{align}
and we also have
\begin{align}
    \begin{pmatrix}
        \frac{\partial u_2}{\partial W_{ijk}}\\
        \frac{\partial v_2}{\partial W_{ijk}}\\
        \frac{\partial w_2}{\partial W_{ijk}}
    \end{pmatrix} = - \frac{1}{\sqrt n} \left( 
    \begin{pmatrix}
        0 & \gT_2(w_2) & \gT_2 (v_2)\\
        \gT_2(w_2)^\top & 0 & \gT_2(u_2)\\
        \gT_2 (v_2)^\top & \gT_2(u_2)^\top & 0
    \end{pmatrix} - \lambda_2 \mI_n
    \right)^{-1} \begin{pmatrix}
        v_{2j} w_{2k} (e_i - u_{2i} u_2)\\
        u_{2i} w_{2k} (e_j - v_{2j} v_2)\\
        u_{2i} v_{2j} (e_k - w_{2k} w_2)
    \end{pmatrix}
\end{align}
\subsubsection{Stieltjes transform}\label{proof_limiting_measure_second_deflation}
Again, since $\gT_1$ is a low-rank perturbation of a random tensor $\gW$, it is easily noticed that
\begin{align*}
    \begin{pmatrix}
        0 & \gT_2(w_2) & \gT_2 (v_2)\\
        \gT_2(w_2)^\top & 0 & \gT_2(u_2)\\
        \gT_2 (v_2)^\top & \gT_2(u_2)^\top & 0
    \end{pmatrix} = \mM + \mP
\end{align*}
where $\mP$ is some low-rank matrix and $\mM$ is a random matrix given by
\begin{align}
    \mM = \frac{1}{\sqrt n} \begin{pmatrix}
    0 & \gW(w_2) & \gW(v_2)\\
    \gW(w_2)^\top & 0 & \gW(u_2) - \gamma \langle u_1, u_2 \rangle \gW(u_1)\\
    \gW(v_2)^\top & \gW(u_2)^\top - \gamma \langle u_1, u_2 \rangle \gW(u_1)^\top & 0
    \end{pmatrix}
\end{align}
Therefore, by Lemma \ref{lemma_perturbation}, the limiting Stieltjes transform corresponding to the analysis of the second deflation step can be computed through the resolvent $Q (z)= (\mM - z \mI_n)^{-1}$ of the random matrix $\mM$ and we denote $\kappa = \langle u_1, u_2 \rangle$. We denote the different sub-blocks of $Q(z)$ as
\begin{align}
    Q(z) = \begin{pmatrix}
 Q^{11}(z) & Q^{12}(z) & Q^{13}(z) \\ 
 Q^{12}(z)^\top & Q^{22}(z) & Q^{23}(z) \\ 
 Q^{13}(z)^\top & Q^{23}(z)^\top & Q^{33}(z)
\end{pmatrix}
\end{align}
Denote
\begin{align}
    \frac1n \Tr Q^{ii}(z) \xrightarrow[n\to\infty]{} q_i(z) \quad \text{and}\quad  \frac1n \Tr Q(z) \xrightarrow[n\to\infty]{} q(z)
\end{align}
\paragraph{Estimation of $\frac1n \Tr Q^{11}(z)$:}
From the identity $\mM Q(z) - zQ(z) = \mI_n$, we have
\begin{align*}
    \frac{1}{\sqrt n} \left[ \gW(w_2)  (Q^{12})^\top \right]_{ii} + \frac{1}{\sqrt n} \left[ \gW(v_2)  (Q^{13})^\top \right]_{ii} - z Q^{11}_{ii} = 1
\end{align*}
Therefore
\begin{align*}
    \frac{1}{n\sqrt n} \sum_{ijk} \E \left[ w_{2k} W_{ijk} Q^{12}_{ij} \right] + \frac{1}{n\sqrt n} \sum_{ijk} \E \left[ v_{2j} W_{ijk} Q^{13}_{ik} \right] - \frac{z}{n}\Tr Q^{11}(z) = \frac13
\end{align*}
where
\begin{align*}
    \bullet \quad\frac{1}{n\sqrt n} \sum_{ijk} \E \left[ w_{2k} W_{ijk} Q^{12}_{ij} \right] \simeq  \frac{1}{n\sqrt n} \sum_{ijk} \E \left[ w_{2k} \frac{ \partial Q^{12}_{ij} }{ \partial W_{ijk} } \right]
\end{align*}
From \cite{seddik2021random}, we have $\frac{ \partial Q^{12}_{ij} }{ \partial W_{ijk} } \simeq  -\frac{1}{\sqrt n} w_{2k} Q^{11}_{ii} Q^{22}_{jj} $, hence
\begin{align*}
    \frac{1}{n\sqrt n} \sum_{ijk} \E \left[ w_{2k} W_{ijk} Q^{12}_{ij} \right] \simeq - \frac{1}{n^2} \sum_{ijk} \E \left[ w_{2k}^2 Q^{11}_{ii} Q^{22}_{jj}  \right] = - 
 \E\left[ \frac1n \Tr Q^{11} \frac1n \Tr Q^{22} \right] \xrightarrow[n\to\infty]{} - q_1(z) q_2(z)
\end{align*}
Similarly, we have
\begin{align*}
    \bullet \quad\frac{1}{n\sqrt n} \sum_{ijk} \E \left[ v_{2j} W_{ijk} Q^{13}_{ik} \right] \xrightarrow[n\to\infty]{} - q_1(z) q_3(z)
\end{align*}
Therefore, $q_1(z) = \lim_{n\to \infty} \frac1n \Tr Q^{11}(z)$ satisfies the equation
\begin{align}
    [q_2(z) + q_3(z) + z] q_1(z) +\frac13 = 0
\end{align}

\paragraph{Estimation of $\frac1n \Tr Q^{22}(z)$:} We have
\begin{align*}
    \frac{1}{\sqrt n} \left[ \gW(w_2)^\top Q^{12} \right]_{jj} + \frac{1}{\sqrt n} \left[ \left( \gW(u_2) - \gamma \kappa \gW(u_1)  \right) (Q^{23})^\top \right]_{jj} - z Q^{22}_{jj} = 1
\end{align*}
Hence
\begin{align*}
    \frac{1}{n\sqrt n} \sum_{ijk} \E \left[ w_{2k} W_{ijk} Q^{12}_{ij} \right] + \frac{1}{n \sqrt n} \sum_{ijk} \E \left[ u_{2i} W_{ijk} Q^{23}_{jk} \right] - \frac{\gamma
    \kappa}{n \sqrt n} \sum_{ijk} \E \left[ u_{1i} W_{ijk} Q^{23}_{jk} \right] - \frac{z}{n}\Tr Q^{22} = \frac13
\end{align*}
where
 \begin{align*}
    \bullet \quad\frac{1}{n\sqrt n} \sum_{ijk} \E \left[ w_{2k} W_{ijk} Q^{12}_{ij} \right] \simeq \frac{1}{n\sqrt n} \sum_{ijk} \E \left[ w_{2k} \frac{\partial Q^{12}_{ij}}{\partial W_{ijk}  } \right] = -\frac{1}{n^2} \sum_{ijk} \E \left[ w_{2k}^2 Q^{11}_{ii} Q^{22}_{jj} \right] \xrightarrow[n\to\infty]{} - q_1(z) q_2(z)
\end{align*}
\begin{align*}
    \bullet \quad \frac{1}{n \sqrt n} \sum_{ijk} \E \left[ u_{2i} W_{ijk} Q^{23}_{jk} \right]\simeq \frac{1}{n \sqrt n} \sum_{ijk} \E \left[ u_{2i} \frac{\partial Q^{23}_{jk}}{\partial W_{ijk} } \right] = -\frac{1}{n^2} \sum_{ijk} \E \left[ (u_{2i}^2 - \gamma\kappa u_{1i}u_{2i}) Q^{22}_{jj} Q^{33}_{kk} \right]\xrightarrow[n\to\infty]{} (\gamma\kappa^2 - 1) q_2(z) q_3(z)
\end{align*}
\begin{align*}
    \bullet \quad \frac{1}{n \sqrt n} \sum_{ijk} \E \left[ u_{1i} W_{ijk} Q^{23}_{jk} \right] \simeq \frac{1}{n \sqrt n} \sum_{ijk} \E \left[ u_{1i} \frac{\partial Q^{23}_{jk}}{\partial W_{ijk} } \right] = -\frac{1}{n^2} \sum_{ijk} \E \left[ (u_{1i} u_{2i} - \gamma\kappa u_{1i}^2) Q^{22}_{jj} Q^{33}_{kk} \right]\xrightarrow[n\to\infty]{}  \kappa(\gamma -1) q_2(z) q_3(z)
\end{align*}
where we used the fact that $ \frac{\partial Q^{23}_{jk}}{\partial W_{ijk} } \simeq -\frac{1}{\sqrt n}(u_{2i} - \gamma \kappa u_{1i} ) Q^{22}_{jj} Q^{33}_{kk} $.

\begin{align}
    \left(q_1(z) + z - \left[\gamma\kappa^2-1+\kappa(\gamma-1)\right] q_3(z)\right) q_2(z) +\frac13 = 0
\end{align}

\paragraph{Estimation of $\frac1n \Tr Q^{33}(z)$:}

From the identity $\mM Q(z) - zQ(z) = \mI_n$, we have
\begin{align*}
    \frac{1}{\sqrt n} \left[ \gW(v_2)^{\top}  Q^{13} \right]_{kk} + \frac{1}{\sqrt n} \left[ \left(\gW(u_2)^\top-\gamma\kappa \gW(u_1)^\top\right)  (Q^{23}) \right]_{kk} - z Q^{33}_{kk} = 1
\end{align*}
Hence
\begin{align*}
\frac{1}{n\sqrt n} \sum_{ijk} \E \left[ v_{2j} W_{ijk} Q^{13}_{ik} \right] + \frac{1}{n\sqrt n} \sum_{ijk} \E \left[ u_{2i} W_{ijk} Q^{23}_{jk} \right] - \gamma\kappa \frac{1}{n\sqrt n} \sum_{ijk} \E \left[ u_{1i} W_{ijk} Q^{23}_{jk} \right] - \frac{z}{n} \Tr{Q^{33}(z)}=\frac{1}{3}
\end{align*}
where 
\begin{align*}
\frac{1}{n\sqrt n} \sum_{ijk} \E \left[ v_{2j} W_{ijk} Q^{13}_{ik} \right] \simeq \frac{1}{n\sqrt n} \sum_{ijk} \E \left[v_{2j}\frac{\partial Q_{ik}^{13}}{\partial W_{ijk}}\right] = -\frac{1}{n^2}\sum_{ijk} \E \left[v_{2j}^2 Q_{ii}^{11} Q_{kk}^{33}\right] \xrightarrow[n\to\infty]{} -q_1(z) q_3(z)
\end{align*}
\begin{align*}
\frac{1}{n\sqrt n} \sum_{ijk} \E \left[ u_{2i} W_{ijk} Q^{23}_{jk} \right] \simeq \frac{1}{n\sqrt n} \sum_{ijk} \E \left[u_{2i}\frac{\partial Q_{jk}^{23}}{\partial W_{ijk}}\right] = -\frac{1}{n^2} \sum_{ijk} \E \left[ \left(u_{2i}^2-\gamma\kappa u_{1i}u_{2i}\right)Q_{jj}^{22}Q_{kk}^{33}\right] \xrightarrow[n\to\infty]{} \left(\gamma\kappa^2-1\right)q_2(z) q_3(z)
\end{align*}
\begin{align*}
\frac{1}{n\sqrt n} \sum_{ijk} \E \left[ u_{1i} W_{ijk} Q^{23}_{jk} \right] \simeq \frac{1}{n\sqrt n} \sum_{ijk} \E \left[u_{1i}\frac{\partial Q_{jk}^{23}}{\partial W_{ijk}}\right] = -\frac{1}{n^2} \sum_{ijk} \E \left[ (u_{1i} u_{2i} - \gamma\kappa u_{1i}^2) Q^{22}_{jj} Q^{33}_{kk} \right]\xrightarrow[n\to\infty]{}  \kappa(\gamma -1) q_2(z) q_3(z) 
\end{align*}
with again $ \frac{\partial Q^{23}_{jk}}{\partial W_{ijk} } \simeq -\frac{1}{\sqrt n}(u_{2i} - \gamma \kappa u_{1i} ) Q^{22}_{jj} Q^{33}_{kk} $.

\begin{align}
    \left(q_1(z) + z - \left[\gamma\kappa^2-1+\kappa(\gamma-1)\right] q_2(z)\right) q_3(z) +\frac13 = 0
\end{align}

Therefore, we have
\begin{equation}
    \begin{cases}
        [q_2(z) + q_3(z) + z] q_1(z) +\frac13 = 0\\
        \left(q_1(z) + z - \left[\gamma\kappa^2-1+\kappa(\gamma-1)\right] q_3(z)\right) q_2(z) +\frac13 = 0\\
        \left(q_1(z) + z - \left[\gamma\kappa^2-1+\kappa(\gamma-1)\right] q_2(z)\right) q_3(z) +\frac13 = 0\\
        q(z) = \sum_{i=1}^3 q_i(z)
    \end{cases}
\end{equation}

Moreover, by symmetry in \eqref{assum_alpha} and since $\gT_1$ is cubic, we have $b(z) = q_2(z) = q_3(z) $ and we denote $a(z) = q_1(z)$ and $\tau = \gamma\kappa^2-1+\kappa(\gamma-1) $. Hence,
\begin{align}
    \boxed{
    \begin{cases}
        \left[ 2b(z) +z \right] a(z) + \frac13 = 0\\
        (a(z) + z - \tau b(z)) b(z) + \frac13 = 0
    \end{cases}
    }
\end{align}
Moreover, $q(z) = a(z) + 2 b(z)$.

\subsubsection{Estimation of the singular value}\label{proof_singular_value_second_deflation}

\paragraph{Estimation of $\lambda_2$:}
We first have
\begin{align*}
    \lambda_2 &= \gT_2(u_2, v_2, w_2) = \gT_1(u_2, v_2, w_2) - {  \gamma} \langle u_1, u_2 \rangle \gT_1(u_1, v_2, w_2)\\
    &= \gS(u_2, v_2, w_2) + \frac{1}{\sqrt n} \gW(u_2, v_2, w_2)-  {  \gamma} \langle u_1, u_2 \rangle \left(  \gS(u_1, v_2, w_2)+ \frac{1}{\sqrt n} \gW(u_1, v_2, w_2)  \right)
\end{align*}
where we have
\begin{align*}
    \frac{1}{\sqrt n}\E \left[ \gW(u_2, v_2, w_2) \right]  \xrightarrow[n\to\infty]{} -q(\lambda_2)
\end{align*}
and
\begin{align*}
   &\frac{1}{\sqrt n}\E \left[ \gW(u_1, v_2, w_2) \right] = \frac{1}{\sqrt n} \sum_{ijk} \E \left[ u_{1i} v_{2j} w_{2k} W_{ijk}\right]\\
   &= \frac{1}{\sqrt n} \sum_{ijk} \E \left[ \frac{ \partial u_{1i} }{ \partial W_{ijk} } v_{2j} w_{2k} \right] + \E \left[ u_{1i} \frac{ \partial  v_{2j} }{ \partial W_{ijk} }  w_{2k} \right] + \E \left[ u_{1i} v_{2j} \frac{ \partial  w_{2k} }{ \partial W_{ijk} }   \right]
\end{align*}
Again, we have
\begin{align*}
    \frac{\partial u_{2i}}{ \partial W_{ijk} }  &\simeq \frac{-1}{\sqrt n} v_{2j} w_{2k} Q^{11}_{ii}(\lambda_2), \quad \frac{\partial v_{2j}}{ \partial W_{ijk} }  \simeq \frac{-1}{\sqrt n} u_{2i} w_{2k} Q^{22}_{jj}(\lambda_2), \quad\frac{\partial w_{2k}}{ \partial W_{ijk} }  \simeq \frac{-1}{\sqrt n} u_{2i} v_{2j} Q^{33}_{kk}(\lambda_2)
\end{align*}
Therefore
\begin{align*}
    \frac{1}{\sqrt n}\E \left[ \gW(u_1, v_2, w_2) \right] &\simeq -\frac1n \sum_{ijk} \E \left[ v_{1j}w_{1k} v_{2j} w_{2k} R^{11}_{ii}(\lambda_1) \right] -\frac1n \sum_{ijk} \E \left[u_{1i} u_{2i} w_{2k}^2 Q^{22}_{jj}(\lambda_2)  \right] - \frac1n \sum_{ijk} \E \left[u_{1i} v_{2j}^2 u_{2i} Q^{33}_{kk}(\lambda_2) \right]\\
    &\xrightarrow[n\to\infty]{} - \langle v_1, v_2 \rangle \langle w_1, w_2 \rangle r_1(\lambda_1) - \langle u_1, u_2 \rangle q_2(\lambda_2)  - \langle u_1, u_2 \rangle q_3(\lambda_2) \\
    &= - \langle v_1, v_2 \rangle \langle w_1, w_2 \rangle r_1(\lambda_1) - \langle u_1, u_2 \rangle \left[ q_2(\lambda_2)  +  q_3(\lambda_2) \right]
\end{align*}
Hence, $\lambda_2$ satisfies
\begin{equation}
    \begin{split}
        &\lambda_2 + q(\lambda_2) -  {  \gamma} \langle u_1, u_2 \rangle \langle v_1, v_2 \rangle \langle w_1, w_2 \rangle r_1(\lambda_1) -  {  \gamma} \langle u_1, u_2 \rangle^2 \left[ q_2(\lambda_2)  +  q_3(\lambda_2) \right]\\
    &= \gS(u_2, v_2, w_2)  -   {  \gamma} \langle u_1, u_2 \rangle\gS(u_1, v_2, w_2)
    \end{split}
\end{equation}
And by symmetry, from \eqref{assum_alpha} and since $\gT_1$ is cubic, we have
\begin{equation}
    \boxed{
    \begin{split}
        &f_q(\lambda_2) -   \frac{\gamma\kappa \eta^2}{3} r(\lambda_1) -   2\gamma \kappa^2 b(\lambda_2)= \sum_{i=1}^2 \beta_i \theta_{2i} \rho_{2i}^2  -    \gamma \kappa \sum_{i=1}^2 \beta_i \rho_{1i} \rho_{2i}^2
    \end{split}
    }
\end{equation}
where we denoted $f_q(z) = z + q(z)$ and
\begin{align*}
    \theta_{2i} &= \langle u_2, x_i\rangle ,\quad \rho_{2i} = \langle v_2, y_i\rangle = \langle w_2, z_i\rangle, \quad \kappa = \langle u_1, u_2\rangle, \quad \eta = \langle v_1, v_2\rangle = \langle w_1, w_2\rangle
\end{align*}

\subsubsection{Estimation of the alignments}\label{proof_alignments_second_deflation}
\paragraph{Estimation of $\langle u_2, x_s \rangle$:} From the identity in \eqref{eq_id_2}, we have
\begin{align*}
    &\lambda_2 \langle u_2, x_s \rangle = \gT_2(x_s, v_2, w_2) = \gT_1(x_s, v_2, w_2) - {  \gamma} \langle u_1, x_s \rangle \gT_1(u_1, v_2, w_2)\\
    &=\gS(x_s, v_2, w_2) + \frac{1}{\sqrt n} \gW(x_s, v_2, w_2) - {  \gamma} \langle u_1, x_s \rangle \left( \gS(u_1, v_2, w_2) + \frac{1}{\sqrt n} \gW(u_1, v_2, w_2) \right)
\end{align*}
where
\begin{align*}
    \frac{1}{\sqrt n}\E \left[ \gW(x_s, v_2, w_2) \right] \xrightarrow[n\to\infty]{} -(q(\lambda_2) - q_1(\lambda_2)) \langle x_s, u_2 \rangle
\end{align*}
and $\E \left[ \frac{1}{\sqrt n} \gW(u_1, v_2, w_2) \right]$ was computed previously. We therefore have
\begin{equation}
    \begin{split}
        &[\lambda_2 + q(\lambda_2) - q_1(\lambda_2)] \langle x_s, u_2 \rangle - {  \gamma} \langle u_1, x_s \rangle \left[ \langle v_1, v_2 \rangle \langle w_1, w_2 \rangle r_1(\lambda_1) +  \langle u_1, u_2 \rangle (q_2(\lambda_2) + q_3(\lambda_2))  \right]\\
        &=\gS(x_s, v_2, w_2) - {  \gamma} \langle x_s, u_1\rangle \gS(u_1, v_2, w_2)
    \end{split}
\end{equation}
Again by symmetry, from \eqref{assum_alpha} and since $\gT_1$ is cubic, we have
\begin{equation}
\boxed{
    \begin{split}
        &[f_q (\lambda_2) - a(\lambda_2)] \theta_{2s} - \gamma \rho_{1s} \left[ \frac{\eta^2}{3} r(\lambda_1) +  2\kappa b(\lambda_2)  \right]=\sum_{i=1}^2 \beta_i \langle x_s, x_i \rangle \rho_{2i}^2 -  \gamma \rho_{1s} \sum_{i=1}^2 \beta_i \rho_{1i} \rho_{2i}^2 \quad \text{for}\quad s\in [2]
    \end{split}
}
\end{equation}

\paragraph{Estimation of $\langle u_1,  u_2 \rangle$:} Again from \eqref{eq_id_2}, we have
\begin{align*}
    \lambda_2 \langle u_1, u_2 \rangle = \gT_2(u_1, v_2, v_2)
\end{align*}
with $\gT_2 = \gT_1 - {  \gamma} u_1 \otimes \gT_1(u_1, \cdot, \cdot) $, therefore
\begin{equation}\label{eq_kappa_zero}
    \begin{split}
        \lambda_2 \langle u_1, u_2 \rangle &= \gT_1(u_1, v_2, w_2) - {  \gamma} \gT_1(u_1, v_2, w_2) = (1 - {  \gamma}) \gT_1(u_1, v_2, w_2)\\
    &= (1 - {  \gamma})  \left( \gS(u_1, v_2, w_2) + \frac{1}{\sqrt n} \gW(u_1, v_2, w_2) \right)
    \end{split}
\end{equation}

Hence, we have
\begin{equation}
    \begin{split}
        &\left[ \lambda_2 + (1 - {  \gamma}) (q_2(\lambda_2) + q_3(\lambda_2))  \right] \langle u_1, u_2 \rangle = (1 - {  \gamma}) \left[ \gS(u_1, v_2, w_2) - \langle v_1, v_2 \rangle \langle w_1, w_2 \rangle r_1(\lambda_1) \right]
    \end{split}
\end{equation}
Again by symmetry, from \eqref{assum_alpha} and since $\gT_1$ is cubic, we have
\begin{equation}
    \boxed{
    \begin{split}
        &\left[ \lambda_2 + 2(1 - {\gamma}) b(\lambda_2)  \right] \kappa= (1 - {\gamma}) \left[\sum_{i=1}^2 \beta_i \rho_{1i} \rho_{2i}^2 - \frac{\eta^2}{3} r(\lambda_1) \right]
    \end{split}
    }
\end{equation}

\paragraph{Estimation of $\langle v_2,  y_s \rangle$:} From \eqref{eq_id_2}, we have
\begin{align*}
    \lambda_2 \langle v_2, y_s \rangle &= \gT_2(u_2, y_s, w_2) = \gT_1(u_2, y_s, w_2) -  {  \gamma}\langle u_1, u_2 \rangle  \gT_1(u_1, y_s, w_2)\\
    &= \gS(u_2, y_s, w_2) + \frac{1}{\sqrt n} \gW(u_2, y_s, w_2) - {  \gamma}\langle u_1, u_2 \rangle \left[ \gS(u_1, y_s, w_2) +  \frac{1}{\sqrt n} \gW(u_1, y_s, w_2)  \right]
\end{align*}
And as previously, we have
\begin{align*}
    \E \left[ \frac{1}{\sqrt n} \gW(u_2, y_s, w_2) \right] \xrightarrow[n\to\infty]{} -(q(\lambda_2) - q_2(\lambda_2)) \langle y_s, v_2 \rangle
\end{align*}
And
\begin{align*}
    &\E \left[ \frac{1}{\sqrt n} \gW(u_1, y_s, w_2)\right] = \frac{1}{\sqrt n} \sum_{ijk} y_{sj} \E \left[ \frac{\partial u_{1i} }{ \partial W_{ijk} } w_{2k} + u_{1i} \frac{\partial  w_{2k} }{ \partial W_{ijk} }    \right] \\
    &\simeq -\frac1n \sum_{ijk} y_{sj} \E \left[ v_{1j} w_{1k} R_{ii}^{11}(\lambda_1) w_{2k} + u_{1i} u_{2i} v_{2j} Q_{kk}^{33}(\lambda_2) \right]\\
    &\xrightarrow[n\to\infty]{} - \langle y_s, v_1 \rangle \langle w_1, w_2 \rangle r_1(\lambda_1) - \langle y_s, v_2 \rangle \langle u_1, u_2 \rangle q_3(\lambda_2)
\end{align*}
Therefore,
\begin{equation}
    \begin{split}
         &\left[ \lambda_2 + q(\lambda_2) - q_2(\lambda_2) -  {  \gamma}  \langle u_1, u_2 \rangle^2 q_3(\lambda_2)\right]  \langle v_2, y_s \rangle =  \\
    &\gS(u_2, y_s, w_2) - {  \gamma}  \langle u_1, u_2 \rangle  \left[ \gS(u_1, y_s, w_2) - \langle v_1, y_s \rangle \langle w_1, w_2 \rangle r_1(\lambda_1) \right] 
    \end{split}
\end{equation}
Again by symmetry, from \eqref{assum_alpha} and since $\gT_1$ is cubic, we have
\begin{equation}
    \boxed{ 
    \begin{split}
         &\left[ f_q(\lambda_2) - (1 + \gamma  \kappa^2) b(\lambda_2)\right]  \rho_{2s} =  \sum_{i=1}^2 \beta_i \theta_{2i} \rho_{2i} \langle y_s, y_i \rangle  -  \gamma  \kappa  \left[ \sum_{i=1}^2 \beta_i \rho_{1i} \rho_{2i} \langle y_s, y_i \rangle  - \frac{\rho_{1s} \eta}{3} r(\lambda_1) \right] \quad \text{for}\quad s\in [2]
    \end{split}
    }
\end{equation}

\paragraph{Estimation of $\langle v_1,  v_2 \rangle$:} From \eqref{eq_id_2}, we have
\begin{align*}
    \lambda_2 \langle v_1, v_2 \rangle &= \gT_2(u_2, v_1, w_2) = \gT_1(u_2, v_1, w_2) -  {  \gamma} \langle u_1, u_2 \rangle \gT_1(u_1, v_1, w_2)\\
    &= \gS(u_2, v_1, w_2)+ \frac{1}{\sqrt n} \gW(u_2, v_1, w_2) -  {  \gamma} \langle u_1, u_2 \rangle \left[ \gS(u_1, v_1, w_2) + \frac{1}{\sqrt n} \gW(u_1, v_1, w_2) \right]
\end{align*}
Where
\begin{align*}
    &\E \left[ \frac{1}{\sqrt n} \gW(u_2, v_1, w_2) \right] = \frac{1}{\sqrt n} \sum_{ijk} \E [ u_{2i} v_{1j} w_{2k} W_{ijk} ]\\
    &=\frac{1}{\sqrt n} \sum_{ijk}  \E \left[ \frac{ \partial u_{2i} }{ \partial W_{ijk} } v_{1j} w_{2k} +  u_{2i} \frac{ \partial  v_{1j} }{ \partial W_{ijk} }  w_{2k}  + u_{2i}  v_{1j} \frac{ \partial  w_{2k}  }{ \partial W_{ijk} }   \right]\\
    &\simeq -\frac1n \sum_{ijk} \E \left[ v_{2j} w_{2k} Q^{11}_{ii}(\lambda_2) v_{1j} w_{2k} + u_{2i} u_{1i} w_{1k} R^{22}_{jj}(\lambda_1) w_{2k} + u_{2i} v_{1j} u_{2i} v_{2j} Q^{33}_{kk}(\lambda_2) \right]\\
    &\xrightarrow[n\to\infty]{} - \langle v_1, v_2 \rangle [q_1(\lambda_2) + q_3(\lambda_2)] - \langle u_1, u_2 \rangle \langle w_1, w_2 \rangle r_2(\lambda_1) 
\end{align*}
And
\begin{align*}
    &\E \left[ \frac{1}{\sqrt n} \gW(u_1, v_1, w_2) \right] = \frac{1}{\sqrt n} \sum_{ijk} \E [ u_{1i} v_{1j} w_{2k} W_{ijk} ]\\
    &= \frac{1}{\sqrt n} \sum_{ijk} \E \left[ \frac{\partial u_{1i}}{\partial W_{ijk} } v_{1j} w_{2k} + u_{1i} \frac{\partial  v_{1j} }{\partial W_{ijk} } w_{2k} + u_{1i} v_{1j} \frac{\partial  w_{2k} }{\partial W_{ijk} }  \right]\\
    &\simeq -\frac1n \sum_{ijk} \E \left[ v_{1j}w_{1k} R^{11}_{ii}(\lambda_1) v_{1j} w_{2k} + u_{1i}^2 w_{1k} w_{2k} R^{22}_{jj}(\lambda_1) + u_{1i} v_{1j} u_{2i} v_{2j} Q^{33}_{kk}(\lambda_2)   \right]\\
    &\xrightarrow[n\to\infty]{} - \langle w_1, w_2 \rangle [r_1(\lambda_1) + r_2(\lambda_1)] - \langle u_1, u_2 \rangle \langle v_1, v_2 \rangle q_3(\lambda_2)
\end{align*}
Hence, we have
\begin{equation}
    \begin{split}
        &\left[ \lambda_2 + q_1(\lambda_2) + q_3(\lambda_2) - \gamma \langle u_1, u_2 \rangle^2 q_3(\lambda_2) \right] \langle v_1, v_2 \rangle + \left[ (1-\gamma) r_2(\lambda_1) - r_1(\lambda_1) \right] \langle u_1, u_2 \rangle \langle w_1, w_2 \rangle\\
        &=\gS(u_2, v_1, w_2) - \gamma \langle u_1, u_2 \rangle \gS(u_1, v_1, w_2)
    \end{split}
\end{equation}
Finally by symmetry, from \eqref{assum_alpha} and since $\gT_1$ is cubic, we have

\begin{equation}
    \boxed{
    \begin{split}
        &\left[ \lambda_2 + a(\lambda_2) + (1 - \gamma \kappa^2) b(\lambda_2) - \frac{\gamma \kappa }{3} r(\lambda_1) \right] \eta   = \sum_{i=1}^2 \beta_i \theta_{2i} \rho_{1i} \rho_{2i} - {\gamma}\kappa  \sum_{i=1}^2 \beta_i \rho_{1i}^2 \rho_{2i} 
    \end{split}
    }
\end{equation}

\subsubsection{System of equations}
The second deflation step is therefore governed by the following system of equations
\begin{equation}
    \begin{cases}
        \left[ 2b(z) +z \right] a(z) + \frac13 = 0\\
        (a(z) + z - \tau b(z)) b(z) + \frac13 = 0\\
        q(z) = a(z) + 2 b(z)\\
        f_q(\lambda_2) -   \frac{\gamma\kappa \eta^2}{3} r(\lambda_1) -   2\gamma \kappa^2 b(\lambda_2)= \sum_{i=1}^2 \beta_i \theta_{2i} \rho_{2i}^2  -    \gamma \kappa \sum_{i=1}^2 \beta_i \rho_{1i} \rho_{2i}^2\\
        [f_q (\lambda_2) - a(\lambda_2)] \theta_{2s} - \gamma \rho_{1s} \left[ \frac{\eta^2}{3} r(\lambda_1) +  2\kappa b(\lambda_2)  \right]=\sum_{i=1}^2 \beta_i \langle x_s, x_i \rangle \rho_{2i}^2 -  \gamma \rho_{1s} \sum_{i=1}^2 \beta_i \rho_{1i} \rho_{2i}^2\\
        \left[ \lambda_2 + 2(1 - {\gamma}) b(\lambda_2)  \right] \kappa= (1 - {\gamma}) \left[\sum_{i=1}^2 \beta_i \rho_{1i} \rho_{2i}^2 - \frac{\eta^2}{3} r(\lambda_1) \right]\\
        \left[ f_q(\lambda_2) - (1 + \gamma  \kappa^2) b(\lambda_2)\right]  \rho_{2s} =  \sum_{i=1}^2 \beta_i \theta_{2i} \rho_{2i} \langle y_s, y_i \rangle  -  \gamma  \kappa  \left[ \sum_{i=1}^2 \beta_i \rho_{1i} \rho_{2i} \langle y_s, y_i \rangle  - \frac{\rho_{1s} \eta}{3} r(\lambda_1) \right] \\
        \left[ \lambda_2 + a(\lambda_2) + (1 - \gamma \kappa^2) b(\lambda_2) - \frac{\gamma \kappa }{3} r(\lambda_1) \right] \eta   = \sum_{i=1}^2 \beta_i \theta_{2i} \rho_{1i} \rho_{2i} - {\gamma}\kappa  \sum_{i=1}^2 \beta_i \rho_{1i}^2 \rho_{2i}
    \end{cases}
\end{equation}
with $f_q(z) = z + q(z)$ and $\tau = \gamma\kappa^2-1+\kappa(\gamma-1)$. In the case $\gamma = 1$, we have $\kappa = 0$ from \eqref{eq_kappa_zero} and therefore the system above reduces to the following system, since $a(z) = b(z) = \frac{r(z)}{3}$ and $q(z) = r(z)$.
\begin{equation}
    \begin{cases}
        f_r(\lambda_2)  = \sum_{i=1}^2 \beta_i \theta_{2i} \rho_{2i}^2  \\
        h_r (\lambda_2) \theta_{2s} -  \frac{\eta^2}{3} r(\lambda_1) \rho_{1s}  =\sum_{i=1}^2 \beta_i \langle x_s, x_i \rangle \rho_{2i}^2 -   \rho_{1s} \sum_{i=1}^2 \beta_i \rho_{1i} \rho_{2i}^2\\
         h_r(\lambda_2)   \rho_{2s} =  \sum_{i=1}^2 \beta_i \theta_{2i} \rho_{2i} \langle y_s, y_i \rangle  \\
        \left[ \lambda_2 + \frac23 r(\lambda_2)  \right] \eta   = \sum_{i=1}^2 \beta_i \theta_{2i} \rho_{1i} \rho_{2i} 
    \end{cases}
\end{equation}

\section{Algorithms}\label{sec_appendix_improved}
Algorithm \ref{alg_fixed_point} below, implements the fixed point equation in Definition \ref{def_limiting_measure} which allows the computation of the Stieltjes transform at the second deflation step.
\begin{algorithm}[h!]
   \caption{Stieltjes Transform by Fixed Point}
   \label{alg_fixed_point}
\begin{algorithmic}
   \STATE {\bfseries Input:} $z\in \sC \setminus \supp(\nu)$ and $\tau$.
   \STATE - Initialize $a$ and $b$.
   \REPEAT
   \STATE - Update $a \leftarrow \frac{-1}{ 3 (2b + z) }$.
   \STATE - Update $b \leftarrow \frac{-1}{ 3 (a + z - \tau b) }$.
   \UNTIL{convergence.}
   \STATE {\bfseries Output:} $a,b$ and Stieltjes transform $q=a+2b$.
\end{algorithmic}
\end{algorithm}

Algorithm \ref{alg_RTT_improved} implements our RTT-improved tensor deflation procedure which is described in more details in Section \ref{sec_improved}.
\begin{algorithm}[h!]
   \caption{RTT-Improved Tensor Deflation Algorithm}
   \label{alg_RTT_improved}
\begin{algorithmic}
   \STATE {\bfseries Input:} Tensor $\gT \in \sR^{p\times p \times p}$ and step size $\epsilon\in [0, 1]$.
   \STATE {\color{gray} \# Perform orthogonalized deflation:}
   \STATE 1- Compute $\hat\lambda_1 \hat u_1\otimes \hat v_1\otimes \hat w_1$ as best rank-one approximation of $\gT$.
   \STATE 2- Compute $\hat\lambda_2 \hat u_2\otimes \hat v_2\otimes \hat w_2$ as best rank-one approximation of $\gT \times_1 (\mI_p - \gamma \hat u_1 \hat u_1^\top)$ for $\gamma=1$.
   \STATE {\color{gray} \# Estimate underlying model parameters:}
   \STATE 3- Compute $\hat \eta \leftarrow \vert \langle \hat v_1, \hat v_2 \rangle \vert$.
   \STATE 4- Estimate $\hat \vbeta = (\hat \beta_1, \hat \beta_2, \hat \alpha)$ and $\hat \vrho = (\hat\rho_{1i}, \hat\rho_{2i}, \hat\theta_{2i}\mid i\in [2])$ by fixing $\hat\vlambda = (\hat \lambda_1, \hat \lambda_2, \hat \eta)$ and solving $\psi(\hat \vbeta, \hat \vlambda, \hat \vrho) = 0$ in $\hat \vbeta$ and $\hat \vrho$ with $\psi$ defined in \eqref{eq_psi}.
   \STATE {\color{gray} \# Estimate optimal $\gamma$:}
   \STATE 5- Initialize $\gamma=1$ and $\hat\kappa=10^{-5}$.
   \STATE 6- Initialize two empty lists $L_\gamma$ and $L_\rho$.
   \REPEAT
   \STATE 7- Set $x_0\leftarrow (\hat\lambda_2,\hat \kappa, \hat \eta, \hat \theta_{2i}, \hat \rho_{2i} \mid i\in [2])$.
   \STATE 8- Estimate $(\hat\lambda_2,\hat \kappa, \hat \eta, \hat \theta_{2i}, \hat \rho_{2i} \mid i\in [2])$ by solving the system in \eqref{eq_system_general_gamma} initialized with $x_0$ and for $(\beta_1, \beta_2,  \alpha) = (\hat \beta_1, \hat \beta_2, \hat \alpha)$ and $\gamma$.
   \STATE 9- Append $L_\gamma$ with $\gamma$.
   \STATE 10- Append $L_\rho$ with $\max \{ \hat\rho_{21}, \hat\rho_{22} \}$.
   \STATE 11- Update $\gamma \leftarrow \gamma - \epsilon$.
   \UNTIL{The maximum is reached in $L_\rho$.}
   \STATE 12- Set optimal $\gamma$ as $\gamma^* \leftarrow L_\gamma[\argmax(L_\rho)]$.
   \STATE {\color{gray} \# Perform orthogonalized deflation with $\gamma^*$:}
   \STATE 13- Compute $\hat\lambda_2 \hat u_2\otimes \hat v_2^* \otimes \hat w_2^*$ as best rank-one approximation of $\gT \times_1 (\mI_p - \gamma^* \hat u_1 \hat u_1^\top)$.
   \STATE 14- Compute $\hat\lambda_2 \hat u_2^*\otimes \hat v_2\otimes \hat w_2^*$ as best rank-one approximation of $\gT \times_2 (\mI_p - \gamma^* \hat v_1 \hat v_1^\top)$.
   \STATE {\color{gray} \# Re-estimate the first component by simple deflation:}
   \STATE 15- Compute $\hat\lambda_1 \hat u_1^*\otimes \hat v_1^*\otimes \hat w_1^*$ as best rank-one approximation of $\gT - \min\{\hat \beta_1, \hat \beta_2\} \hat u_2^* \otimes \hat v_2^* \otimes \hat w_2^* $.
   \STATE {\bfseries Output:} Estimates of the signal components $(\max\{\hat \beta_1, \hat \beta_2\}, \hat u_1^*, \hat v_1^* , \hat w_1^*), (\min\{\hat \beta_1, \hat \beta_2\}, \hat u_2^*, \hat v_2^* , \hat w_2^*)$.
\end{algorithmic}
\end{algorithm}


\end{document}

%% file: math.tex
\usepackage{amsmath,amsfonts,bm}

\def\1{\bm{1}}

\def\vbeta{{\bm{\beta}}}
\def\vlambda{{\bm{\lambda}}}
\def\vrho{{\bm{\rho}}}

\def\mA{{\bm{A}}}
\def\mB{{\bm{B}}}

\def\mI{{\bm{I}}}

\def\mM{{\bm{M}}}
\def\mN{{\bm{N}}}

\def\mP{{\bm{P}}}

\def\mR{{\bm{R}}}
\def\mS{{\bm{S}}}

\def\mU{{\bm{U}}}
\def\mV{{\bm{V}}}

\DeclareMathAlphabet{\mathsfit}{\encodingdefault}{\sfdefault}{m}{sl}
\SetMathAlphabet{\mathsfit}{bold}{\encodingdefault}{\sfdefault}{bx}{n}

\def\gN{{\mathcal{N}}}

\def\gS{{\mathcal{S}}}
\def\gT{{\mathcal{T}}}

\def\gW{{\mathcal{W}}}

\def\sC{{\mathbb{C}}}

\def\sR{{\mathbb{R}}}
\def\sS{{\mathbb{S}}}

\newcommand{\E}{\mathbb{E}}

\newcommand{\spec}{\mathrm{Sp}}
\newcommand{\rank}{\mathrm{rank}}

\DeclareMathOperator*{\argmax}{arg\,max}

\DeclareMathOperator{\Tr}{tr}

%% file: figs/intro_alignments.tex
\begin{tikzpicture}

\definecolor{darkorange25512714}{RGB}{255,127,14}
\definecolor{darkslategray38}{RGB}{38,38,38}
\definecolor{lightgray204}{RGB}{204,204,204}
\definecolor{steelblue31119180}{RGB}{31,119,180}
\definecolor{mediumpurple148103189}{RGB}{90,40,130}

\begin{groupplot}[group style={group size=2 by 2}]
\nextgroupplot[
axis line style={lightgray204},
legend cell align={left},
legend style={fill opacity=0.8, draw opacity=1,at={(0.97,0.5)},anchor=east, text opacity=1, draw=lightgray204},
tick align=outside,
title={\(\displaystyle \langle x_1, x_2 \rangle = 0 \)},
x grid style={lightgray204},
ylabel={Alignments of $u_1$},
width=0.26\textwidth,
height=0.23\textwidth,
xmajorgrids,
xmajorticks=true,
xmin=0.0499999999999999, xmax=20.95,
xtick style={color=darkslategray38},
xtick={5,12,20},
xticklabels={5,\(\displaystyle \textcolor{mediumpurple148103189}{\beta_1=\beta_2}\),20},
y grid style={lightgray204},
ymajorgrids,
ymajorticks=true,
ymin=-0.05, ymax=1.05,
ytick style={color=darkslategray38},
ytick={0,1},
]
\addplot [thick, steelblue31119180, solid, line width=1.5pt]
table {%
1 1
1.5 1
2 1
2.5 1
3 1
3.5 1
4 1
4.5 1
5 1
5.5 1
6 1
6.5 1
7 1
7.5 1
8 1
8.5 1
9 1
9.5 1
10 1
10.5 1
11 1
11.5 1
12 1
12.5 0
13 0
13.5 0
14 0
14.5 0
15 0
15.5 5.08535733021103e-47
16 0
16.5 0
17 0
17.5 3.38083245094719e-32
18 0
18.5 0
19 0
19.5 0
20 0
};
\addlegendentry{$\langle u_1, x_1 \rangle$}
\addplot [thick, darkorange25512714, dashed, line width=1.5pt]
table {%
1 0
1.5 0
2 0
2.5 0
3 0
3.5 0
4 0
4.5 0
5 0
5.5 0
6 0
6.5 0
7 0
7.5 0
8 0
8.5 0
9 0
9.5 0
10 0
10.5 0
11 0
11.5 0
12 3.08148791101958e-33
12.5 1
13 1
13.5 1
14 1
14.5 1
15 1
15.5 1
16 1
16.5 1
17 1
17.5 1
18 1
18.5 1
19 1
19.5 1
20 1
};
\addlegendentry{$\langle u_1, x_2 \rangle$}

\nextgroupplot[
axis line style={lightgray204},
legend cell align={left},
legend style={
  fill opacity=0.8,
  draw opacity=1,
  text opacity=1,
  at={(0.03,0.03)},
  anchor=south west,
  draw=lightgray204
},
width=0.26\textwidth,
height=0.23\textwidth,
tick align=outside,
title={\(\displaystyle \langle x_1, x_2 \rangle = 0.5\)},
x grid style={lightgray204},
xmajorgrids,
xmajorticks=true,
xmin=0.0499999999999999, xmax=20.95,
xtick style={color=darkslategray38},
xtick={5,12,20},
xticklabels={5,\(\displaystyle \textcolor{mediumpurple148103189}{\beta_1=\beta_2}\),20},
y grid style={lightgray204},
ymajorgrids,
ymajorticks=true,
ymin=-0.05, ymax=1.05,
ytick={0,1},
ytick style={color=darkslategray38}
]
\addplot [thick, steelblue31119180, solid, line width=1.5pt]
table {%
1 0.999818784700321
1.5 0.999568662197534
2 0.999187284671912
2.5 0.99865149512165
3 0.997933623948189
3.5 0.997000497196941
4 0.995812215575987
4.5 0.994320659229601
5 0.992467675073596
5.5 0.990182918070884
6 0.987381358569642
6.5 0.983960558955651
7 0.979798005528202
7.5 0.974749119641012
8 0.968647149108203
8.5 0.961307021090352
9 0.952536337426445
9.5 0.942157482585339
10 0.930043884361225
10.5 0.916168513968288
11 0.900652412022085
11.5 0.883789686841853
12 0.866025403784439
12.5 0.847883796489824
13 0.829874584947
13.5 0.812417034652207
14 0.795804726948423
14.5 0.780208231673125
15 0.76569911102166
15.5 0.75227944854175
16 0.739907874394404
16.5 0.728519046088818
17 0.718036809824446
17.5 0.708382432650472
18 0.699479391099724
18.5 0.691255913687878
19 0.683646130326667
19.5 0.676590396259628
20 0.670035151449432
};
\addlegendentry{$\langle u_1, x_1 \rangle$}
\addplot [thick, darkorange25512714, dashed, line width=1.5pt]
table {%
1 0.516395699404195
1.5 0.525217914496983
2 0.534501773372701
2.5 0.544285657130014
3 0.554611772619471
3.5 0.565526410966527
4 0.577080115405605
4.5 0.589327677259771
5 0.602327831007252
5.5 0.616142446517875
6 0.630834908335394
6.5 0.646467216837467
7 0.66309513803339
7.5 0.680760483091316
8 0.699479391099724
8.5 0.719225529863861
9 0.739907874394404
9.5 0.761344926898817
10 0.783241645957285
10.5 0.805181419908093
11 0.826649105313091
11.5 0.847094599188356
12 0.866025403784439
12.5 0.883093072410352
13 0.898136143821766
13.5 0.911167843762586
14 0.922326931381656
14.5 0.931820218995885
15 0.939876173847404
15.5 0.946715506683948
16 0.952536337426445
16.5 0.957509048635774
17 0.961776527879346
17.5 0.965456908830671
18 0.968647149108203
18.5 0.971426603649485
19 0.97386022891799
19.5 0.976001303453363
20 0.977893669436873
};
\addlegendentry{$\langle u_1, x_2 \rangle$}

\nextgroupplot[
axis line style={lightgray204},
legend cell align={left},
legend style={fill opacity=0.8, at={(0.97,0.5)},
  anchor=east, draw opacity=1, text opacity=1, draw=lightgray204},
tick align=outside,
x grid style={lightgray204},
xlabel=\textcolor{darkslategray38}{\(\displaystyle \beta_{2}\)},
width=0.26\textwidth,
height=0.23\textwidth,
xmajorgrids,
xmajorticks=true,
xmin=0.0499999999999999, xmax=20.95,
xtick style={color=darkslategray38},
xtick={5,12,20},
xticklabels={5,\(\displaystyle \textcolor{mediumpurple148103189}{\beta_1=\beta_2}\),20},
y grid style={lightgray204},
ylabel={Alignments of $u_2$},
ymajorgrids,
ymajorticks=true,
ymin=-0.05, ymax=1.05,
ytick style={color=darkslategray38},
ytick={0,1},
]
\addplot [thick, steelblue31119180, solid, line width=1.5pt]
table {%
1 0
1.5 0
2 0
2.5 0
3 0
3.5 0
4 0
4.5 0
5 0
5.5 0
6 0
6.5 0
7 0
7.5 0
8 0
8.5 0
9 0
9.5 0
10 0
10.5 0
11 0
11.5 0
12 1
12.5 1
13 1
13.5 1
14 1
14.5 1
15 1
15.5 1
16 1
16.5 1
17 1
17.5 1
18 1
18.5 1
19 1
19.5 1
20 1
};
\addlegendentry{$\langle u_2, x_1 \rangle$}
\addplot [thick, darkorange25512714, dashed, line width=1.5pt]
table {%
1 1
1.5 1
2 1
2.5 1
3 1
3.5 1
4 1
4.5 1
5 1
5.5 1
6 1
6.5 1
7 1
7.5 1
8 1
8.5 1
9 1
9.5 1
10 1
10.5 1
11 1
11.5 1
12 0
12.5 0
13 3.64176793515635e-158
13.5 0
14 0
14.5 0
15 0
15.5 0
16 0
16.5 0
17 0
17.5 0
18 0
18.5 0
19 0
19.5 0
20 0
};
\addlegendentry{$\langle u_2, x_2 \rangle$}

\nextgroupplot[
axis line style={lightgray204},
legend cell align={left},
legend style={
  fill opacity=0.8,
  draw opacity=1,
  text opacity=1,
  at={(0.03,0.43)},
  anchor=south west,
  draw=lightgray204
},
width=0.28\textwidth,
height=0.23\textwidth,
tick align=outside,
x grid style={lightgray204},
xlabel=\textcolor{darkslategray38}{\(\displaystyle \beta_{2}\)},
xmajorgrids,
xmajorticks=true,
xmin=0.0499999999999999, xmax=20.95,
xtick style={color=darkslategray38},
xtick={5,12,20},
xticklabels={5,\(\displaystyle \textcolor{mediumpurple148103189}{\beta_1=\beta_2}\),20},
y grid style={lightgray204},
ymajorgrids,
ymajorticks=true,
ymin=-0.05, ymax=1.05,
ytick={0,1},
ytick style={color=darkslategray38}
]
\addplot [thick, steelblue31119180, solid, line width=1.5pt]
table {%
1 0.410551583621425
1.5 0.407810785644279
2 0.404882963760712
2.5 0.401748395066085
3 0.398384653038993
3.5 0.394766194999501
4 0.390863900858157
4.5 0.386644573110089
5 0.382070423350124
5.5 0.377098597861145
6 0.371680841624489
6.5 0.365763478136505
7 0.359288007255327
7.5 0.352192811653849
8 0.344416719097484
8.5 0.335905449924487
9 0.326622120259781
9.5 0.316562563474763
10 0.305774577698944
10.5 0.294376695404028
11 0.282567487363379
11.5 0.270614592744117
12 0.974556715862179
12.5 0.968896909078112
13 0.971564870434026
13.5 0.973921876776764
14 0.975981754584335
14.5 0.977771279155425
15 0.979322683600213
15.5 0.980668607247081
16 0.981839308217734
16.5 0.982861468133533
17 0.983757938642206
17.5 0.984547967241567
18 0.985247622256635
18.5 0.985870265883394
19 0.986427002723558
19.5 0.986927074597227
20 0.98737819458936
};
\addlegendentry{$\langle u_2, x_1 \rangle$}
\addplot [thick, darkorange25512714, dashed, line width=1.5pt]
table {%
1 0.994950122082031
1.5 0.994644143822839
2 0.994307846390339
2.5 0.993937038861754
3 0.993526797109343
3.5 0.993071314368777
4 0.992563720122432
4.5 0.991995862281316
5 0.991358048848157
5.5 0.990638748816537
6 0.989824260187753
6.5 0.988898369508416
7 0.987842058489139
7.5 0.986633368318159
8 0.985247622256635
8.5 0.983658337627297
9 0.981839308217734
9.5 0.979768413441046
10 0.977433485690361
10.5 0.974839706005891
11 0.972016378958443
11.5 0.969019360963577
12 0.29316651986874
12.5 0.270137734592152
13 0.280730546896313
13.5 0.29047395802653
14 0.299325270890626
14.5 0.307302341215576
15 0.314460568287094
15.5 0.320873948171141
16 0.326622120259781
16.5 0.33178268101319
17 0.336427244921442
17.5 0.340619853606553
18 0.344416719097484
18.5 0.347866652094701
19 0.351011791725911
19.5 0.353888424652343
20 0.356527784829788
};
\addlegendentry{$\langle u_2, x_2 \rangle$}
\end{groupplot}

\end{tikzpicture}

%% file: figs/semi_circle_law.tex
\begin{tikzpicture}

\definecolor{darkslategray38}{RGB}{38,38,38}
\definecolor{forestgreen4416044}{RGB}{44,160,44}
\definecolor{lightgray204}{RGB}{204,204,204}

\begin{axis}[
axis line style={lightgray204},
legend cell align={left},
legend style={
  fill opacity=0.8,
  draw opacity=1,
  text opacity=1,
  at={(0.5,0.09)},
  anchor=south,
  draw=lightgray204
},
width=0.45\textwidth,
height=0.27\textwidth,
tick align=outside,
x grid style={lightgray204},
xmajorgrids,
xmajorticks=true,
xmin=-2.2, xmax=2.2,
xtick style={color=darkslategray38},
y grid style={lightgray204},
ylabel=\textcolor{darkslategray38}{Density},
ymajorgrids,
ymajorticks=true,
ymin=0, ymax=0.436885290232199,
ytick style={color=darkslategray38}
]
\draw[draw=white,fill=forestgreen4416044,fill opacity=0.75] (axis cs:-1.6102757803112,0) rectangle (axis cs:-1.5021238213041,0.12328332704965);
\addlegendimage{ybar,ybar legend,draw=white,fill=forestgreen4416044,fill opacity=0.75}
\addlegendentry{Eigenvalues of $\mN$}

\draw[draw=white,fill=forestgreen4416044,fill opacity=0.75] (axis cs:-1.5021238213041,0) rectangle (axis cs:-1.39397186229699,0.169514574693269);
\draw[draw=white,fill=forestgreen4416044,fill opacity=0.75] (axis cs:-1.39397186229699,0) rectangle (axis cs:-1.28581990328989,0.231156238218095);
\draw[draw=white,fill=forestgreen4416044,fill opacity=0.75] (axis cs:-1.28581990328989,0) rectangle (axis cs:-1.17766794428278,0.246566654099301);
\draw[draw=white,fill=forestgreen4416044,fill opacity=0.75] (axis cs:-1.17766794428278,0) rectangle (axis cs:-1.06951598527568,0.277387485861713);
\draw[draw=white,fill=forestgreen4416044,fill opacity=0.75] (axis cs:-1.06951598527568,0) rectangle (axis cs:-0.961364026268572,0.308208317624126);
\draw[draw=white,fill=forestgreen4416044,fill opacity=0.75] (axis cs:-0.961364026268572,0) rectangle (axis cs:-0.853212067261466,0.323618733505332);
\draw[draw=white,fill=forestgreen4416044,fill opacity=0.75] (axis cs:-0.853212067261466,0) rectangle (axis cs:-0.745060108254361,0.369849981148951);
\draw[draw=white,fill=forestgreen4416044,fill opacity=0.75] (axis cs:-0.745060108254361,0) rectangle (axis cs:-0.636908149247256,0.354439565267745);
\draw[draw=white,fill=forestgreen4416044,fill opacity=0.75] (axis cs:-0.636908149247256,0) rectangle (axis cs:-0.52875619024015,0.323618733505332);
\draw[draw=white,fill=forestgreen4416044,fill opacity=0.75] (axis cs:-0.528756190240151,0) rectangle (axis cs:-0.420604231233045,0.369849981148951);
\draw[draw=white,fill=forestgreen4416044,fill opacity=0.75] (axis cs:-0.420604231233045,0) rectangle (axis cs:-0.31245227222594,0.385260397030158);
\draw[draw=white,fill=forestgreen4416044,fill opacity=0.75] (axis cs:-0.31245227222594,0) rectangle (axis cs:-0.204300313218835,0.385260397030157);
\draw[draw=white,fill=forestgreen4416044,fill opacity=0.75] (axis cs:-0.204300313218835,0) rectangle (axis cs:-0.0961483542117296,0.41608122879257);
\draw[draw=white,fill=forestgreen4416044,fill opacity=0.75] (axis cs:-0.0961483542117296,0) rectangle (axis cs:0.0120036047953758,0.385260397030157);
\draw[draw=white,fill=forestgreen4416044,fill opacity=0.75] (axis cs:0.0120036047953758,0) rectangle (axis cs:0.120155563802481,0.400670812911364);
\draw[draw=white,fill=forestgreen4416044,fill opacity=0.75] (axis cs:0.120155563802481,0) rectangle (axis cs:0.228307522809586,0.385260397030158);
\draw[draw=white,fill=forestgreen4416044,fill opacity=0.75] (axis cs:0.228307522809586,0) rectangle (axis cs:0.336459481816691,0.385260397030157);
\draw[draw=white,fill=forestgreen4416044,fill opacity=0.75] (axis cs:0.336459481816691,0) rectangle (axis cs:0.444611440823797,0.369849981148951);
\draw[draw=white,fill=forestgreen4416044,fill opacity=0.75] (axis cs:0.444611440823797,0) rectangle (axis cs:0.552763399830902,0.369849981148951);
\draw[draw=white,fill=forestgreen4416044,fill opacity=0.75] (axis cs:0.552763399830902,0) rectangle (axis cs:0.660915358838007,0.400670812911364);
\draw[draw=white,fill=forestgreen4416044,fill opacity=0.75] (axis cs:0.660915358838007,0) rectangle (axis cs:0.769067317845112,0.323618733505332);
\draw[draw=white,fill=forestgreen4416044,fill opacity=0.75] (axis cs:0.769067317845112,0) rectangle (axis cs:0.877219276852218,0.323618733505332);
\draw[draw=white,fill=forestgreen4416044,fill opacity=0.75] (axis cs:0.877219276852218,0) rectangle (axis cs:0.985371235859323,0.339029149386539);
\draw[draw=white,fill=forestgreen4416044,fill opacity=0.75] (axis cs:0.985371235859323,0) rectangle (axis cs:1.09352319486643,0.292797901742919);
\draw[draw=white,fill=forestgreen4416044,fill opacity=0.75] (axis cs:1.09352319486643,0) rectangle (axis cs:1.20167515387353,0.261977069980507);
\draw[draw=white,fill=forestgreen4416044,fill opacity=0.75] (axis cs:1.20167515387353,0) rectangle (axis cs:1.30982711288064,0.261977069980507);
\draw[draw=white,fill=forestgreen4416044,fill opacity=0.75] (axis cs:1.30982711288064,0) rectangle (axis cs:1.41797907188774,0.200335406455682);
\draw[draw=white,fill=forestgreen4416044,fill opacity=0.75] (axis cs:1.41797907188774,0) rectangle (axis cs:1.52613103089485,0.169514574693269);
\draw[draw=white,fill=forestgreen4416044,fill opacity=0.75] (axis cs:1.52613103089485,0) rectangle (axis cs:1.63428298990195,0.0924624952872376);
\addplot [thick, black, line width=1.3pt]
table {%
-2 -0
-1.97989949748744 -0
-1.95979899497487 -0
-1.93969849246231 -0
-1.91959798994975 -0
-1.89949748743719 -0
-1.87939698492462 -0
-1.85929648241206 -0
-1.8391959798995 -0
-1.81909547738693 -0
-1.79899497487437 -0
-1.77889447236181 -0
-1.75879396984925 -0
-1.73869346733668 -0
-1.71859296482412 -0
-1.69849246231156 -0
-1.67839195979899 -0
-1.65829145728643 -0
-1.63819095477387 -0
-1.61809045226131 0.0525482284801871
-1.59798994974874 0.0802846807061343
-1.57788944723618 0.100418572362237
-1.55778894472362 0.116944839449879
-1.53768844221106 0.131233560311949
-1.51758793969849 0.143952580073599
-1.49748743718593 0.15548758582937
-1.47738693467337 0.166085458315139
-1.4572864321608 0.175915649275321
-1.43718592964824 0.185100507659136
-1.41708542713568 0.193731841190047
-1.39698492462312 0.201880659612568
-1.37688442211055 0.209603247207264
-1.35678391959799 0.216945126318979
-1.33668341708543 0.223943744067097
-1.31658291457286 0.230630351991071
-1.2964824120603 0.237031356769026
-1.27638190954774 0.243169313449517
-1.25628140703518 0.249063670563889
-1.23618090452261 0.254731338993673
-1.21608040201005 0.260187133071114
-1.19597989949749 0.265444117368469
-1.17587939698492 0.27051388273948
-1.15577889447236 0.275406768514424
-1.1356783919598 0.280132043172196
-1.11557788944724 0.284698052609436
-1.09547738693467 0.289112342847699
-1.07537688442211 0.293381762373991
-1.05527638190955 0.297512548105046
-1.03517587939698 0.301510398072354
-1.01507537688442 0.305380533254897
-0.994974874371859 0.309127750478415
-0.974874371859296 0.312756467910976
-0.954773869346734 0.31627076438387
-0.934673366834171 0.319674413532377
-0.914572864321608 0.322970913566663
-0.894472361809045 0.326163513337107
-0.874371859296482 0.329255235241946
-0.85427135678392 0.332248895431658
-0.834170854271357 0.335147121688952
-0.814070351758794 0.337952369301837
-0.793969849246231 0.34066693519707
-0.773869346733668 0.34329297056
-0.753768844221105 0.345832492132762
-0.733668341708543 0.348287392354517
-0.71356783919598 0.350659448483803
-0.693467336683417 0.352950330823368
-0.673366834170854 0.355161610151194
-0.653266331658291 0.357294764447447
-0.633165829145729 0.359351184995164
-0.613065326633166 0.361332181922403
-0.592964824120603 0.363238989244939
-0.57286432160804 0.365072769461198
-0.552763819095477 0.366834617744777
-0.532663316582915 0.36852556577441
-0.512562814070352 0.370146585236504
-0.492462311557789 0.371698591031288
-0.472361809045226 0.373182444209999
-0.452261306532663 0.374598954667464
-0.4321608040201 0.37594888361168
-0.412060301507538 0.377232945829609
-0.391959798994975 0.378451811766313
-0.371859296482412 0.379606109432697
-0.351758793969849 0.380696426155482
-0.331658291457286 0.381723310181588
-0.311557788944724 0.3826872721478
-0.291457286432161 0.38358878642547
-0.271356783919598 0.384428292348923
-0.251256281407035 0.385206195335377
-0.231155778894472 0.385922867903301
-0.21105527638191 0.386578650595406
-0.190954773869347 0.387173852811789
-0.170854271356784 0.387708753558117
-0.150753768844221 0.388183602113171
-0.130653266331658 0.388598618619558
-0.110552763819096 0.38895399460091
-0.0904522613065326 0.389249893408444
-0.0703517587939699 0.389486450599336
-0.050251256281407 0.389663774248976
-0.0301507537688441 0.389781945198785
-0.0100502512562815 0.389841017240932
0.0100502512562812 0.389841017240932
0.0301507537688441 0.389781945198785
0.050251256281407 0.389663774248976
0.0703517587939699 0.389486450599336
0.0904522613065328 0.389249893408444
0.110552763819095 0.38895399460091
0.130653266331658 0.388598618619558
0.150753768844221 0.388183602113171
0.170854271356784 0.387708753558117
0.190954773869347 0.387173852811789
0.21105527638191 0.386578650595406
0.231155778894472 0.385922867903301
0.251256281407035 0.385206195335377
0.271356783919598 0.384428292348923
0.291457286432161 0.38358878642547
0.311557788944724 0.3826872721478
0.331658291457286 0.381723310181588
0.351758793969849 0.380696426155482
0.371859296482412 0.379606109432697
0.391959798994975 0.378451811766313
0.412060301507538 0.377232945829609
0.4321608040201 0.37594888361168
0.452261306532663 0.374598954667464
0.472361809045226 0.373182444209999
0.492462311557789 0.371698591031288
0.512562814070352 0.370146585236504
0.532663316582914 0.36852556577441
0.552763819095477 0.366834617744777
0.57286432160804 0.365072769461198
0.592964824120603 0.363238989244939
0.613065326633166 0.361332181922403
0.633165829145728 0.359351184995164
0.653266331658291 0.357294764447447
0.673366834170854 0.355161610151194
0.693467336683417 0.352950330823368
0.71356783919598 0.350659448483803
0.733668341708543 0.348287392354517
0.753768844221105 0.345832492132762
0.773869346733668 0.34329297056
0.793969849246231 0.34066693519707
0.814070351758794 0.337952369301837
0.834170854271357 0.335147121688952
0.85427135678392 0.332248895431658
0.874371859296482 0.329255235241946
0.894472361809045 0.326163513337107
0.914572864321608 0.322970913566663
0.934673366834171 0.319674413532377
0.954773869346734 0.31627076438387
0.974874371859296 0.312756467910976
0.994974874371859 0.309127750478415
1.01507537688442 0.305380533254897
1.03517587939699 0.301510398072354
1.05527638190955 0.297512548105046
1.07537688442211 0.293381762373991
1.09547738693467 0.289112342847699
1.11557788944724 0.284698052609436
1.1356783919598 0.280132043172196
1.15577889447236 0.275406768514424
1.17587939698492 0.27051388273948
1.19597989949749 0.265444117368469
1.21608040201005 0.260187133071114
1.23618090452261 0.254731338993673
1.25628140703518 0.249063670563889
1.27638190954774 0.243169313449517
1.2964824120603 0.237031356769026
1.31658291457286 0.230630351991071
1.33668341708543 0.223943744067097
1.35678391959799 0.216945126318979
1.37688442211055 0.209603247207264
1.39698492462312 0.201880659612568
1.41708542713568 0.193731841190047
1.43718592964824 0.185100507659135
1.4572864321608 0.175915649275321
1.47738693467337 0.166085458315139
1.49748743718593 0.15548758582937
1.51758793969849 0.143952580073599
1.53768844221106 0.131233560311948
1.55778894472362 0.116944839449879
1.57788944723618 0.100418572362237
1.59798994974874 0.0802846807061343
1.61809045226131 0.0525482284801871
1.63819095477387 -0
1.65829145728643 -0
1.67839195979899 -0
1.69849246231156 -0
1.71859296482412 -0
1.73869346733668 -0
1.75879396984925 -0
1.77889447236181 -0
1.79899497487437 -0
1.81909547738693 -0
1.8391959798995 -0
1.85929648241206 -0
1.87939698492462 -0
1.89949748743719 -0
1.91959798994975 -0
1.93969849246231 -0
1.95979899497487 -0
1.97989949748744 -0
2 -0
};
\addlegendentry{Semi-circle Law}
\end{axis}

\end{tikzpicture}

%% file: figs/first_deflation_step_singular_value_alignments.tex
\begin{tikzpicture}

\definecolor{crimson2143940}{RGB}{214,39,40}
\definecolor{darkorange25512714}{RGB}{255,127,14}
\definecolor{darkslategray38}{RGB}{38,38,38}
\definecolor{darkviolet1910191}{RGB}{191,0,191}
\definecolor{forestgreen4416044}{RGB}{44,160,44}
\definecolor{lightgray204}{RGB}{204,204,204}
\definecolor{steelblue31119180}{RGB}{31,119,180}

\begin{groupplot}[group style={group size=2 by 1}]
\nextgroupplot[
axis line style={lightgray204},
legend cell align={left},
legend style={
  fill opacity=0.8,
  draw opacity=1,
  text opacity=1,
  at={(0.03,0.97)},
  anchor=north west,
  draw=lightgray204
},
width=.25\textwidth,
height=.32\textwidth,
tick align=outside,
x grid style={lightgray204},
xlabel=\textcolor{darkslategray38}{\(\displaystyle \beta_1\)},
xmajorgrids,
xmajorticks=true,
xmin=-0.75, xmax=15.75,
xtick style={color=darkslategray38},
y grid style={lightgray204},
ymajorgrids,
ymajorticks=true,
ymin=4.48480492246538, ymax=16.4298348341458,
ytick style={color=darkslategray38}
]
\addplot [semithick, steelblue31119180, mark=*, mark size=2, mark options={solid}, only marks, opacity=0.4]
table {%
0 5.06834890522459
0.306122448979592 5.02776082754176
0.612244897959184 5.18337341349332
0.918367346938776 5.28326076111957
1.22448979591837 5.1738607198979
1.53061224489796 5.30966840233826
1.83673469387755 5.36320441249225
2.14285714285714 5.3394563691669
2.44897959183673 5.41944919234862
2.75510204081633 5.56914207527932
3.06122448979592 5.60812258459896
3.36734693877551 5.87164114489597
3.6734693877551 5.83189383011304
3.97959183673469 5.88895219132424
4.28571428571429 6.0863350775787
4.59183673469388 6.35339040682992
4.89795918367347 6.49146921649855
5.20408163265306 6.69929893144694
5.51020408163265 7.0281011541316
5.81632653061224 7.04138755023617
6.12244897959184 7.43500527995133
6.42857142857143 7.70845652771654
6.73469387755102 7.94456668198327
7.04081632653061 8.20941455679112
7.3469387755102 8.41953230982811
7.6530612244898 8.84607831407048
7.95918367346939 9.04374516225509
8.26530612244898 9.36120391683834
8.57142857142857 9.5536001030323
8.87755102040816 9.82797710774058
9.18367346938776 10.2129112756272
9.48979591836735 10.4375819886394
9.79591836734694 10.5652518021869
10.1020408163265 11.0199692211883
10.4081632653061 11.2273443279996
10.7142857142857 11.624351705482
11.0204081632653 11.8883403293611
11.3265306122449 12.1966206821576
11.6326530612245 12.5578600869982
11.9387755102041 12.7833817149523
12.2448979591837 13.134665207602
12.5510204081633 13.3516984450714
12.8571428571429 13.7240741481944
13.1632653061224 13.9441922200175
13.469387755102 14.311615707998
13.7755102040816 14.6062828883754
14.0816326530612 14.9189834892632
14.3877551020408 15.3245697533131
14.6938775510204 15.5399575327841
15 15.8868789290694
};
\addlegendentry{$\hat \lambda_1$}
\addplot [thick, darkorange25512714, line width=1.3pt]
table {%
0 5.10033864439089
0.306122448979592 5.1391797671831
0.612244897959184 5.18094388039137
0.918367346938776 5.22601373369521
1.22448979591837 5.27484330585898
1.53061224489796 5.32797443771788
1.83673469387755 5.38605770571801
2.14285714285714 5.44987832053234
2.44897959183673 5.52038746902877
2.75510204081633 5.59873831670879
3.06122448979592 5.6863228052168
3.36734693877551 5.78479860535019
3.6734693877551 5.89608236886955
3.97959183673469 6.02226486813049
4.28571428571429 6.1653861667365
4.59183673469388 6.32703435203866
4.89795918367347 6.50785680874223
5.20408163265306 6.70725191235712
5.51020408163265 6.92349285713936
5.81632653061224 7.15421581257757
6.12244897959184 7.3969501174235
6.42857142857143 7.64945812443059
6.73469387755102 7.90986610514222
7.04081632653061 8.17666722944131
7.3469387755102 8.448671429824
7.6530612244898 8.72494371648026
7.95918367346939 9.00474807546782
8.26530612244898 9.28750175310168
8.57142857142857 9.57273975584615
8.87755102040816 9.86008785548775
9.18367346938776 10.1492422100453
9.48979591836735 10.4399539784233
9.79591836734694 10.7320176563902
10.1020408163265 11.0252621779014
10.4081632653061 11.3195440747275
10.7142857142857 11.6147421764553
11.0204081632653 11.9107534693712
11.3265306122449 12.2074898351974
11.6326530612245 12.5048754628052
11.9387755102041 12.8028447765991
12.2448979591837 13.1013407681601
12.5510204081633 13.400313642335
12.8571428571429 13.6997197112056
13.1632653061224 13.9995204851443
13.469387755102 14.2996819211758
13.7755102040816 14.60017379791
14.0816326530612 14.9009691930368
14.3877551020408 15.2020440441338
14.6938775510204 15.5033767779944
15 15.8049479959547
};
\addlegendentry{$\lambda_1$}

\nextgroupplot[
axis line style={lightgray204},
legend cell align={left},
legend style={
  fill opacity=0.8,
  draw opacity=1,
  text opacity=1,
  at={(0.97,0.03)},
  anchor=south east,
  draw=lightgray204
},
width=.3\textwidth,
height=.32\textwidth,
tick align=outside,
x grid style={lightgray204},
xlabel=\textcolor{darkslategray38}{\(\displaystyle \beta_1\)},
xmajorgrids,
xmajorticks=true,
xmin=-0.75, xmax=15.75,
xtick style={color=darkslategray38},
y grid style={lightgray204},
ymajorgrids,
ymajorticks=true,
ymin=-0.05, ymax=1.05,
ytick style={color=darkslategray38}
]
\addplot [semithick, steelblue31119180, mark=*, mark size=2, mark options={solid}, only marks, opacity=0.4]
table {%
0 0.511580433692535
0.306122448979592 0.510735894652918
0.612244897959184 0.517133986040996
0.918367346938776 0.537079967777238
1.22448979591837 0.541710181151406
1.53061224489796 0.570639059769243
1.83673469387755 0.573988391864707
2.14285714285714 0.601401508449372
2.44897959183673 0.633949911475929
2.75510204081633 0.679866318337801
3.06122448979592 0.682676611804476
3.36734693877551 0.685539332930463
3.6734693877551 0.73678521899881
3.97959183673469 0.773849694777971
4.28571428571429 0.793683587067168
4.59183673469388 0.816682500306288
4.89795918367347 0.839693427875512
5.20408163265306 0.891674085041098
5.51020408163265 0.888818622607556
5.81632653061224 0.924415847505284
6.12244897959184 0.926625354353264
6.42857142857143 0.950807816835553
6.73469387755102 0.952793704862617
7.04081632653061 0.96374713545866
7.3469387755102 0.963599292245647
7.6530612244898 0.97023057616127
7.95918367346939 0.975381270325259
8.26530612244898 0.976079465701112
8.57142857142857 0.978768128822156
8.87755102040816 0.980182189064784
9.18367346938776 0.982609537676736
9.48979591836735 0.98518180389968
9.79591836734694 0.987871667461637
10.1020408163265 0.986226946726453
10.4081632653061 0.986742476875918
10.7142857142857 0.989287962207819
11.0204081632653 0.989552728595612
11.3265306122449 0.990120128435342
11.6326530612245 0.990191658895642
11.9387755102041 0.991856231345316
12.2448979591837 0.992327354219351
12.5510204081633 0.992885017251549
12.8571428571429 0.992552411028294
13.1632653061224 0.992541522933433
13.469387755102 0.994657981551508
13.7755102040816 0.993995627979823
14.0816326530612 0.9945632566164
14.3877551020408 0.994461036199277
14.6938775510204 0.994903705745556
15 0.994487551021358
};
\addlegendentry{$\hat\rho_{11}$}
\addplot [semithick, darkorange25512714, mark=*, mark size=2, mark options={solid}, only marks, opacity=0.4]
table {%
0 0.993709703302266
0.306122448979592 0.986048733384978
0.612244897959184 0.992432099374432
0.918367346938776 0.985319580272814
1.22448979591837 0.983604569329485
1.53061224489796 0.989441015398485
1.83673469387755 0.986900534735645
2.14285714285714 0.978483207937456
2.44897959183673 0.982327661295646
2.75510204081633 0.966417996540964
3.06122448979592 0.965681808049234
3.36734693877551 0.96709790396714
3.6734693877551 0.94608098463366
3.97959183673469 0.9274428759181
4.28571428571429 0.916289977632279
4.59183673469388 0.897006293556077
4.89795918367347 0.883108627663682
5.20408163265306 0.831247316929619
5.51020408163265 0.833876152250433
5.81632653061224 0.780894580862702
6.12244897959184 0.781160328855145
6.42857142857143 0.734140221363608
6.73469387755102 0.731599306231812
7.04081632653061 0.704084232815527
7.3469387755102 0.704579672949939
7.6530612244898 0.68474185837884
7.95918367346939 0.660984827050006
8.26530612244898 0.668986759866228
8.57142857142857 0.658716624816283
8.87755102040816 0.65297140116524
9.18367346938776 0.64319405431714
9.48979591836735 0.632037971609603
9.79591836734694 0.622510697091553
10.1020408163265 0.630151037830299
10.4081632653061 0.627120863119696
10.7142857142857 0.613634683287109
11.0204081632653 0.612261762241407
11.3265306122449 0.608694560736358
11.6326530612245 0.61013907862148
11.9387755102041 0.599736529592212
12.2448979591837 0.595216056916031
12.5510204081633 0.590862503009708
12.8571428571429 0.596709910195264
13.1632653061224 0.595212655010211
13.469387755102 0.580612449272717
13.7755102040816 0.584317861928862
14.0816326530612 0.581763529244959
14.3877551020408 0.582741115450934
14.6938775510204 0.579483235962036
15 0.581779797893066
};
\addlegendentry{$\hat\rho_{12}$}
\addplot [thick, forestgreen4416044, line width=1.3pt]
table {%
0 0.496609500507431
0.306122448979592 0.508474104668416
0.612244897959184 0.521223257609768
0.918367346938776 0.534964623334802
1.22448979591837 0.549822249577729
1.53061224489796 0.565938553405003
1.83673469387755 0.58347565827143
2.14285714285714 0.602615008804086
2.44897959183673 0.623553126776454
2.75510204081633 0.64648948168538
3.06122448979592 0.671599199956136
3.36734693877551 0.698978813425324
3.6734693877551 0.728549691473543
3.97959183673469 0.759911363110877
4.28571428571429 0.792179160675193
4.59183673469388 0.823934573462793
4.89795918367347 0.853478469372119
5.20408163265306 0.879387195429317
5.51020408163265 0.900989476734471
5.81632653061224 0.91839486451933
6.12244897959184 0.932177480728199
6.42857142857143 0.943043975859929
6.73469387755102 0.951644671145064
7.04081632653061 0.958509175971862
7.3469387755102 0.96404527606906
7.6530612244898 0.968559568893557
7.95918367346939 0.97228097365123
8.26530612244898 0.975380754069071
8.57142857142857 0.977987902605716
8.87755102040816 0.980200459048983
9.18367346938776 0.98209369467629
9.48979591836735 0.983726003838125
9.79591836734694 0.985143155591325
10.1020408163265 0.986381380938678
10.4081632653061 0.987469632854337
10.7142857142857 0.98843125593023
11.0204081632653 0.989285231748913
11.3265306122449 0.990047116575566
11.6326530612245 0.990729754166348
11.9387755102041 0.991343822486706
12.2448979591837 0.991898256874374
12.5510204081633 0.992400579969592
12.8571428571429 0.992857161534504
13.1632653061224 0.99327342412561
13.469387755102 0.993654007342208
13.7755102040816 0.994002899822358
14.0816326530612 0.994323545991778
14.3877551020408 0.994618932953647
14.6938775510204 0.994891661597806
15 0.995144005142787
};
\addlegendentry{$\rho_{11}$}
\addplot [thick, crimson2143940, line width=1.3pt]
table {%
0 0.993219001015151
0.306122448979592 0.993232685660997
0.612244897959184 0.993030574461547
0.918367346938776 0.992554530729609
1.22448979591837 0.99172933029512
1.53061224489796 0.990456917633975
1.83673469387755 0.988608595730863
2.14285714285714 0.986014517985551
2.44897959183673 0.9824499208287
2.75510204081633 0.977618140550017
3.06122448979592 0.971132489042058
3.36734693877551 0.96250437968579
3.6734693877551 0.951156710043489
3.97959183673469 0.936500513483576
4.28571428571429 0.91812494904813
4.59183673469388 0.896102255988513
4.89795918367347 0.871244026110076
5.20408163265306 0.845008806252141
5.51020408163265 0.819002174281159
5.81632653061224 0.794454287374959
6.12244897959184 0.772029683254758
6.42857142857143 0.751929475617759
6.73469387755102 0.73407810330906
7.04081632653061 0.718272078982621
7.3469387755102 0.704268356113593
7.6530612244898 0.691828263494332
7.95918367346939 0.680735837717774
8.26530612244898 0.670803213155484
8.57142857142857 0.66187011976861
8.87755102040816 0.653801056007241
9.18367346938776 0.646481821456102
9.48979591836735 0.6398161458439
9.79591836734694 0.633722694100558
10.1020408163265 0.628132518024155
10.4081632653061 0.622986935512687
10.7142857142857 0.618235785375152
11.0204081632653 0.613835998396283
11.3265306122449 0.60975042926593
11.6326530612245 0.605946900669653
11.9387755102041 0.602397419435762
12.2448979591837 0.599077531623483
12.5510204081633 0.595965790154114
12.8571428571429 0.593043313338998
13.1632653061224 0.590293417500579
13.469387755102 0.587701309849421
13.7755102040816 0.585253830666006
14.0816326530612 0.582939236002012
14.3877551020408 0.580747013753215
14.6938775510204 0.578667727422
15 0.576692882805176
};
\addlegendentry{$\rho_{12}$}
\addplot [semithick, black, dashed, line width=1pt]
table {%
5 0
5 1
};
\addlegendentry{$\beta_1 = \beta_2$}
\addplot [semithick, darkviolet1910191, dashed, line width=1pt]
table {%
0 0.5
15 0.5
};
\addlegendentry{$\alpha = 0.5$}
\end{groupplot}

\end{tikzpicture}

%% file: figs/limiting_law.tex
\begin{tikzpicture}

\definecolor{darkslategray38}{RGB}{38,38,38}
\definecolor{forestgreen4416044}{RGB}{44,160,44}
\definecolor{lightgray204}{RGB}{204,204,204}

\begin{axis}[
axis line style={lightgray204},
legend cell align={left},
legend style={
  fill opacity=0.8,
  draw opacity=1,
  text opacity=1,
  at={(0.5,0.09)},
  anchor=south,
  draw=lightgray204
},
width=0.45\textwidth,
height=0.27\textwidth,
tick align=outside,
x grid style={lightgray204},
xmajorgrids,
xmajorticks=true,
xmin=-2.2, xmax=2.2,
xtick style={color=darkslategray38},
y grid style={lightgray204},
ylabel=\textcolor{darkslategray38}{Density},
ymajorgrids,
ymajorticks=true,
ymin=0, ymax=0.565840568769189,
ytick style={color=darkslategray38}
]
\draw[draw=white,fill=forestgreen4416044,fill opacity=0.75] (axis cs:-1.5141025800361,0) rectangle (axis cs:-1.41513477339436,0.0505214793543919);
\addlegendimage{ybar,ybar legend,draw=white,fill=forestgreen4416044,fill opacity=0.75}
\addlegendentry{Eigenvalues of $\mM$}

\draw[draw=white,fill=forestgreen4416044,fill opacity=0.75] (axis cs:-1.41513477339436,0) rectangle (axis cs:-1.31616696675262,0.151564438063176);
\draw[draw=white,fill=forestgreen4416044,fill opacity=0.75] (axis cs:-1.31616696675262,0) rectangle (axis cs:-1.21719916011088,0.202085917417568);
\draw[draw=white,fill=forestgreen4416044,fill opacity=0.75] (axis cs:-1.21719916011088,0) rectangle (axis cs:-1.11823135346915,0.286288383008221);
\draw[draw=white,fill=forestgreen4416044,fill opacity=0.75] (axis cs:-1.11823135346915,0) rectangle (axis cs:-1.01926354682741,0.26944788989009);
\draw[draw=white,fill=forestgreen4416044,fill opacity=0.75] (axis cs:-1.01926354682741,0) rectangle (axis cs:-0.92029574018567,0.303128876126351);
\draw[draw=white,fill=forestgreen4416044,fill opacity=0.75] (axis cs:-0.92029574018567,0) rectangle (axis cs:-0.821327933543932,0.353650355480743);
\draw[draw=white,fill=forestgreen4416044,fill opacity=0.75] (axis cs:-0.821327933543932,0) rectangle (axis cs:-0.722360126902194,0.353650355480743);
\draw[draw=white,fill=forestgreen4416044,fill opacity=0.75] (axis cs:-0.722360126902194,0) rectangle (axis cs:-0.623392320260457,0.370490848598874);
\draw[draw=white,fill=forestgreen4416044,fill opacity=0.75] (axis cs:-0.623392320260457,0) rectangle (axis cs:-0.524424513618719,0.353650355480743);
\draw[draw=white,fill=forestgreen4416044,fill opacity=0.75] (axis cs:-0.524424513618719,0) rectangle (axis cs:-0.425456706976981,0.387331341717004);
\draw[draw=white,fill=forestgreen4416044,fill opacity=0.75] (axis cs:-0.425456706976981,0) rectangle (axis cs:-0.326488900335244,0.421012327953266);
\draw[draw=white,fill=forestgreen4416044,fill opacity=0.75] (axis cs:-0.326488900335244,0) rectangle (axis cs:-0.227521093693506,0.421012327953266);
\draw[draw=white,fill=forestgreen4416044,fill opacity=0.75] (axis cs:-0.227521093693506,0) rectangle (axis cs:-0.128553287051768,0.454693314189527);
\draw[draw=white,fill=forestgreen4416044,fill opacity=0.75] (axis cs:-0.128553287051768,0) rectangle (axis cs:-0.0295854804100306,0.505214793543919);
\draw[draw=white,fill=forestgreen4416044,fill opacity=0.75] (axis cs:-0.0295854804100306,0) rectangle (axis cs:0.069382326231707,0.53889577978018);
\draw[draw=white,fill=forestgreen4416044,fill opacity=0.75] (axis cs:0.069382326231707,0) rectangle (axis cs:0.168350132873445,0.488374300425788);
\draw[draw=white,fill=forestgreen4416044,fill opacity=0.75] (axis cs:0.168350132873445,0) rectangle (axis cs:0.267317939515182,0.454693314189527);
\draw[draw=white,fill=forestgreen4416044,fill opacity=0.75] (axis cs:0.267317939515182,0) rectangle (axis cs:0.36628574615692,0.437852821071397);
\draw[draw=white,fill=forestgreen4416044,fill opacity=0.75] (axis cs:0.36628574615692,0) rectangle (axis cs:0.465253552798658,0.370490848598874);
\draw[draw=white,fill=forestgreen4416044,fill opacity=0.75] (axis cs:0.465253552798658,0) rectangle (axis cs:0.564221359440395,0.353650355480743);
\draw[draw=white,fill=forestgreen4416044,fill opacity=0.75] (axis cs:0.564221359440395,0) rectangle (axis cs:0.663189166082133,0.370490848598873);
\draw[draw=white,fill=forestgreen4416044,fill opacity=0.75] (axis cs:0.663189166082133,0) rectangle (axis cs:0.762156972723871,0.353650355480744);
\draw[draw=white,fill=forestgreen4416044,fill opacity=0.75] (axis cs:0.762156972723871,0) rectangle (axis cs:0.861124779365608,0.370490848598873);
\draw[draw=white,fill=forestgreen4416044,fill opacity=0.75] (axis cs:0.861124779365608,0) rectangle (axis cs:0.960092586007346,0.303128876126352);
\draw[draw=white,fill=forestgreen4416044,fill opacity=0.75] (axis cs:0.960092586007346,0) rectangle (axis cs:1.05906039264908,0.336809862362612);
\draw[draw=white,fill=forestgreen4416044,fill opacity=0.75] (axis cs:1.05906039264908,0) rectangle (axis cs:1.15802819929082,0.269447889890091);
\draw[draw=white,fill=forestgreen4416044,fill opacity=0.75] (axis cs:1.15802819929082,0) rectangle (axis cs:1.25699600593256,0.26944788989009);
\draw[draw=white,fill=forestgreen4416044,fill opacity=0.75] (axis cs:1.25699600593256,0) rectangle (axis cs:1.3559638125743,0.185245424299437);
\draw[draw=white,fill=forestgreen4416044,fill opacity=0.75] (axis cs:1.3559638125743,0) rectangle (axis cs:1.45493161921603,0.117883451826914);
\addplot [thick, black, line width=1.3pt]
table {%
-2 1.47895717263261e-16
-1.97989949748744 1.53915011331877e-16
-1.95979899497487 1.60399000394992e-16
-1.93969849246231 1.67405089070707e-16
-1.91959798994975 1.75000876622563e-16
-1.89949748743719 1.83266589407964e-16
-1.87939698492462 1.92298258729284e-16
-1.85929648241206 2.02211932638681e-16
-1.8391959798995 2.1314933930831e-16
-1.81909547738693 2.25285634018427e-16
-1.79899497487437 2.38840200631348e-16
-1.77889447236181 2.54092034994193e-16
-1.75879396984925 2.71402197379662e-16
-1.73869346733668 2.91247521610738e-16
-1.71859296482412 3.14272902738998e-16
-1.69849246231156 3.41375611892025e-16
-1.67839195979899 3.73847693554224e-16
-1.65829145728643 4.13630491525321e-16
-1.63819095477387 4.63803051291545e-16
-1.61809045226131 5.29609424956524e-16
-1.59798994974874 6.2090503294979e-16
-1.57788944723618 7.59124822208642e-16
-1.55778894472362 1.00376737748864e-15
-1.53768844221106 1.63436104619533e-15
-1.51758793969849 0.0313351485270386
-1.49748743718593 0.0739053278771045
-1.47738693467337 0.0994019670682313
-1.4572864321608 0.119323489191002
-1.43718592964824 0.136139520903076
-1.41708542713568 0.150891714487733
-1.39698492462312 0.164137286254538
-1.37688442211055 0.176216105356467
-1.35678391959799 0.1873537708616
-1.33668341708543 0.197709201788255
-1.31658291457286 0.207399399929158
-1.2964824120603 0.216513498341227
-1.27638190954774 0.225121271582443
-1.25628140703518 0.233278564876913
-1.23618090452261 0.241030905211174
-1.21608040201005 0.248415988902186
-1.19597989949749 0.25546544480757
-1.17587939698492 0.262206116266505
-1.15577889447236 0.268661013299481
-1.1356783919598 0.274850033373683
-1.11557788944724 0.280790516063036
-1.09547738693467 0.286497676456468
-1.07537688442211 0.291984948442244
-1.05527638190955 0.297264259872462
-1.03517587939698 0.302346256081692
-1.01507537688442 0.307240483151512
-0.994974874371859 0.311955540129074
-0.974874371859296 0.316499207183399
-0.954773869346734 0.32087855444157
-0.934673366834171 0.325100036331809
-0.914572864321608 0.329169574607283
-0.894472361809045 0.333092633018714
-0.874371859296482 0.336874286458017
-0.85427135678392 0.340519286711558
-0.834170854271357 0.344032127717982
-0.814070351758794 0.34741711231424
-0.793969849246231 0.350678423814126
-0.773869346733668 0.353820205295166
-0.753768844221105 0.356846650326941
-0.733668341708543 0.359762109927642
-0.71356783919598 0.362571220550031
-0.693467336683417 0.365279059772847
-0.673366834170854 0.367891336152337
-0.653266331658291 0.37041462104432
-0.633165829145729 0.372856628401796
-0.613065326633166 0.375226546504021
-0.592964824120603 0.377535417647793
-0.57286432160804 0.379796548752892
-0.552763819095477 0.382025910343111
-0.532663316582915 0.384242442751603
-0.512562814070352 0.386468133560295
-0.492462311557789 0.388727666644443
-0.472361809045226 0.39104739683416
-0.452261306532663 0.393453425890966
-0.4321608040201 0.395968712895541
-0.412060301507538 0.398609488957388
-0.391959798994975 0.401381688296101
-0.371859296482412 0.404278425500969
-0.351758793969849 0.407279447276565
-0.331658291457286 0.410352880120347
-0.311557788944724 0.413458772946333
-0.291457286432161 0.416553377152169
-0.271356783919598 0.419593086898737
-0.251256281407035 0.422537359207296
-0.231155778894472 0.425350428238596
-0.21105527638191 0.42800197878066
-0.190954773869347 0.430467089194477
-0.170854271356784 0.432725750224103
-0.150753768844221 0.434762194429436
-0.130653266331658 0.436564188490168
-0.110552763819096 0.438122373625659
-0.0904522613065326 0.439429693765335
-0.0703517587939699 0.440480923853559
-0.050251256281407 0.441272296420198
-0.0301507537688441 0.441801218403891
-0.0100502512562815 0.442066068801428
0.0100502512562812 0.442066068801468
0.0301507537688441 0.441801218403885
0.050251256281407 0.441272296420106
0.0703517587939699 0.440480923853596
0.0904522613065328 0.439429693765583
0.110552763819095 0.438122373625991
0.130653266331658 0.436564188489882
0.150753768844221 0.434762194429952
0.170854271356784 0.432725750224338
0.190954773869347 0.430467089194235
0.21105527638191 0.428001978781509
0.231155778894472 0.425350428239078
0.251256281407035 0.422537359206681
0.271356783919598 0.419593086900398
0.291457286432161 0.416553377151746
0.311557788944724 0.41345877294712
0.331658291457286 0.410352880121667
0.351758793969849 0.407279447275252
0.371859296482412 0.404278425502944
0.391959798994975 0.401381688296583
0.412060301507538 0.398609488961699
0.4321608040201 0.395968712897287
0.452261306532663 0.393453425890497
0.472361809045226 0.391047396838693
0.492462311557789 0.388727666638848
0.512562814070352 0.386468133564267
0.532663316582914 0.38424244275096
0.552763819095477 0.382025910336339
0.57286432160804 0.379796548751794
0.592964824120603 0.377535417638554
0.613065326633166 0.375226546496133
0.633165829145728 0.372856628405589
0.653266331658291 0.370414621033183
0.673366834170854 0.367891336159563
0.693467336683417 0.365279059776556
0.71356783919598 0.362571220554836
0.733668341708543 0.359762109920947
0.753768844221105 0.356846650332757
0.773869346733668 0.353820205283402
0.793969849246231 0.350678423827389
0.814070351758794 0.347417112319141
0.834170854271357 0.344032127712272
0.85427135678392 0.340519286719282
0.874371859296482 0.336874286465692
0.894472361809045 0.333092633031206
0.914572864321608 0.329169574616507
0.934673366834171 0.325100036332789
0.954773869346734 0.320878554439687
0.974874371859296 0.316499207189953
0.994974874371859 0.311955540121099
1.01507537688442 0.307240483148028
1.03517587939699 0.302346256098372
1.05527638190955 0.297264259887338
1.07537688442211 0.291984948449949
1.09547738693467 0.286497676475338
1.11557788944724 0.280790516054076
1.1356783919598 0.274850033386405
1.15577889447236 0.268661013300949
1.17587939698492 0.262206116269175
1.19597989949749 0.255465444809006
1.21608040201005 0.24841598893648
1.23618090452261 0.24103090519896
1.25628140703518 0.233278564858952
1.27638190954774 0.22512127158142
1.2964824120603 0.216513498338902
1.31658291457286 0.207399399892247
1.33668341708543 0.197709201790615
1.35678391959799 0.187353770895035
1.37688442211055 0.176216105327015
1.39698492462312 0.164137286299951
1.41708542713568 0.150891714450115
1.43718592964824 0.136139520899554
1.4572864321608 0.119323489222144
1.47738693467337 0.0994019670655548
1.49748743718593 0.0739053279270648
1.51758793969849 0.0313351488636833
1.53768844221106 1.63436104861227e-15
1.55778894472362 1.0037673781255e-15
1.57788944723618 7.59124822707349e-16
1.59798994974874 6.20905033207472e-16
1.61809045226131 5.29609425094257e-16
1.63819095477387 4.63803051333712e-16
1.65829145728643 4.13630490749969e-16
1.67839195979899 3.73847693561523e-16
1.69849246231156 3.41375611892737e-16
1.71859296482412 3.14272902730248e-16
1.73869346733668 2.91247521604064e-16
1.75879396984925 2.71402197360478e-16
1.77889447236181 2.54092034983331e-16
1.79899497487437 2.38840200597354e-16
1.81909547738693 2.25285633998668e-16
1.8391959798995 2.13149339296797e-16
1.85929648241206 2.02211932593612e-16
1.87939698492462 1.922982587011e-16
1.89949748743719 1.83266589390147e-16
1.91959798994975 1.75000876731278e-16
1.93969849246231 1.67405089133305e-16
1.95979899497487 1.60399000354821e-16
1.97989949748744 1.53915011304099e-16
2 1.47895717243855e-16
};
\addlegendentry{Limiting Measure}
\addplot [thick, red, dash pattern=on 1pt off 3pt on 3pt off 3pt, line width=1.3pt]
table {%
-2 -0
-1.97989949748744 -0
-1.95979899497487 -0
-1.93969849246231 -0
-1.91959798994975 -0
-1.89949748743719 -0
-1.87939698492462 -0
-1.85929648241206 -0
-1.8391959798995 -0
-1.81909547738693 -0
-1.79899497487437 -0
-1.77889447236181 -0
-1.75879396984925 -0
-1.73869346733668 -0
-1.71859296482412 -0
-1.69849246231156 -0
-1.67839195979899 -0
-1.65829145728643 -0
-1.63819095477387 -0
-1.61809045226131 0.0525482284801871
-1.59798994974874 0.0802846807061343
-1.57788944723618 0.100418572362237
-1.55778894472362 0.116944839449879
-1.53768844221106 0.131233560311949
-1.51758793969849 0.143952580073599
-1.49748743718593 0.15548758582937
-1.47738693467337 0.166085458315139
-1.4572864321608 0.175915649275321
-1.43718592964824 0.185100507659136
-1.41708542713568 0.193731841190047
-1.39698492462312 0.201880659612568
-1.37688442211055 0.209603247207264
-1.35678391959799 0.216945126318979
-1.33668341708543 0.223943744067097
-1.31658291457286 0.230630351991071
-1.2964824120603 0.237031356769026
-1.27638190954774 0.243169313449517
-1.25628140703518 0.249063670563889
-1.23618090452261 0.254731338993673
-1.21608040201005 0.260187133071114
-1.19597989949749 0.265444117368469
-1.17587939698492 0.27051388273948
-1.15577889447236 0.275406768514424
-1.1356783919598 0.280132043172196
-1.11557788944724 0.284698052609436
-1.09547738693467 0.289112342847699
-1.07537688442211 0.293381762373991
-1.05527638190955 0.297512548105046
-1.03517587939698 0.301510398072354
-1.01507537688442 0.305380533254897
-0.994974874371859 0.309127750478415
-0.974874371859296 0.312756467910976
-0.954773869346734 0.31627076438387
-0.934673366834171 0.319674413532377
-0.914572864321608 0.322970913566663
-0.894472361809045 0.326163513337107
-0.874371859296482 0.329255235241946
-0.85427135678392 0.332248895431658
-0.834170854271357 0.335147121688952
-0.814070351758794 0.337952369301837
-0.793969849246231 0.34066693519707
-0.773869346733668 0.34329297056
-0.753768844221105 0.345832492132762
-0.733668341708543 0.348287392354517
-0.71356783919598 0.350659448483803
-0.693467336683417 0.352950330823368
-0.673366834170854 0.355161610151194
-0.653266331658291 0.357294764447447
-0.633165829145729 0.359351184995164
-0.613065326633166 0.361332181922403
-0.592964824120603 0.363238989244939
-0.57286432160804 0.365072769461198
-0.552763819095477 0.366834617744777
-0.532663316582915 0.36852556577441
-0.512562814070352 0.370146585236504
-0.492462311557789 0.371698591031288
-0.472361809045226 0.373182444209999
-0.452261306532663 0.374598954667464
-0.4321608040201 0.37594888361168
-0.412060301507538 0.377232945829609
-0.391959798994975 0.378451811766313
-0.371859296482412 0.379606109432697
-0.351758793969849 0.380696426155482
-0.331658291457286 0.381723310181588
-0.311557788944724 0.3826872721478
-0.291457286432161 0.38358878642547
-0.271356783919598 0.384428292348923
-0.251256281407035 0.385206195335377
-0.231155778894472 0.385922867903301
-0.21105527638191 0.386578650595406
-0.190954773869347 0.387173852811789
-0.170854271356784 0.387708753558117
-0.150753768844221 0.388183602113171
-0.130653266331658 0.388598618619558
-0.110552763819096 0.38895399460091
-0.0904522613065326 0.389249893408444
-0.0703517587939699 0.389486450599336
-0.050251256281407 0.389663774248976
-0.0301507537688441 0.389781945198785
-0.0100502512562815 0.389841017240932
0.0100502512562812 0.389841017240932
0.0301507537688441 0.389781945198785
0.050251256281407 0.389663774248976
0.0703517587939699 0.389486450599336
0.0904522613065328 0.389249893408444
0.110552763819095 0.38895399460091
0.130653266331658 0.388598618619558
0.150753768844221 0.388183602113171
0.170854271356784 0.387708753558117
0.190954773869347 0.387173852811789
0.21105527638191 0.386578650595406
0.231155778894472 0.385922867903301
0.251256281407035 0.385206195335377
0.271356783919598 0.384428292348923
0.291457286432161 0.38358878642547
0.311557788944724 0.3826872721478
0.331658291457286 0.381723310181588
0.351758793969849 0.380696426155482
0.371859296482412 0.379606109432697
0.391959798994975 0.378451811766313
0.412060301507538 0.377232945829609
0.4321608040201 0.37594888361168
0.452261306532663 0.374598954667464
0.472361809045226 0.373182444209999
0.492462311557789 0.371698591031288
0.512562814070352 0.370146585236504
0.532663316582914 0.36852556577441
0.552763819095477 0.366834617744777
0.57286432160804 0.365072769461198
0.592964824120603 0.363238989244939
0.613065326633166 0.361332181922403
0.633165829145728 0.359351184995164
0.653266331658291 0.357294764447447
0.673366834170854 0.355161610151194
0.693467336683417 0.352950330823368
0.71356783919598 0.350659448483803
0.733668341708543 0.348287392354517
0.753768844221105 0.345832492132762
0.773869346733668 0.34329297056
0.793969849246231 0.34066693519707
0.814070351758794 0.337952369301837
0.834170854271357 0.335147121688952
0.85427135678392 0.332248895431658
0.874371859296482 0.329255235241946
0.894472361809045 0.326163513337107
0.914572864321608 0.322970913566663
0.934673366834171 0.319674413532377
0.954773869346734 0.31627076438387
0.974874371859296 0.312756467910976
0.994974874371859 0.309127750478415
1.01507537688442 0.305380533254897
1.03517587939699 0.301510398072354
1.05527638190955 0.297512548105046
1.07537688442211 0.293381762373991
1.09547738693467 0.289112342847699
1.11557788944724 0.284698052609436
1.1356783919598 0.280132043172196
1.15577889447236 0.275406768514424
1.17587939698492 0.27051388273948
1.19597989949749 0.265444117368469
1.21608040201005 0.260187133071114
1.23618090452261 0.254731338993673
1.25628140703518 0.249063670563889
1.27638190954774 0.243169313449517
1.2964824120603 0.237031356769026
1.31658291457286 0.230630351991071
1.33668341708543 0.223943744067097
1.35678391959799 0.216945126318979
1.37688442211055 0.209603247207264
1.39698492462312 0.201880659612568
1.41708542713568 0.193731841190047
1.43718592964824 0.185100507659135
1.4572864321608 0.175915649275321
1.47738693467337 0.166085458315139
1.49748743718593 0.15548758582937
1.51758793969849 0.143952580073599
1.53768844221106 0.131233560311948
1.55778894472362 0.116944839449879
1.57788944723618 0.100418572362237
1.59798994974874 0.0802846807061343
1.61809045226131 0.0525482284801871
1.63819095477387 -0
1.65829145728643 -0
1.67839195979899 -0
1.69849246231156 -0
1.71859296482412 -0
1.73869346733668 -0
1.75879396984925 -0
1.77889447236181 -0
1.79899497487437 -0
1.81909547738693 -0
1.8391959798995 -0
1.85929648241206 -0
1.87939698492462 -0
1.89949748743719 -0
1.91959798994975 -0
1.93969849246231 -0
1.95979899497487 -0
1.97989949748744 -0
2 -0
};
\addlegendentry{Semi-circle Law}
\end{axis}

\end{tikzpicture}

%% file: figs/optimization_of_gamma_alpha_0.6.tex
\begin{tikzpicture}

\definecolor{darkorange25512714}{RGB}{255,127,14}
\definecolor{darkslategray38}{RGB}{38,38,38}
\definecolor{darkviolet1910191}{RGB}{191,0,191}
\definecolor{forestgreen4416044}{RGB}{44,160,44}
\definecolor{lightgray204}{RGB}{204,204,204}
\definecolor{steelblue31119180}{RGB}{31,119,180}

\begin{groupplot}[group style={group size=3 by 1}]
\nextgroupplot[
axis line style={lightgray204},
legend cell align={left},
legend style={
  fill opacity=0.8,
  draw opacity=1,
  text opacity=1,
  at={(0.03,0.03)},
  anchor=south west,
  draw=lightgray204
},
width=0.37\textwidth,
height=0.28\textwidth,
tick align=outside,
x grid style={lightgray204},
xlabel=\textcolor{darkslategray38}{\(\displaystyle \gamma\)},
xmajorgrids,
xmajorticks=true,
xmin=-0.05, xmax=1.05,
xtick style={color=darkslategray38},
y grid style={lightgray204},
ymajorgrids,
ymajorticks=true,
ymin=-0.05, ymax=1.05,
ytick style={color=darkslategray38}
]
\addplot [semithick, steelblue31119180, line width=1.3pt]
table {%
0 0.915742036743531
0.0101010101010101 0.915733451268328
0.0202020202020202 0.915724507505299
0.0303030303030303 0.915715188424193
0.0404040404040404 0.915705476059759
0.0505050505050505 0.915695351241399
0.0606060606060606 0.915684793906031
0.0707070707070707 0.915673782775754
0.0808080808080808 0.915662295167898
0.0909090909090909 0.915650307113226
0.101010101010101 0.91563779320473
0.111111111111111 0.915624726458407
0.121212121212121 0.915611078236601
0.131313131313131 0.915596818069063
0.141414141414141 0.915581913574472
0.151515151515152 0.915566330255955
0.161616161616162 0.91555003143816
0.171717171717172 0.915532977994552
0.181818181818182 0.915515128194919
0.191919191919192 0.91549643751091
0.202020202020202 0.915476858371334
0.212121212121212 0.915456339930595
0.222222222222222 0.915434827758262
0.232323232323232 0.915412263617927
0.242424242424242 0.915388585057256
0.252525252525253 0.915363725100805
0.262626262626263 0.91533761181832
0.272727272727273 0.915310167936481
0.282828282828283 0.915281310293806
0.292929292929293 0.915250949368885
0.303030303030303 0.915218988632475
0.313131313131313 0.915185323957193
0.323232323232323 0.915149842810489
0.333333333333333 0.915112423575605
0.343434343434343 0.91507293434286
0.353535353535354 0.915031232441672
0.363636363636364 0.914987162752916
0.373737373737374 0.91494055679261
0.383838383838384 0.914891231191273
0.393939393939394 0.91483898604906
0.404040404040404 0.914783603366276
0.414141414141414 0.914724844622393
0.424242424242424 0.914662448811023
0.434343434343434 0.914596130089885
0.444444444444444 0.914525573874917
0.454545454545455 0.914450434773879
0.464646464646465 0.914370331627624
0.474747474747475 0.914284843625421
0.484848484848485 0.914193505386815
0.494949494949495 0.914095800861274
0.505050505050505 0.913991156940988
0.515151515151515 0.913878935329341
0.525252525252525 0.91375842341718
0.535353535353535 0.913628824816497
0.545454545454546 0.9134892453942
0.555555555555556 0.913338680466685
0.565656565656566 0.913175995871623
0.575757575757576 0.912999909627995
0.585858585858586 0.912808967912294
0.595959595959596 0.912601516149714
0.606060606060606 0.697659805975492
0.616161616161616 0.63541688619909
0.626262626262626 0.582959430900733
0.636363636363636 0.536421381458963
0.646464646464647 0.494002193171432
0.656565656565657 0.454671764044713
0.666666666666667 0.417767701790645
0.676767676767677 0.382830656076441
0.686868686868687 0.349524921699677
0.696969696969697 0.317595602692053
0.707070707070707 0.286843521120366
0.717171717171717 0.257109560866976
0.727272727272727 0.228264399636081
0.737373737373737 0.2002014928962
0.747474747474748 0.172832114949771
0.757575757575758 0.146081751228005
0.767676767676768 0.11988740748893
0.777777777777778 0.0941955588492308
0.787878787878788 0.0689605550484533
0.797979797979798 0.0441433600509822
0.808080808080808 0.0197105385400649
0.818181818181818 0.00436656988298618
0.828282828282828 0.028112533023182
0.838383838383838 0.0515483962878686
0.848484848484849 0.0746921647682677
0.858585858585859 0.0975592141327214
0.868686868686869 0.120162645289947
0.878787878787879 0.142513588757774
0.888888888888889 0.16462147917429
0.898989898989899 0.186494276705841
0.909090909090909 0.208138689609114
0.919191919191919 0.229560340142457
0.929292929292929 0.250763936998448
0.939393939393939 0.271753412393325
0.94949494949495 0.292532044118455
0.95959595959596 0.313102572472208
0.96969696969697 0.333467291434563
0.97979797979798 0.353628135428681
0.98989898989899 0.373586753751976
1 0.393344574691079
};
\addlegendentry{$\theta_{21}$}
\addplot [semithick, darkorange25512714, line width=1.3pt]
table {%
0 0.869062259366828
0.0101010101010101 0.869054111539202
0.0202020202020202 0.86904562368516
0.0303030303030303 0.869036779640998
0.0404040404040404 0.869027562300223
0.0505050505050505 0.869017953655297
0.0606060606060606 0.86900793454342
0.0707070707070707 0.868997484638802
0.0808080808080808 0.868986582607949
0.0909090909090909 0.868975205642683
0.101010101010101 0.868963329623161
0.111111111111111 0.868950928980424
0.121212121212121 0.868937976456688
0.131313131313131 0.868924443192868
0.141414141414141 0.868910298433809
0.151515151515152 0.868895509490293
0.161616161616162 0.868880041503696
0.171717171717172 0.868863857353375
0.181818181818182 0.868846917446437
0.191919191919192 0.868829179514065
0.202020202020202 0.868810598417716
0.212121212121212 0.868791125887236
0.222222222222222 0.868770710307171
0.232323232323232 0.868749296377464
0.242424242424242 0.868726824817373
0.252525252525253 0.868703232096704
0.262626262626263 0.868678449931278
0.272727272727273 0.868652404996194
0.282828282828283 0.868625018366869
0.292929292929293 0.868596205081141
0.303030303030303 0.868565873546266
0.313131313131313 0.868533924912568
0.323232323232323 0.868500252405251
0.333333333333333 0.868464740584049
0.343434343434343 0.868427264343519
0.353535353535354 0.868387688175538
0.363636363636364 0.868345864938814
0.373737373737374 0.868301634717182
0.383838383838384 0.868254823478461
0.393939393939394 0.868205241562422
0.404040404040404 0.868152681955493
0.414141414141414 0.868096918472154
0.424242424242424 0.868037703520128
0.434343434343434 0.867974765262002
0.444444444444444 0.867907805631724
0.454545454545455 0.867836496727644
0.464646464646465 0.867760476675238
0.474747474747475 0.867679346397621
0.484848484848485 0.86759266410834
0.494949494949495 0.867499940171474
0.505050505050505 0.867400630401804
0.515151515151515 0.867294129069334
0.525252525252525 0.867179760559927
0.535353535353535 0.867056768017743
0.545454545454546 0.866924303730646
0.555555555555556 0.866781413346398
0.565656565656566 0.866627021948928
0.575757575757576 0.866459912704205
0.585858585858586 0.866278704027019
0.595959595959596 0.866081822654295
0.606060606060606 0.98636005971171
0.616161616161616 0.993638276175949
0.626262626262626 0.994364118019804
0.636363636363636 0.991488563502892
0.646464646464647 0.986306228449923
0.656565656565657 0.97950991453731
0.666666666666667 0.971514123263399
0.676767676767677 0.962585246337199
0.686868686868687 0.952903187228078
0.696969696969697 0.942593911468275
0.707070707070707 0.931748067141544
0.717171717171717 0.920432301918139
0.727272727272727 0.90869647695786
0.737373737373737 0.896578444999239
0.747474747474748 0.884107316106819
0.757575757575758 0.87130574853254
0.767676767676768 0.858191590842497
0.777777777777778 0.844779080173828
0.787878787878788 0.831079729399068
0.797979797979798 0.817102991568218
0.808080808080808 0.80285676148351
0.818181818181818 0.78834775651646
0.828282828282828 0.773581805839185
0.838383838383838 0.758564069780958
0.848484848484849 0.743299204639264
0.858585858585859 0.727791484883901
0.868686868686869 0.712044891460242
0.878787878787879 0.696063172451713
0.888888888888889 0.679849883583076
0.898989898989899 0.663408408132781
0.909090909090909 0.646741965089475
0.919191919191919 0.629853605847833
0.929292929292929 0.612746198974712
0.939393939393939 0.595422410252281
0.94949494949495 0.577884673311722
0.95959595959596 0.560135159604248
0.96969696969697 0.542175738313367
0.97979797979798 0.524007939228548
0.98989898989899 0.505632906149537
1 0.487051351290462
};
\addlegendentry{$\theta_{22}$}
\addplot [semithick, forestgreen4416044, dashed, line width=1pt]
table {%
0 0.994364118019804
1 0.994364118019804
};
\addlegendentry{$\theta_{22}^{*}=0.994$}
\addplot [semithick, black, dashed, line width=1pt]
table {%
0.6 0
0.6 1
};
\addlegendentry{$\alpha=0.6$}
\addplot [semithick, darkviolet1910191, dashed, line width=1pt]
table {%
0.626262626262626 0
0.626262626262626 1
};
\addlegendentry{$\gamma^*(\theta_{22})=0.63$}

\nextgroupplot[
axis line style={lightgray204},
legend cell align={left},
legend style={
  fill opacity=0.8,
  draw opacity=1,
  text opacity=1,
  at={(0.03,0.03)},
  anchor=south west,
  draw=lightgray204
},
width=0.37\textwidth,
height=0.28\textwidth,
tick align=outside,
x grid style={lightgray204},
xlabel=\textcolor{darkslategray38}{\(\displaystyle \gamma\)},
xmajorgrids,
xmajorticks=true,
xmin=-0.05, xmax=1.05,
xtick style={color=darkslategray38},
y grid style={lightgray204},
ymajorgrids,
ymajorticks=true,
ymin=-0.05, ymax=1.05,
ytick style={color=darkslategray38}
]
\addplot [thick, steelblue31119180, line width=1.3pt]
table {%
0 0.915740551807119
0.0101010101010101 0.91573205312783
0.0202020202020202 0.915723377158979
0.0303030303030303 0.915714518269811
0.0404040404040404 0.915705470615414
0.0505050505050505 0.915696227936691
0.0606060606060606 0.915686783934925
0.0707070707070707 0.915677131884022
0.0808080808080808 0.915667264815567
0.0909090909090909 0.915657175355649
0.101010101010101 0.915646855871797
0.111111111111111 0.915636298328193
0.121212121212121 0.915625494351368
0.131313131313131 0.915614435129539
0.141414141414141 0.915603111429963
0.151515151515152 0.915591513542008
0.161616161616162 0.915579631331748
0.171717171717172 0.915567454075405
0.181818181818182 0.91555497056353
0.191919191919192 0.915542168956539
0.202020202020202 0.915529036815246
0.212121212121212 0.915515561028561
0.222222222222222 0.915501727746112
0.232323232323232 0.91548752240315
0.242424242424242 0.915472929571245
0.252525252525253 0.915457932974837
0.262626262626263 0.915442515380048
0.272727272727273 0.915426658553417
0.282828282828283 0.915410343161882
0.292929292929293 0.915393548709059
0.303030303030303 0.9153762534206
0.313131313131313 0.915358434180734
0.323232323232323 0.915340066366829
0.333333333333333 0.91532112380511
0.343434343434343 0.91530157843719
0.353535353535354 0.915281400559314
0.363636363636364 0.915260558215574
0.373737373737374 0.915239017302441
0.383838383838384 0.915216741262189
0.393939393939394 0.915193690832712
0.404040404040404 0.915169823896055
0.414141414141414 0.915145095010776
0.424242424242424 0.91511945523472
0.434343434343434 0.915092851921593
0.444444444444444 0.91506522774103
0.454545454545455 0.91503652097323
0.464646464646465 0.915006664355674
0.474747474747475 0.914975584669333
0.484848484848485 0.914943202266061
0.494949494949495 0.914909430012585
0.505050505050505 0.914874172611083
0.515151515151515 0.914837325307987
0.525252525252525 0.914798772518111
0.535353535353535 0.914758387295046
0.545454545454546 0.914716028305397
0.555555555555556 0.914671539212949
0.565656565656566 0.914624744789518
0.575757575757576 0.914575449685118
0.585858585858586 0.914523435038885
0.595959595959596 0.914468454353264
0.606060606060606 0.750680544439095
0.616161616161616 0.706769625803228
0.626262626262626 0.671107059262451
0.636363636363636 0.640592542448845
0.646464646464647 0.613771064624847
0.656565656565657 0.589803986804958
0.666666666666667 0.568148052890246
0.676767676767677 0.54842371952523
0.686868686868687 0.530351458553212
0.696969696969697 0.513717275596775
0.707070707070707 0.498352444868852
0.717171717171717 0.484120824162934
0.727272727272727 0.470910511143999
0.737373737373737 0.458628128472643
0.747474747474748 0.447194777647949
0.757575757575758 0.436543093342828
0.767676767676768 0.426615047548223
0.777777777777778 0.417360279868449
0.787878787878788 0.408734804486792
0.797979797979798 0.400699995136704
0.808080808080808 0.393221776371574
0.818181818181818 0.386269972707974
0.828282828282828 0.379817779018469
0.838383838383838 0.373841326272204
0.848484848484849 0.36831932229135
0.858585858585859 0.363232753245637
0.868686868686869 0.358564633471157
0.878787878787879 0.354299797496886
0.888888888888889 0.350424716721375
0.898989898989899 0.346927360840981
0.909090909090909 0.343797063832865
0.919191919191919 0.34102441717997
0.929292929292929 0.338601183918223
0.939393939393939 0.336520219357075
0.94949494949495 0.334775414658615
0.95959595959596 0.333361636261724
0.96969696969697 0.332274692755927
0.97979797979798 0.331511296308209
0.98989898989899 0.33106904166388
1 0.330946387966995
};
\addlegendentry{$\rho_{21}$}
\addplot [thick, darkorange25512714, line width=1.3pt]
table {%
0 0.869060850135494
0.0101010101010101 0.869052784667923
0.0202020202020202 0.869044550958087
0.0303030303030303 0.869036143647571
0.0404040404040404 0.869027557138492
0.0505050505050505 0.869018785661091
0.0606060606060606 0.86900982309293
0.0707070707070707 0.869000663058498
0.0808080808080808 0.868991298928702
0.0909090909090909 0.868981723776906
0.101010101010101 0.868971930314529
0.111111111111111 0.868961910966684
0.121212121212121 0.86895165771352
0.131313131313131 0.868941162223885
0.141414141414141 0.868930415735007
0.151515151515152 0.868919409062838
0.161616161616162 0.868908132536917
0.171717171717172 0.868896576022446
0.181818181818182 0.868884728851031
0.191919191919192 0.868872579807738
0.202020202020202 0.868860117074997
0.212121212121212 0.868847328200261
0.222222222222222 0.868834200080828
0.232323232323232 0.86882071885619
0.242424242424242 0.868806869885459
0.252525252525253 0.868792637741007
0.262626262626263 0.868778006053814
0.272727272727273 0.868762957525724
0.282828282828283 0.868747473810418
0.292929292929293 0.86873153544761
0.303030303030303 0.868715121792865
0.313131313131313 0.868698210876228
0.323232323232323 0.868680779344516
0.333333333333333 0.868662802328165
0.343434343434343 0.868644253349866
0.353535353535354 0.86862510401488
0.363636363636364 0.868605324120991
0.373737373737374 0.868584881260828
0.383838383838384 0.868563740733257
0.393939393939394 0.868541865319717
0.404040404040404 0.868519214966924
0.414141414141414 0.868495746666282
0.424242424242424 0.868471414041174
0.434343434343434 0.868446166746302
0.444444444444444 0.868419950711763
0.454545454545455 0.868392707287108
0.464646464646465 0.868364372462165
0.474747474747475 0.868334877047433
0.484848484848485 0.868304145331083
0.494949494949495 0.868272094699898
0.505050505050505 0.86823863449793
0.515151515151515 0.868203665318906
0.525252525252525 0.868167078014141
0.535353535353535 0.868128751263616
0.545454545454546 0.868088551602397
0.555555555555556 0.86804632993719
0.565656565656566 0.868001921164477
0.575757575757576 0.867955139695348
0.585858585858586 0.867905776344753
0.595959595959596 0.867853594695297
0.606060606060606 0.975498268428194
0.616161616161616 0.986445234812523
0.626262626262626 0.99200885969239
0.636363636363636 0.994721745034392
0.646464646464647 0.995701614784978
0.656565656565657 0.995550470888879
0.666666666666667 0.994631108849702
0.676767676767677 0.993178634350066
0.686868686868687 0.991353346216761
0.696969696969697 0.989268716028101
0.707070707070707 0.987007414994921
0.717171717171717 0.984631064715929
0.727272727272727 0.982186457119702
0.737373737373737 0.979709676263195
0.747474747474748 0.977228916571479
0.757575757575758 0.974766460674934
0.767676767676768 0.972340098350994
0.777777777777778 0.969964163673446
0.787878787878788 0.9676503052779
0.797979797979798 0.965408066365412
0.808080808080808 0.963245326478288
0.818181818181818 0.961168641372559
0.828282828282828 0.959183506604974
0.838383838383838 0.957294563383769
0.848484848484849 0.955505760047893
0.858585858585859 0.95382047944353
0.868686868686869 0.952241639232684
0.878787878787879 0.950771771481014
0.888888888888889 0.949413084867909
0.898989898989899 0.948167515031006
0.909090909090909 0.947036761933817
0.919191919191919 0.94602232146019
0.929292929292929 0.945125507642066
0.939393939393939 0.944347471090058
0.94949494949495 0.943689212379836
0.95959595959596 0.943151591943033
0.96969696969697 0.942735336469282
0.97979797979798 0.942441045710549
0.98989898989899 0.942269193481203
1 0.942220130294375
};
\addlegendentry{$\rho_{22}$}
\addplot [semithick, forestgreen4416044, dashed, line width=1pt]
table {%
0 0.995701614784978
1 0.995701614784978
};
\addlegendentry{$\rho_{22}^{*}=0.996$}
\addplot [semithick, black, dashed, line width=1pt]
table {%
0.6 0
0.6 1
};
\addlegendentry{$\alpha=0.6$}
\addplot [semithick, darkviolet1910191, dashed, line width=1pt]
table {%
0.646464646464647 0
0.646464646464647 1
};
\addlegendentry{$\gamma^*(\rho_{22})=0.65$}

\nextgroupplot[
axis line style={lightgray204},
legend cell align={left},
legend style={
  fill opacity=0.8,
  draw opacity=1,
  text opacity=1,
  at={(0.03,0.03)},
  anchor=south west,
  draw=lightgray204
},
width=0.37\textwidth,
height=0.28\textwidth,
tick align=outside,
x grid style={lightgray204},
xlabel=\textcolor{darkslategray38}{\(\displaystyle \gamma\)},
xmajorgrids,
xmajorticks=true,
xmin=-0.05, xmax=1.05,
xtick style={color=darkslategray38},
y grid style={lightgray204},
ymajorgrids,
ymajorticks=true,
ymin=-0.05, ymax=1.05,
ytick style={color=darkslategray38}
]
\addplot [semithick, steelblue31119180, line width=1.3pt]
table {%
0 0.999993525428552
0.0101010101010101 0.999984083673122
0.0202020202020202 0.999974249304312
0.0303030303030303 0.999964003662319
0.0404040404040404 0.999953327074433
0.0505050505050505 0.999942198590323
0.0606060606060606 0.999930596246612
0.0707070707070707 0.999918496695101
0.0808080808080808 0.999905875166473
0.0909090909090909 0.999892705426932
0.101010101010101 0.999878959664622
0.111111111111111 0.999864608377024
0.121212121212121 0.99984962016776
0.131313131313131 0.999833961738206
0.141414141414141 0.999817597654793
0.151515151515152 0.99980049017127
0.161616161616162 0.999782599154533
0.171717171717172 0.999763881791922
0.181818181818182 0.999744292438181
0.191919191919192 0.999723782370542
0.202020202020202 0.999702299550342
0.212121212121212 0.999679788346589
0.222222222222222 0.99965618923324
0.232323232323232 0.999631438486918
0.242424242424242 0.999605467797272
0.252525252525253 0.999578203918439
0.262626262626263 0.999549568133379
0.272727272727273 0.999519475961286
0.282828282828283 0.999487836422574
0.292929292929293 0.999454551613331
0.303030303030303 0.999419515962723
0.313131313131313 0.999382615572131
0.323232323232323 0.999343727390936
0.333333333333333 0.999302718394024
0.343434343434343 0.999259444383256
0.353535353535354 0.999213749273766
0.363636363636364 0.999165463475742
0.373737373737374 0.999114402740996
0.383838383838384 0.999060366571512
0.393939393939394 0.9990031363648
0.404040404040404 0.998942473770789
0.414141414141414 0.998878118042756
0.424242424242424 0.998809783803635
0.434343434343434 0.998737158417278
0.444444444444444 0.998659897891841
0.454545454545455 0.998577624929159
0.464646464646465 0.99848992251454
0.474747474747475 0.998396330804025
0.484848484848485 0.99829634098632
0.494949494949495 0.998189388984249
0.505050505050505 0.99807484818888
0.515151515151515 0.997952020814912
0.525252525252525 0.997820128328949
0.535353535353535 0.997678299631128
0.545454545454546 0.997525557864843
0.555555555555556 0.997360804446004
0.565656565656566 0.997182800461259
0.575757575757576 0.996990144655384
0.585858585858586 0.996781247161623
0.595959595959596 0.996554298143366
0.606060606060606 0.923962436580426
0.616161616161616 0.889137956061726
0.626262626262626 0.857079316890383
0.636363636363636 0.826875458655891
0.646464646464647 0.798062138040668
0.656565656565657 0.770349797456644
0.666666666666667 0.743536723854997
0.676767676767677 0.717472495904426
0.686868686868687 0.692039782842229
0.696969696969697 0.667144176159717
0.707070707070707 0.642708010167401
0.717171717171717 0.618666355065237
0.727272727272727 0.594964284678133
0.737373737373737 0.57155493384278
0.747474747474748 0.548398069535345
0.757575757575758 0.525459008750591
0.767676767676768 0.502707778040176
0.777777777777778 0.480118446214091
0.787878787878788 0.457668583184208
0.797979797979798 0.435338813664239
0.808080808080808 0.413112442191967
0.818181818181818 0.390975133407498
0.828282828282828 0.368914635298843
0.838383838383838 0.346920536611171
0.848484848484849 0.324984051577735
0.858585858585859 0.303097827195405
0.868686868686869 0.281255768107894
0.878787878787879 0.259452878714565
0.888888888888889 0.23768511575634
0.898989898989899 0.215949256945069
0.909090909090909 0.194242774928087
0.919191919191919 0.172563723858834
0.929292929292929 0.15091063334971
0.939393939393939 0.129282410076715
0.94949494949495 0.107678247388581
0.95959595959596 0.0860975400922158
0.96969696969697 0.0645398067065054
0.97979797979798 0.043004616931159
0.98989898989899 0.0214915236495018
1 9.09207982948046e-25
};
\addlegendentry{$\kappa$}
\addplot [semithick, darkorange25512714, line width=1.3pt]
table {%
0 0.998199131559886
0.0101010101010101 0.99818983461482
0.0202020202020202 0.998180343757713
0.0303030303030303 0.998170652806455
0.0404040404040404 0.998160755365349
0.0505050505050505 0.998150644597786
0.0606060606060606 0.998140313609615
0.0707070707070707 0.998129755039077
0.0808080808080808 0.998118961260181
0.0909090909090909 0.998107924218226
0.101010101010101 0.998096635552416
0.111111111111111 0.998085086513142
0.121212121212121 0.998073267879923
0.131313131313131 0.998061170044858
0.141414141414141 0.998048782923437
0.151515151515152 0.99803609587661
0.161616161616162 0.99802309782086
0.171717171717172 0.998009777034866
0.181818181818182 0.997996121252058
0.191919191919192 0.997982117525994
0.202020202020202 0.997967752250301
0.212121212121212 0.997953011082677
0.222222222222222 0.997937878884164
0.232323232323232 0.997922339714608
0.242424242424242 0.997906376704763
0.252525252525253 0.997889972058297
0.262626262626263 0.997873106913067
0.272727272727273 0.997855761345263
0.282828282828283 0.997837914196739
0.292929292929293 0.997819543060605
0.303030303030303 0.997800624124448
0.313131313131313 0.997781132097642
0.323232323232323 0.997761040063684
0.333333333333333 0.997740319374734
0.343434343434343 0.997718939420677
0.353535353535354 0.997696867631506
0.363636363636364 0.997674069108858
0.373737373737374 0.997650506544743
0.383838383838384 0.997626139969254
0.393939393939394 0.997600926349935
0.404040404040404 0.997574819911781
0.414141414141414 0.997547770721633
0.424242424242424 0.997519725375441
0.434343434343434 0.997490626181239
0.444444444444444 0.997460410509448
0.454545454545455 0.997429010984825
0.464646464646465 0.997396353907921
0.474747474747475 0.99736235937197
0.484848484848485 0.997326940253904
0.494949494949495 0.997290001337469
0.505050505050505 0.997251438349813
0.515151515151515 0.997211136772213
0.525252525252525 0.997168970534635
0.535353535353535 0.997124800463744
0.545454545454546 0.997078472425615
0.555555555555556 0.997029815183844
0.565656565656566 0.996978637838358
0.575757575757576 0.996924726810165
0.585858585858586 0.996867842199544
0.595959595959596 0.996807713421186
0.606060606060606 0.949656115658539
0.616161616161616 0.928069582601307
0.626262626262626 0.908876405874116
0.636363636363636 0.891431277456711
0.646464646464647 0.875395959759012
0.656565656565657 0.860554408840524
0.666666666666667 0.846753105844767
0.676767676767677 0.833875626694497
0.686868686868687 0.821829807946202
0.696969696969697 0.810540475811387
0.707070707070707 0.799944964739826
0.717171717171717 0.789990171647023
0.727272727272727 0.780630521659232
0.737373737373737 0.771826505607205
0.747474747474748 0.763543594629535
0.757575757575758 0.755751413072297
0.767676767676768 0.748423094234294
0.777777777777778 0.741534769645821
0.787878787878788 0.735065157140145
0.797979797979798 0.728995225007391
0.808080808080808 0.723307914522897
0.818181818181818 0.717987908586063
0.828282828282828 0.713021437111985
0.838383838383838 0.708396112227229
0.848484848484849 0.704100787681811
0.858585858585859 0.700125438606015
0.868686868686869 0.696461057041276
0.878787878787879 0.693099564134277
0.888888888888889 0.69003372792187
0.898989898989899 0.687257103177892
0.909090909090909 0.684763967526258
0.919191919191919 0.682549271355654
0.929292929292929 0.680608596069339
0.939393939393939 0.678938114400845
0.94949494949495 0.67753456230935
0.95959595959596 0.676395206639244
0.96969696969697 0.675517827957613
0.97979797979798 0.674900700007829
0.98989898989899 0.674542577473149
1 0.674442686401171
};
\addlegendentry{$\eta$}
\end{groupplot}

\end{tikzpicture}

%% file: figs/second_deflation_step_singular_value_alignments.tex
\begin{tikzpicture}

\definecolor{crimson2143940}{RGB}{214,39,40}
\definecolor{darkorange25512714}{RGB}{255,127,14}
\definecolor{darkslategray38}{RGB}{38,38,38}
\definecolor{darkviolet1910191}{RGB}{191,0,191}
\definecolor{forestgreen4416044}{RGB}{44,160,44}
\definecolor{lightgray204}{RGB}{204,204,204}
\definecolor{steelblue31119180}{RGB}{31,119,180}

\begin{groupplot}[group style={group size=2 by 2}]
\nextgroupplot[
axis line style={lightgray204},
legend cell align={left},
legend style={
  fill opacity=0.8,
  draw opacity=1,
  text opacity=1,
  at={(0.97,0.5)},
  anchor=east,
  draw=lightgray204
},
width=.5\textwidth,
height=.3\textwidth,
tick align=outside,
x grid style={lightgray204},
xlabel=\textcolor{darkslategray38}{\(\displaystyle \beta_1\)},
xmajorgrids,
xmajorticks=true,
xmin=-0.75, xmax=15.75,
xtick style={color=darkslategray38},
y grid style={lightgray204},
ymajorgrids,
ymajorticks=true,
ymin=1.42126649142161, ymax=4.25340368014614,
ytick style={color=darkslategray38}
]
\addplot [semithick, steelblue31119180, mark=*, mark size=2, mark options={solid}, only marks, opacity=0.4]
table {%
0 1.57920612359315
0.306122448979592 1.60278070288744
0.612244897959184 1.57381346612636
0.918367346938776 1.57627317435267
1.22448979591837 1.58635779233998
1.53061224489796 1.56309201378582
1.83673469387755 1.5889409803919
2.14285714285714 1.96092693381698
2.44897959183673 2.09326141087103
2.75510204081633 2.19980907932339
3.06122448979592 2.23876483920559
3.36734693877551 2.43908551266825
3.6734693877551 2.42998642235708
3.97959183673469 2.58025732807777
4.28571428571429 2.48586342467868
4.59183673469388 2.60056214847291
4.89795918367347 2.49430345946486
5.20408163265306 2.60656603671157
5.51020408163265 2.71662818707595
5.81632653061224 2.92400618569789
6.12244897959184 3.03865894691129
6.42857142857143 3.16958102288105
6.73469387755102 3.23767717223037
7.04081632653061 3.35645249061651
7.3469387755102 3.42836420072646
7.6530612244898 3.50035030148288
7.95918367346939 3.55890290638037
8.26530612244898 3.56219280227882
8.57142857142857 3.75084573274325
8.87755102040816 3.68292272683292
9.18367346938776 3.70745081188309
9.48979591836735 3.67108859349658
9.79591836734694 3.74106336751917
10.1020408163265 3.74120423823751
10.4081632653061 3.76860730298164
10.7142857142857 3.67909314680915
11.0204081632653 3.81423221482993
11.3265306122449 3.99497825973641
11.6326530612245 3.91197810936142
11.9387755102041 3.95836572233666
12.2448979591837 4.03184567602635
12.5510204081633 4.03176621252515
12.8571428571429 3.96762874267082
13.1632653061224 4.02645896136173
13.469387755102 4.01281753113342
13.7755102040816 3.85806338583652
14.0816326530612 4.0549668590244
14.3877551020408 4.06346258039003
14.6938775510204 4.12467017156776
15 3.99114162132785
};
\addlegendentry{$\hat\lambda_2$}
\addplot [thick, darkorange25512714, line width=1.3pt]
table {%
0 1.55
0.306122448979592 1.55
0.612244897959184 1.55
0.918367346938776 1.55
1.22448979591837 1.55
1.53061224489796 1.55
1.83673469387755 1.55
2.14285714285714 1.94876968557782
2.44897959183673 2.09952717043858
2.75510204081633 2.24207798700421
3.06122448979592 2.3698671450842
3.36734693877551 2.4769205762515
3.6734693877551 2.55756366292142
3.97959183673469 2.60693804798093
4.28571428571429 2.62239587276002
4.59183673469388 2.60547464033094
4.89795918367347 2.48287915400383
5.20408163265306 2.66652818341195
5.51020408163265 2.83628305644311
5.81632653061224 2.98693771277033
6.12244897959184 3.11747276017082
6.42857142857143 3.22936519470683
6.73469387755102 3.32508641422267
7.04081632653061 3.40723167648277
7.3469387755102 3.47814152734604
7.6530612244898 3.53979093112323
7.95918367346939 3.59379628220386
8.26530612244898 3.64146369055386
8.57142857142857 3.68384493970131
8.87755102040816 3.72178844312858
9.18367346938776 3.75598164605305
9.48979591836735 3.78698490668169
9.79591836734694 3.81525806973216
10.1020408163265 3.84118115279169
10.4081632653061 3.8650704280571
10.7142857142857 3.88719095261553
11.0204081632653 3.90776637394466
11.3265306122449 3.92698664183431
11.6326530612245 3.94501411516803
11.9387755102041 3.96198842552857
12.2448979591837 3.97803037778582
12.5510204081633 3.99324509974539
12.8571428571429 4.00772460040876
13.1632653061224 4.02154986479236
13.469387755102 4.03479257800368
13.7755102040816 4.04751655569669
14.0816326530612 4.05977893796822
14.3877551020408 4.07163119508852
14.6938775510204 4.08311997835158
15 4.09428785081963
};
\addlegendentry{$\lambda_2$}
\addplot [semithick, black, dashed, line width=1pt]
table {%
5 1
5 4
};
\addlegendentry{$\beta_1=\beta_2$}

\nextgroupplot[
axis line style={lightgray204},
legend cell align={left},
legend columns=1,
legend style={
  fill opacity=0.8,
  draw opacity=1,
  text opacity=1,
  at={(0.97,0.5)},
  anchor=east,
  draw=lightgray204
},
width=.5\textwidth,
height=.3\textwidth,
tick align=outside,
x grid style={lightgray204},
xlabel=\textcolor{darkslategray38}{\(\displaystyle \beta_1\)},
xmajorgrids,
xmajorticks=true,
xmin=-0.75, xmax=15.75,
xtick style={color=darkslategray38},
y grid style={lightgray204},
ymajorgrids,
ymajorticks=true,
ymin=-0.05, ymax=1.05,
ytick style={color=darkslategray38}
]
\addplot [semithick, steelblue31119180, mark=*, mark size=2, mark options={solid}, only marks, opacity=0.4]
table {%
0 0.0103167312862061
0.306122448979592 0.0874769505397888
0.612244897959184 0.15670182103481
0.918367346938776 0.0889825808506161
1.22448979591837 0.0657658142835187
1.53061224489796 0.0170826841758828
1.83673469387755 0.14900819635198
2.14285714285714 0.838291416973586
2.44897959183673 0.855583684928476
2.75510204081633 0.843968078844029
3.06122448979592 0.837621833409223
3.36734693877551 0.821139424192641
3.6734693877551 0.813630855767426
3.97959183673469 0.803868764490665
4.28571428571429 0.765984449295882
4.59183673469388 0.769256878636676
4.89795918367347 0.165309318515771
5.20408163265306 0.142135620344791
5.51020408163265 0.0908714662380193
5.81632653061224 0.0994272650303078
6.12244897959184 0.0807932105027309
6.42857142857143 0.0453971147782687
6.73469387755102 0.017430722130811
7.04081632653061 0.0181595947916056
7.3469387755102 0.00258917207382595
7.6530612244898 0.0146838559431839
7.95918367346939 0.0127507310446786
8.26530612244898 0.0410497599658324
8.57142857142857 0.0412023192844761
8.87755102040816 0.0595873965202748
9.18367346938776 0.0740453766563981
9.48979591836735 0.080722410701378
9.79591836734694 0.0944672797029843
10.1020408163265 0.0944046860988101
10.4081632653061 0.102757282873257
10.7142857142857 0.119364168415696
11.0204081632653 0.129272490277664
11.3265306122449 0.121685033265455
11.6326530612245 0.126385147130183
11.9387755102041 0.130693264818345
12.2448979591837 0.146058079011683
12.5510204081633 0.14326940182912
12.8571428571429 0.168153747635011
13.1632653061224 0.155144955415828
13.469387755102 0.176266562262348
13.7755102040816 0.20380780832963
14.0816326530612 0.192221556894641
14.3877551020408 0.183874164736717
14.6938775510204 0.184359798596004
15 0.227111562065197
};
\addlegendentry{$\hat\theta_{21}$}
\addplot [semithick, darkorange25512714, mark=*, mark size=2, mark options={solid}, only marks, opacity=0.4]
table {%
0 0.0268410673416415
0.306122448979592 0.0120678411719301
0.612244897959184 0.0022206761936192
0.918367346938776 0.00119409707718287
1.22448979591837 0.0201446982948982
1.53061224489796 0.0156752848123952
1.83673469387755 0.000954723362567394
2.14285714285714 0.0699615774003186
2.44897959183673 0.0652058917101836
2.75510204081633 0.0459989197735076
3.06122448979592 0.00206386524609501
3.36734693877551 0.0146619344588599
3.6734693877551 0.0524933381287763
3.97959183673469 0.0676131198469428
4.28571428571429 0.118166771733867
4.59183673469388 0.131946567156963
4.89795918367347 0.737166834053658
5.20408163265306 0.760274280150087
5.51020408163265 0.796753195295341
5.81632653061224 0.791221281117889
6.12244897959184 0.808138364218652
6.42857142857143 0.824665330302241
6.73469387755102 0.840780936061189
7.04081632653061 0.838116286295287
7.3469387755102 0.840622170985242
7.6530612244898 0.846510860079836
7.95918367346939 0.859975212755772
8.26530612244898 0.866939127422838
8.57142857142857 0.872470656504758
8.87755102040816 0.884609974418473
9.18367346938776 0.889252130572216
9.48979591836735 0.889421086288324
9.79591836734694 0.90056188913263
10.1020408163265 0.896689911632966
10.4081632653061 0.901950659006877
10.7142857142857 0.879958512696902
11.0204081632653 0.905293201649589
11.3265306122449 0.910817762731626
11.6326530612245 0.912642247250494
11.9387755102041 0.913466100511998
12.2448979591837 0.921387950749783
12.5510204081633 0.914859334725673
12.8571428571429 0.913866646594425
13.1632653061224 0.921506331278307
13.469387755102 0.930430557499516
13.7755102040816 0.914283125174857
14.0816326530612 0.93809340591195
14.3877551020408 0.930653233755377
14.6938775510204 0.936032663308006
15 0.943498225104859
};
\addlegendentry{$\hat\theta_{22}$}
\addplot [thick, forestgreen4416044, line width=1.3pt]
table {%
0 0
0.306122448979592 0
0.612244897959184 0
0.918367346938776 0
1.22448979591837 0
1.53061224489796 0
1.83673469387755 0
2.14285714285714 0.839152467235848
2.44897959183673 0.845727004167832
2.75510204081633 0.844699767760204
3.06122448979592 0.838827222231154
3.36734693877551 0.829318249001953
3.6734693877551 0.81682015744061
3.97959183673469 0.801892308830749
4.28571428571429 0.78534546621391
4.59183673469388 0.768470745571345
4.89795918367347 0.161304599983208
5.20408163265306 0.139849523418814
5.51020408163265 0.116640931302356
5.81632653061224 0.093497096725465
6.12244897959184 0.0715220060878958
6.42857142857143 0.0512077759823825
6.73469387755102 0.0326627749881543
7.04081632653061 0.0157998070904429
7.3469387755102 0.000451288012903237
7.6530612244898 0.0135706676333999
7.95918367346939 0.0264465488312558
8.26530612244898 0.0383386185312703
8.57142857142857 0.0493881850368537
8.87755102040816 0.0597164472637052
9.18367346938776 0.0694267011950725
9.48979591836735 0.078606880588373
9.79591836734694 0.0873319864939583
10.1020408163265 0.09566624035299
10.4081632653061 0.103664921480449
10.7142857142857 0.111375906883095
11.0204081632653 0.118840949891371
11.3265306122449 0.126096737591601
11.6326530612245 0.133175767858977
11.9387755102041 0.140107078137743
12.2448979591837 0.146916854961268
12.5510204081633 0.153628947671418
12.8571428571429 0.160265307168672
13.1632653061224 0.166846363173846
13.469387755102 0.173391355758536
13.7755102040816 0.179918631699045
14.0816326530612 0.186445914107874
14.3877551020408 0.192990556459935
14.6938775510204 0.19956978642185
15 0.206200949546033
};
\addlegendentry{$\theta_{21}$}
\addplot [thick, crimson2143940, line width=1.3pt]
table {%
0 0
0.306122448979592 0
0.612244897959184 0
0.918367346938776 0
1.22448979591837 0
1.53061224489796 0
1.83673469387755 0
2.14285714285714 0.0881083697881104
2.44897959183673 0.0674882715181377
2.75510204081633 0.0440502716688583
3.06122448979592 0.0185032780844774
3.36734693877551 0.00889304457757152
3.6734693877551 0.037854479802757
3.97959183673469 0.0677669610480007
4.28571428571429 0.0974301037843195
4.59183673469388 0.124962208720857
4.89795918367347 0.742082024138298
5.20408163265306 0.761028704144324
5.51020408163265 0.779024870381728
5.81632653061224 0.795414867530006
6.12244897959184 0.809939044750505
6.42857142857143 0.822636473899565
6.73469387755102 0.833695497968176
7.04081632653061 0.843349675110177
7.3469387755102 0.851825152884586
7.6530612244898 0.859320355573024
7.95918367346939 0.866001650798246
8.26530612244898 0.872005637569897
8.57142857142857 0.877443544648154
8.87755102040816 0.88240581473423
9.18367346938776 0.886966169120577
9.48979591836735 0.891184977213293
9.79591836734694 0.89511196280459
10.1020408163265 0.89878834522605
10.4081632653061 0.902248523807534
10.7142857142857 0.905521401535241
11.0204081632653 0.908631427983513
11.3265306122449 0.911599423286164
11.6326530612245 0.914443233941406
11.9387755102041 0.91717825604344
12.2448979591837 0.919817856384841
12.5510204081633 0.922373712997166
12.8571428571429 0.924856092478392
13.1632653061224 0.927274077000429
13.469387755102 0.929635751569722
13.7755102040816 0.931948359490092
14.0816326530612 0.934218431999195
14.3877551020408 0.936451897455854
14.6938775510204 0.938654173630652
15 0.94083024649292
};
\addlegendentry{$\theta_{22}$}
\addplot [semithick, darkviolet1910191, dashed, line width=1pt]
table {%
0 0.5
15 0.5
};
\addlegendentry{$\alpha = 0.5$}
\addplot [semithick, black, dashed, line width=1pt]
table {%
5 0
5 1
};

\nextgroupplot[
axis line style={lightgray204},
legend cell align={left},
legend columns=1,
legend style={
  fill opacity=0.8,
  draw opacity=1,
  text opacity=1,
  at={(0.97,0.03)},
  anchor=south east,
  draw=lightgray204
},
width=.5\textwidth,
height=.3\textwidth,
tick align=outside,
x grid style={lightgray204},
xlabel=\textcolor{darkslategray38}{\(\displaystyle \beta_1\)},
xmajorgrids,
xmajorticks=true,
xmin=-0.75, xmax=15.75,
xtick style={color=darkslategray38},
y grid style={lightgray204},
ymajorgrids,
ymajorticks=true,
ymin=-0.05, ymax=1.05,
ytick style={color=darkslategray38}
]
\addplot [semithick, steelblue31119180, mark=*, mark size=2, mark options={solid}, only marks, opacity=0.4]
table {%
0 0.140783710367106
0.306122448979592 0.145081590351213
0.612244897959184 0.029068271949647
0.918367346938776 0.00556379403377348
1.22448979591837 0.22263070054907
1.53061224489796 0.0301542948237143
1.83673469387755 0.0949635824126043
2.14285714285714 0.936651884007587
2.44897959183673 0.936993338375425
2.75510204081633 0.948080554760983
3.06122448979592 0.948858700328714
3.36734693877551 0.943253962750541
3.6734693877551 0.947632354387087
3.97959183673469 0.953252152008219
4.28571428571429 0.933967569780605
4.59183673469388 0.952681269833881
4.89795918367347 0.288849214824292
5.20408163265306 0.362035668081587
5.51020408163265 0.28937600141759
5.81632653061224 0.30937531900349
6.12244897959184 0.383892670493865
6.42857142857143 0.32753963407066
6.73469387755102 0.330657129325707
7.04081632653061 0.40871065164002
7.3469387755102 0.380861125312582
7.6530612244898 0.379395097275903
7.95918367346939 0.399016988525581
8.26530612244898 0.386043865231228
8.57142857142857 0.371219475982404
8.87755102040816 0.400019260319446
9.18367346938776 0.396908555448555
9.48979591836735 0.399086416977153
9.79591836734694 0.419411453956597
10.1020408163265 0.384842368081549
10.4081632653061 0.413545976110972
10.7142857142857 0.419720613172588
11.0204081632653 0.422062300851295
11.3265306122449 0.410072300362456
11.6326530612245 0.423285768194308
11.9387755102041 0.421644110671784
12.2448979591837 0.407472557322519
12.5510204081633 0.426134655986834
12.8571428571429 0.421669803221406
13.1632653061224 0.432606329171537
13.469387755102 0.447626671962234
13.7755102040816 0.459712614103297
14.0816326530612 0.447783453397836
14.3877551020408 0.4621337614866
14.6938775510204 0.449391696706319
15 0.477468741874832
};
\addlegendentry{$\hat\rho_{21}$}
\addplot [semithick, darkorange25512714, mark=*, mark size=2, mark options={solid}, only marks, opacity=0.4]
table {%
0 0.163410524172768
0.306122448979592 0.231304097479898
0.612244897959184 0.0534344863918842
0.918367346938776 0.0233352797516536
1.22448979591837 0.196039993386644
1.53061224489796 0.0538046876002835
1.83673469387755 0.241917643972548
2.14285714285714 0.36457138341625
2.44897959183673 0.387583899235935
2.75510204081633 0.372570052221875
3.06122448979592 0.363231701333896
3.36734693877551 0.286539456976385
3.6734693877551 0.369891587751162
3.97959183673469 0.305879058978368
4.28571428571429 0.2409867387511
4.59183673469388 0.28779722211938
4.89795918367347 0.931822366832286
5.20408163265306 0.960805510422496
5.51020408163265 0.936006006901193
5.81632653061224 0.958320314527604
6.12244897959184 0.975611135391354
6.42857142857143 0.967342406094196
6.73469387755102 0.943761908488532
7.04081632653061 0.977794378368198
7.3469387755102 0.977465006905099
7.6530612244898 0.976762249063463
7.95918367346939 0.976171702243701
8.26530612244898 0.976887379702971
8.57142857142857 0.978199512503001
8.87755102040816 0.980273409373859
9.18367346938776 0.979912539755554
9.48979591836735 0.977270420733381
9.79591836734694 0.982266496196948
10.1020408163265 0.980686407821993
10.4081632653061 0.983856878714215
10.7142857142857 0.978885906358335
11.0204081632653 0.983365625654859
11.3265306122449 0.984930052524373
11.6326530612245 0.981150725918023
11.9387755102041 0.984503462696324
12.2448979591837 0.9787872870276
12.5510204081633 0.981866453799208
12.8571428571429 0.982117569839532
13.1632653061224 0.985667619996681
13.469387755102 0.987243750678696
13.7755102040816 0.983114722756368
14.0816326530612 0.980153851593087
14.3877551020408 0.98402326850303
14.6938775510204 0.990415058268536
15 0.984575140359754
};
\addlegendentry{$\hat\rho_{22}$}
\addplot [thick, forestgreen4416044, line width=1.3pt]
table {%
0 0
0.306122448979592 0
0.612244897959184 0
0.918367346938776 0
1.22448979591837 0
1.53061224489796 0
1.83673469387755 0
2.14285714285714 0.922831157471991
2.44897959183673 0.938648134183858
2.75510204081633 0.947830923981152
3.06122448979592 0.953180815939425
3.36734693877551 0.955976853599139
3.6734693877551 0.956854068570503
3.97959183673469 0.956148371470089
4.28571428571429 0.954083795063263
4.59183673469388 0.950912109018808
4.89795918367347 0.281082122226045
5.20408163265306 0.296091490553457
5.51020408163265 0.310191286751166
5.81632653061224 0.322927943096481
6.12244897959184 0.334180757961251
6.42857142857143 0.34403593771064
6.73469387755102 0.3526682460131
7.04081632653061 0.360270008465522
7.3469387755102 0.367018794085827
7.6530612244898 0.373066789188226
7.95918367346939 0.378540074159644
8.26530612244898 0.383541614170553
8.57142857142857 0.388155163595348
8.87755102040816 0.392448972665356
9.18367346938776 0.396478942816822
9.48979591836735 0.40029118689528
9.79591836734694 0.403924060652778
10.1020408163265 0.40740976056057
10.4081632653061 0.410775576579858
10.7142857142857 0.414044878166317
11.0204081632653 0.417237893669857
11.3265306122449 0.420372329737359
11.6326530612245 0.423463869803378
11.9387755102041 0.426526577752174
12.2448979591837 0.429573229257058
12.5510204081633 0.432615587195388
12.8571428571429 0.435664635470973
13.1632653061224 0.438730779249488
13.469387755102 0.441824022180857
13.7755102040816 0.444954126720494
14.0816326530612 0.448130762464608
14.3877551020408 0.451363649949768
14.6938775510204 0.454662702367112
15 0.458038172121709
};
\addlegendentry{$\rho_{21}$}
\addplot [thick, crimson2143940, line width=1.3pt]
table {%
0 0
0.306122448979592 0
0.612244897959184 0
0.918367346938776 0
1.22448979591837 0
1.53061224489796 0
1.83673469387755 0
2.14285714285714 0.390983138036187
2.44897959183673 0.386608153237596
2.75510204081633 0.379529881671279
3.06122448979592 0.370623708358385
3.36734693877551 0.360145474676025
3.6734693877551 0.34816899040074
3.97959183673469 0.334792963779722
4.28571428571429 0.320305589511515
4.59183673469388 0.305313162793724
4.89795918367347 0.942759830158145
5.20408163265306 0.950753907364434
5.51020408163265 0.957075313652419
5.81632653061224 0.962036968733825
6.12244897959184 0.965938958817721
6.42857142857143 0.969036545895811
6.73469387755102 0.971529581272453
7.04081632653061 0.973567595637656
7.3469387755102 0.975260316281996
7.6530612244898 0.976687961259873
7.95918367346939 0.977909472165583
8.26530612244898 0.97896857680284
8.57142857142857 0.979898099854646
8.87755102040816 0.980722990860908
9.18367346938776 0.981462450015139
9.48979591836735 0.982131429397788
9.79591836734694 0.982741702841882
10.1020408163265 0.983302636449513
10.4081632653061 0.983821749680247
10.7142857142857 0.984305128631134
11.0204081632653 0.984757733798794
11.3265306122449 0.985183631003846
11.6326530612245 0.985586167132581
11.9387755102041 0.98596810406343
12.2448979591837 0.986331721787323
12.5510204081633 0.986678897954152
12.8571428571429 0.98701116942421
13.1632653061224 0.987329779448003
13.469387755102 0.987635713653002
13.7755102040816 0.987929726794015
14.0816326530612 0.988212361568177
14.3877551020408 0.988483960875629
14.6938775510204 0.988744673660246
15 0.988994454906345
};
\addlegendentry{$\rho_{22}$}
\addplot [semithick, darkviolet1910191, dashed, line width=1pt]
table {%
0 0.5
15 0.5
};
\addlegendentry{$\alpha = 0.5$}
\addplot [semithick, black, dashed, line width=1pt]
table {%
5 0
5 1
};

\nextgroupplot[
axis line style={lightgray204},
legend cell align={left},
legend columns=1,
legend style={fill opacity=0.8, draw opacity=1, text opacity=1, draw=lightgray204},
tick align=outside,
x grid style={lightgray204},
xlabel=\textcolor{darkslategray38}{\(\displaystyle \beta_1\)},
width=.5\textwidth,
height=.3\textwidth,
xmajorgrids,
xmajorticks=true,
xmin=-0.75, xmax=15.75,
xtick style={color=darkslategray38},
y grid style={lightgray204},
ymajorgrids,
ymajorticks=true,
ymin=-0.05, ymax=1.05,
ytick style={color=darkslategray38}
]
\addplot [semithick, steelblue31119180, mark=*, mark size=2, mark options={solid}, only marks, opacity=0.4]
table {%
0 0.00498081811594366
0.306122448979592 0.00180730028230461
0.612244897959184 0.0107168690179619
0.918367346938776 0.00526603018903755
1.22448979591837 0.0114189056215614
1.53061224489796 0.0100305635128522
1.83673469387755 0.0035262090079587
2.14285714285714 0.190368946764032
2.44897959183673 0.193100285429923
2.75510204081633 0.19700016192477
3.06122448979592 0.236746131048114
3.36734693877551 0.224563917065517
3.6734693877551 0.265324253309102
3.97959183673469 0.257281261051193
4.28571428571429 0.282503530315596
4.59183673469388 0.291424057370953
4.89795918367347 0.328644244090526
5.20408163265306 0.310545690969397
5.51020408163265 0.30869141025984
5.81632653061224 0.290612784267747
6.12244897959184 0.285305856263646
6.42857142857143 0.26311273255872
6.73469387755102 0.251475037332739
7.04081632653061 0.251457073768017
7.3469387755102 0.247351391298147
7.6530612244898 0.247171324512872
7.95918367346939 0.243039514463699
8.26530612244898 0.24352752292616
8.57142857142857 0.229768363723582
8.87755102040816 0.25126019584269
9.18367346938776 0.248114364696177
9.48979591836735 0.242633747084838
9.79591836734694 0.247355721140453
10.1020408163265 0.24492439417586
10.4081632653061 0.25210457608725
10.7142857142857 0.258840676725432
11.0204081632653 0.253868764581172
11.3265306122449 0.245987992357968
11.6326530612245 0.252933891764525
11.9387755102041 0.248841452086474
12.2448979591837 0.258901136983119
12.5510204081633 0.253893130928463
12.8571428571429 0.269610913221798
13.1632653061224 0.261901390857331
13.469387755102 0.27645610416125
13.7755102040816 0.29955366049557
14.0816326530612 0.285216481963021
14.3877551020408 0.277072534517164
14.6938775510204 0.280745911535422
15 0.31757739824206
};
\addlegendentry{$\hat\kappa$}
\addplot [semithick, darkorange25512714, mark=*, mark size=2, mark options={solid}, only marks, opacity=0.4]
table {%
0 0.156918305448784
0.306122448979592 0.231427649541786
0.612244897959184 0.0242011040149169
0.918367346938776 0.0134530066177799
1.22448979591837 0.213926091720086
1.53061224489796 0.0489966571483399
1.83673469387755 0.245502707350675
2.14285714285714 0.467100383494081
2.44897959183673 0.513453802364633
2.75510204081633 0.522419242789339
3.06122448979592 0.55558055428351
3.36734693877551 0.515297955471398
3.6734693877551 0.634282267807388
3.97959183673469 0.585371074609459
4.28571428571429 0.600559673513676
4.59183673469388 0.661811769547053
4.89795918367347 0.705086621886091
5.20408163265306 0.732281943972524
5.51020408163265 0.651026208852564
5.81632653061224 0.647796403247457
6.12244897959184 0.683975537748124
6.42857142857143 0.603488507202142
6.73469387755102 0.572847940777434
7.04081632653061 0.634988778864692
7.3469387755102 0.59955075261902
7.6530612244898 0.58181188784075
7.95918367346939 0.587414495826509
8.26530612244898 0.566966509706476
8.57142857142857 0.538230949664784
8.87755102040816 0.571267905067651
9.18367346938776 0.557284436322654
9.48979591836735 0.53746570613651
9.79591836734694 0.545164854513707
10.1020408163265 0.523662518326069
10.4081632653061 0.538975901795477
10.7142857142857 0.538085254693946
11.0204081632653 0.539978961120451
11.3265306122449 0.524384081388389
11.6326530612245 0.529407861065219
11.9387755102041 0.522404050278488
12.2448979591837 0.511696935079674
12.5510204081633 0.525618788270986
12.8571428571429 0.527846604937282
13.1632653061224 0.525445257185965
13.469387755102 0.540276327339509
13.7755102040816 0.542016911450933
14.0816326530612 0.537478304009083
14.3877551020408 0.54007401759562
14.6938775510204 0.526239981075104
15 0.560184063310266
};
\addlegendentry{$\hat\eta$}
\addplot [thick, forestgreen4416044, line width=1.3pt]
table {%
0 0
0.306122448979592 0
0.612244897959184 0
0.918367346938776 0
1.22448979591837 0
1.53061224489796 0
1.83673469387755 0
2.14285714285714 0.200901229971765
2.44897959183673 0.20735495180538
2.75510204081633 0.213807388114038
3.06122448979592 0.221559995046615
3.36734693877551 0.231383155593623
3.6734693877551 0.243914590758891
3.97959183673469 0.259762138429509
4.28571428571429 0.279447660264038
4.59183673469388 0.303220494086503
4.89795918367347 0.349226653698488
5.20408163265306 0.325170299949289
5.51020408163265 0.305311646298235
5.81632653061224 0.289525575534012
6.12244897959184 0.277248962942191
6.42857142857143 0.267818195185098
6.73469387755102 0.260634858332267
7.04081632653061 0.255214686321996
7.3469387755102 0.251183742271584
7.6530612244898 0.248258378331252
7.95918367346939 0.246224115000155
8.26530612244898 0.244918094170782
8.57142857142857 0.244215704567212
8.87755102040816 0.24402073806406
9.18367346938776 0.244258234678966
9.48979591836735 0.244869288709218
9.79591836734694 0.245807262094368
10.1020408163265 0.247035003344643
10.4081632653061 0.248522787687839
10.7142857142857 0.250246777876105
11.0204081632653 0.252187864545742
11.3265306122449 0.254330784939235
11.6326530612245 0.25666345082108
11.9387755102041 0.259176431940106
12.2448979591837 0.261862559694804
12.5510204081633 0.264716623031617
12.8571428571429 0.267735138539475
13.1632653061224 0.270916177923917
13.469387755102 0.274259244866153
13.7755102040816 0.277765193196415
14.0816326530612 0.281436180428304
14.3877551020408 0.285275655680057
14.6938775510204 0.28928837878935
15 0.293480472358433
};
\addlegendentry{$\kappa$}
\addplot [thick, crimson2143940, line width=1.3pt]
table {%
0 0
0.306122448979592 0
0.612244897959184 0
0.918367346938776 0
1.22448979591837 0
1.53061224489796 0
1.83673469387755 0
2.14285714285714 0.491772672642233
2.44897959183673 0.511294896280442
2.75510204081633 0.530379138732586
3.06122448979592 0.550359314375845
3.36734693877551 0.571858546920691
3.6734693877551 0.595158198010742
3.97959183673469 0.620234962363103
4.28571428571429 0.646678600871944
4.59183673469388 0.673627460764899
4.89795918367347 0.715312411149022
5.20408163265306 0.694149829989042
5.51020408163265 0.673504168062806
5.81632653061224 0.654421419981506
6.12244897959184 0.637369172801855
6.42857142857143 0.622416282670301
6.73469387755102 0.609423818801227
7.04081632653061 0.598173146248602
7.3469387755102 0.588433569351852
7.6530612244898 0.579991997089771
7.95918367346939 0.572662895797738
8.26530612244898 0.566289234286734
8.57142857142857 0.560739852013055
8.87755102040816 0.555905770850269
9.18367346938776 0.551696524655998
9.48979591836735 0.548036905620256
9.79591836734694 0.544864224623354
10.1020408163265 0.54212606497062
10.4081632653061 0.539778462361951
10.7142857142857 0.537784439800783
11.0204081632653 0.536112830549266
11.3265306122449 0.534737331005077
11.6326530612245 0.533635739913821
11.9387755102041 0.532789344682624
12.2448979591837 0.532182426689667
12.5510204081633 0.531801862291284
12.8571428571429 0.531636803113073
13.1632653061224 0.531678419407977
13.469387755102 0.531919698122011
13.7755102040816 0.532355286907904
14.0816326530612 0.532981376880247
14.3877551020408 0.53379562265261
14.6938775510204 0.534797093911512
15 0.535986259586597
};
\addlegendentry{$\eta$}
\addplot [semithick, darkviolet1910191, dashed, line width=1pt]
table {%
0 0.5
15 0.5
};
\addlegendentry{$\alpha = 0.5$}
\addplot [semithick, black, dashed, line width=1pt]
table {%
5 0
5 1
};
\end{groupplot}

\end{tikzpicture}

%% file: figs/optimization_of_gamma_alpha_0.4.tex
\begin{tikzpicture}

\definecolor{darkorange25512714}{RGB}{255,127,14}
\definecolor{darkslategray38}{RGB}{38,38,38}
\definecolor{darkviolet1910191}{RGB}{191,0,191}
\definecolor{forestgreen4416044}{RGB}{44,160,44}
\definecolor{lightgray204}{RGB}{204,204,204}
\definecolor{steelblue31119180}{RGB}{31,119,180}

\begin{groupplot}[group style={group size=3 by 1}]
\nextgroupplot[
axis line style={lightgray204},
legend cell align={left},
legend style={
  fill opacity=0.8,
  draw opacity=1,
  text opacity=1,
  at={(0.03,0.03)},
  anchor=south west,
  draw=lightgray204
},
width=0.37\textwidth,
height=0.3\textwidth,
tick align=outside,
title={\(\displaystyle \alpha=0.4\)},
x grid style={lightgray204},
xlabel=\textcolor{darkslategray38}{\(\displaystyle \gamma\)},
xmajorgrids,
xmajorticks=true,
xmin=-0.05, xmax=1.05,
xtick style={color=darkslategray38},
y grid style={lightgray204},
ymajorgrids,
ymajorticks=true,
ymin=-0.05, ymax=1.05,
ytick style={color=darkslategray38}
]
\addplot [semithick, steelblue31119180, line width=1.3pt]
table {%
0 0.927082567974495
0.0101010101010101 0.927069744680428
0.0202020202020202 0.927056384547503
0.0303030303030303 0.927042461946022
0.0404040404040404 0.927027949796785
0.0505050505050505 0.927012819406613
0.0606060606060606 0.926997040422172
0.0707070707070707 0.92698058073359
0.0808080808080808 0.926963406245349
0.0909090909090909 0.92694548084154
0.101010101010101 0.926926766192878
0.111111111111111 0.926907221600857
0.121212121212121 0.926886803820691
0.131313131313131 0.926865466888652
0.141414141414141 0.926843161674331
0.151515151515152 0.926819836734712
0.161616161616162 0.926795435955029
0.171717171717172 0.9267699004474
0.181818181818182 0.926743167025831
0.191919191919192 0.926715168347242
0.202020202020202 0.926685832077996
0.212121212121212 0.926655081922674
0.222222222222222 0.926622834886456
0.232323232323232 0.926589002546933
0.242424242424242 0.663797615666647
0.252525252525253 0.62286187355094
0.262626262626263 0.590625844834321
0.272727272727273 0.562972569299768
0.282828282828283 0.538270597149206
0.292929292929293 0.515670550261258
0.303030303030303 0.494661672172982
0.313131313131313 0.474907605386055
0.323232323232323 0.456172445018573
0.333333333333333 0.438282895437857
0.343434343434343 0.421107052195207
0.353535353535354 0.404541656356231
0.363636363636364 0.388504015231005
0.373737373737374 0.37292665319486
0.383838383838384 0.357753640852741
0.393939393939394 0.342937999741829
0.404040404040404 0.328439819449227
0.414141414141414 0.314224861486402
0.424242424242424 0.300263503493427
0.434343434343434 0.286529928908
0.444444444444444 0.273001493246029
0.454545454545455 0.259658224051522
0.464646464646465 0.246482419592755
0.474747474747475 0.233458326698603
0.484848484848485 0.220571874421038
0.494949494949495 0.207810456942591
0.505050505050505 0.195162752390387
0.515151515151515 0.182618570463488
0.525252525252525 0.170168726114822
0.535353535353535 0.157804929463123
0.545454545454546 0.145519693422729
0.555555555555556 0.133306255384206
0.565656565656566 0.121158507009704
0.575757575757576 0.109070935183128
0.585858585858586 0.0970385699504596
0.595959595959596 0.0850569384535683
0.606060606060606 0.0731220252706488
0.616161616161616 0.0612302361207475
0.626262626262626 0.049378366328818
0.636363636363636 0.0375635724592797
0.646464646464647 0.0257833467284579
0.656565656565657 0.0140354938289365
0.666666666666667 0.00231810999176535
0.676767676767677 0.00937043593473108
0.686868686868687 0.0210315202355744
0.696969696969697 0.0326662803582469
0.707070707070707 0.0442756298976018
0.717171717171717 0.0558602725208074
0.727272727272727 0.0674207148652589
0.737373737373737 0.0789572788718005
0.747474747474748 0.0904701131924101
0.757575757575758 0.101959204145517
0.767676767676768 0.113424386073801
0.777777777777778 0.124865351143423
0.787878787878788 0.13628165885305
0.797979797979798 0.147672745672083
0.808080808080808 0.159037932884277
0.818181818181818 0.170376435662557
0.828282828282828 0.181687371124791
0.838383838383838 0.192969766435556
0.848484848484849 0.204222566196359
0.858585858585859 0.215444639995006
0.868686868686869 0.226634790641543
0.878787878787879 0.237791759859349
0.888888888888889 0.248914235953472
0.898989898989899 0.260000860739074
0.909090909090909 0.271050235587856
0.919191919191919 0.282060928913908
0.929292929292929 0.293031480104718
0.939393939393939 0.303960408361874
0.94949494949495 0.314846217251383
0.95959595959596 0.325687399342378
0.96969696969697 0.336482444882296
0.97979797979798 0.347229843968039
0.98989898989899 0.357928090405187
1 0.368575693657833
};
\addlegendentry{$\theta_{21}$}
\addplot [semithick, darkorange25512714, line width=1.3pt]
table {%
0 0.711151451577952
0.0101010101010101 0.711141615016662
0.0202020202020202 0.711131366661879
0.0303030303030303 0.711120686871665
0.0404040404040404 0.711109554774106
0.0505050505050505 0.711097948467081
0.0606060606060606 0.711085844666549
0.0707070707070707 0.71107321865907
0.0808080808080808 0.711060044357044
0.0909090909090909 0.711046294047221
0.101010101010101 0.711031938315372
0.111111111111111 0.711016945945152
0.121212121212121 0.711001283772872
0.131313131313131 0.710984916516494
0.141414141414141 0.710967807035485
0.151515151515152 0.710949914313297
0.161616161616162 0.710931196894348
0.171717171717172 0.710911608900179
0.181818181818182 0.710891102095379
0.191919191919192 0.710869624669484
0.202020202020202 0.710847121946333
0.212121212121212 0.710823533405363
0.222222222222222 0.710798796858626
0.232323232323232 0.710772844609047
0.242424242424242 0.948533747372201
0.252525252525253 0.963744555007923
0.262626262626263 0.97342056484726
0.272727272727273 0.980257216853906
0.282828282828283 0.985302861897694
0.292929292929293 0.98909277854316
0.303030303030303 0.991942315493039
0.313131313131313 0.994054713247087
0.323232323232323 0.995569323201998
0.333333333333333 0.996586113557976
0.343434343434343 0.997179312498263
0.353535353535354 0.997405544216475
0.363636363636364 0.997308948795588
0.373737373737374 0.996924546717472
0.383838383838384 0.9962805303514
0.393939393939394 0.995399871903796
0.404040404040404 0.99430148108652
0.414141414141414 0.993001057153731
0.424242424242424 0.991511728204464
0.434343434343434 0.989844538663645
0.444444444444444 0.988008827536664
0.454545454545455 0.986012524934072
0.464646464646465 0.983862387990997
0.474747474747475 0.981564190559246
0.484848484848485 0.979122877098655
0.494949494949495 0.976542688967756
0.505050505050505 0.973827268789384
0.515151515151515 0.970979747387755
0.525252525252525 0.968002817080173
0.535353535353535 0.964898793208326
0.545454545454546 0.961669667480663
0.555555555555556 0.95831715192127
0.565656565656566 0.95484271898367
0.575757575757576 0.951247634706383
0.585858585858586 0.94753298828164
0.595959595959596 0.943699717858034
0.606060606060606 0.939748632939798
0.616161616161616 0.935680434938455
0.626262626262626 0.93149573451767
0.636363636363636 0.927195067649934
0.646464646464647 0.922778909701541
0.656565656565657 0.918247688090664
0.666666666666667 0.913601793588186
0.676767676767677 0.908841590723747
0.686868686868687 0.903967426726106
0.696969696969697 0.898979639939106
0.707070707070707 0.893878567219976
0.717171717171717 0.888664550621
0.727272727272727 0.883337943361002
0.737373737373737 0.877899115216291
0.747474747474748 0.872348457238212
0.757575757575758 0.866686385963282
0.767676767676768 0.86091334728529
0.777777777777778 0.855029819340463
0.787878787878788 0.849036315423061
0.797979797979798 0.842933386576873
0.808080808080808 0.836721623050112
0.818181818181818 0.830401655989577
0.828282828282828 0.823974158719806
0.838383838383838 0.817439847193414
0.848484848484849 0.810799480546721
0.858585858585859 0.804053860947401
0.868686868686869 0.797203834128199
0.878787878787879 0.790250287616608
0.888888888888889 0.783194150498662
0.898989898989899 0.776036392274209
0.909090909090909 0.768778021112343
0.919191919191919 0.761420082098192
0.929292929292929 0.753963655763814
0.939393939393939 0.746409854482773
0.94949494949495 0.738759821684088
0.95959595959596 0.731014727506893
0.96969696969697 0.723175767281388
0.97979797979798 0.715244156592623
0.98989898989899 0.707221129126149
1 0.699107932694874
};
\addlegendentry{$\theta_{22}$}
\addplot [semithick, forestgreen4416044, dashed, line width=1pt]
table {%
0 0.997405544216475
1 0.997405544216475
};
\addlegendentry{$\theta_{22}^{*}=0.997$}
\addplot [semithick, black, dashed, line width=1pt]
table {%
0.4 0
0.4 1
};
\addlegendentry{$\alpha=0.4$}
\addplot [semithick, darkviolet1910191, dashed, line width=1pt]
table {%
0.353535353535354 0
0.353535353535354 1
};
\addlegendentry{$\gamma^*(\theta_{22})=0.35$}

\nextgroupplot[
axis line style={lightgray204},
legend cell align={left},
legend style={
  fill opacity=0.8,
  draw opacity=1,
  text opacity=1,
  at={(0.03,0.03)},
  anchor=south west,
  draw=lightgray204
},
width=0.37\textwidth,
height=0.3\textwidth,
tick align=outside,
title={\(\displaystyle \alpha=0.4\)},
x grid style={lightgray204},
xlabel=\textcolor{darkslategray38}{\(\displaystyle \gamma\)},
xmajorgrids,
xmajorticks=true,
xmin=-0.05, xmax=1.05,
xtick style={color=darkslategray38},
y grid style={lightgray204},
ymajorgrids,
ymajorticks=true,
ymin=-0.05, ymax=1.05,
ytick style={color=darkslategray38}
]
\addplot [thick, steelblue31119180, line width=1.3pt]
table {%
0 0.927079326399828
0.0101010101010101 0.927066632703917
0.0202020202020202 0.927053672762643
0.0303030303030303 0.92704043803645
0.0404040404040404 0.92702691964224
0.0505050505050505 0.927013108245548
0.0606060606060606 0.926998994124055
0.0707070707070707 0.926984567150185
0.0808080808080808 0.926969816670245
0.0909090909090909 0.92695473157208
0.101010101010101 0.92693930021762
0.111111111111111 0.926923510411496
0.121212121212121 0.926907349367875
0.131313131313131 0.926890803687479
0.141414141414141 0.926873859098
0.151515151515152 0.926856501399379
0.161616161616162 0.926838714477179
0.171717171717172 0.92682048225135
0.181818181818182 0.926801787465668
0.191919191919192 0.926782612075557
0.202020202020202 0.926762936734176
0.212121212121212 0.926742741990991
0.222222222222222 0.926722006153682
0.232323232323232 0.926700706720032
0.242424242424242 0.689586791284395
0.252525252525253 0.654455091711085
0.262626262626263 0.627457106817044
0.272727272727273 0.604834599852235
0.282828282828283 0.585090444993781
0.292929292929293 0.567441419757583
0.303030303030303 0.551414399956287
0.313131313131313 0.536696462878657
0.323232323232323 0.523067235342687
0.333333333333333 0.510364192191885
0.343434343434343 0.49846315553827
0.353535353535354 0.487266554023684
0.363636363636364 0.476695963476537
0.373737373737374 0.466687156823363
0.383838383838384 0.457186699241077
0.393939393939394 0.448149535205406
0.404040404040404 0.439537233823424
0.414141414141414 0.431316684540569
0.424242424242424 0.423459108286711
0.434343434343434 0.415939296927915
0.444444444444444 0.408735016551395
0.454545454545455 0.4018265357178
0.464646464646465 0.395196245765867
0.474747474747475 0.388828355095872
0.484848484848485 0.382708635985858
0.494949494949495 0.376824217509588
0.505050505050505 0.371163412307098
0.515151515151515 0.365715570676973
0.525252525252525 0.360470959033404
0.535353535353535 0.355420654167759
0.545454545454546 0.350556454151391
0.555555555555556 0.34587080268739
0.565656565656566 0.341356721260181
0.575757575757576 0.337007752192766
0.585858585858586 0.332817908380211
0.595959595959596 0.328781628651542
0.606060606060606 0.324893740179407
0.616161616161616 0.321149422983578
0.626262626262626 0.317544180621712
0.636363636363636 0.314073812962745
0.646464646464647 0.310734392352409
0.656565656565657 0.307522242183631
0.666666666666667 0.304433917752416
0.676767676767677 0.30146618921284
0.686868686868687 0.29861602603865
0.696969696969697 0.295880583064659
0.707070707070707 0.293257188052578
0.717171717171717 0.290743330143545
0.727272727272727 0.288336649728492
0.737373737373737 0.286034928895382
0.747474747474748 0.283836083178817
0.757575757575758 0.281738153648111
0.767676767676768 0.279739299935936
0.777777777777778 0.277837793790686
0.787878787878788 0.276032013762706
0.797979797979798 0.274320438325641
0.808080808080808 0.272701642894717
0.818181818181818 0.271174294406339
0.828282828282828 0.269737147322015
0.838383838383838 0.268389039873098
0.848484848484849 0.267128891378533
0.858585858585859 0.265955699231006
0.868686868686869 0.264868533290263
0.878787878787879 0.263866537848235
0.888888888888889 0.26294892682637
0.898989898989899 0.262114981394438
0.909090909090909 0.261364049317367
0.919191919191919 0.260695543042035
0.929292929292929 0.260108936882401
0.939393939393939 0.259603769577928
0.94949494949495 0.25917963730879
0.95959595959596 0.258836198959026
0.96969696969697 0.258573169265982
0.97979797979798 0.258390324568753
0.98989898989899 0.258287498128929
1 0.258264580290369
};
\addlegendentry{$\rho_{21}$}
\addplot [thick, darkorange25512714, line width=1.3pt]
table {%
0 0.711148965013342
0.0101010101010101 0.71113922786544
0.0202020202020202 0.71112928649154
0.0303030303030303 0.711119134346991
0.0404040404040404 0.711108764557285
0.0505050505050505 0.711098170032862
0.0606060606060606 0.711087343321695
0.0707070707070707 0.711076276579327
0.0808080808080808 0.711064961701487
0.0909090909090909 0.711053390146874
0.101010101010101 0.711041552982022
0.111111111111111 0.711029440853028
0.121212121212121 0.711017043959993
0.131313131313131 0.711004352003576
0.141414141414141 0.710991354520676
0.151515151515152 0.710978039238126
0.161616161616162 0.710964395188633
0.171717171717172 0.710950409460747
0.181818181818182 0.710936068976498
0.191919191919192 0.710921359767229
0.202020202020202 0.710906267727533
0.212121212121212 0.710890776158533
0.222222222222222 0.710874869704435
0.232323232323232 0.710858531260126
0.242424242424242 0.937216640273458
0.252525252525253 0.952441484249417
0.262626262626263 0.96232616735928
0.272727272727273 0.969502819676518
0.282828282828283 0.974997278799081
0.292929292929293 0.979334456940611
0.303030303030303 0.982824343422642
0.313131313131313 0.985667285513949
0.323232323232323 0.988001105562975
0.333333333333333 0.989925062879586
0.343434343434343 0.991513201949167
0.353535353535354 0.992822316635155
0.363636363636364 0.993896964138885
0.373737373737374 0.994772761028084
0.383838383838384 0.995478628569669
0.393939393939394 0.996038368469891
0.404040404040404 0.996471797268474
0.414141414141414 0.996795581017649
0.424242424242424 0.99702386123842
0.434343434343434 0.997168732023358
0.444444444444444 0.997240609605846
0.454545454545455 0.997248521874397
0.464646464646465 0.997200338240502
0.474747474747475 0.997102953946251
0.484848484848485 0.996962439341425
0.494949494949495 0.996784161806462
0.505050505050505 0.996572886224833
0.515151515151515 0.996332858070849
0.525252525252525 0.996067873451098
0.535353535353535 0.995781337133725
0.545454545454546 0.995476312406274
0.555555555555556 0.995155562640377
0.565656565656566 0.99482158767464
0.575757575757576 0.994476654391599
0.585858585858586 0.994122823338745
0.595959595959596 0.99376197171914
0.606060606060606 0.993395813283638
0.616161616161616 0.993025915985542
0.626262626262626 0.992653716894177
0.636363636363636 0.99228053579207
0.646464646464647 0.991907586796795
0.656565656565657 0.991535988805319
0.666666666666667 0.991166774553662
0.676767676767677 0.990800898848425
0.686868686868687 0.990439245717076
0.696969696969697 0.990082634776864
0.707070707070707 0.989731827013838
0.717171717171717 0.989387529814783
0.727272727272727 0.989050401524185
0.737373737373737 0.988721055506487
0.747474747474748 0.988400063759373
0.757575757575758 0.988087960163603
0.767676767676768 0.987785243407571
0.777777777777778 0.98749237952657
0.787878787878788 0.987209804252186
0.797979797979798 0.986937925153915
0.808080808080808 0.986677123375871
0.818181818181818 0.986427755347704
0.828282828282828 0.986190154273669
0.838383838383838 0.985964631400791
0.848484848484849 0.985751477172942
0.858585858585859 0.985550962231492
0.868686868686869 0.985363338510309
0.878787878787879 0.985188839721752
0.888888888888889 0.985027682216574
0.898989898989899 0.984880065604489
0.909090909090909 0.984746173161332
0.919191919191919 0.984626172364131
0.929292929292929 0.98452021522951
0.939393939393939 0.984428438422413
0.94949494949495 0.984350963897994
0.95959595959596 0.984287898675001
0.96969696969697 0.984239335301484
0.97979797979798 0.984205351601592
0.98989898989899 0.984186011077199
1 0.984181362600645
};
\addlegendentry{$\rho_{22}$}
\addplot [semithick, forestgreen4416044, dashed, line width=1pt]
table {%
0 0.997248521874397
1 0.997248521874397
};
\addlegendentry{$\rho_{22}^{*}=0.997$}
\addplot [semithick, black, dashed, line width=1pt]
table {%
0.4 0
0.4 1
};
\addlegendentry{$\alpha=0.4$}
\addplot [semithick, darkviolet1910191, dashed, line width=1pt]
table {%
0.454545454545455 0
0.454545454545455 1
};
\addlegendentry{$\gamma^*(\rho_{22})=0.45$}

\nextgroupplot[
axis line style={lightgray204},
legend cell align={left},
legend style={fill opacity=0.8, draw opacity=1, text opacity=1, draw=lightgray204},
tick align=outside,
title={\(\displaystyle \alpha=0.4\)},
width=0.37\textwidth,
height=0.3\textwidth,
x grid style={lightgray204},
xlabel=\textcolor{darkslategray38}{\(\displaystyle \gamma\)},
xmajorgrids,
xmajorticks=true,
xmin=-0.05, xmax=1.05,
xtick style={color=darkslategray38},
y grid style={lightgray204},
ymajorgrids,
ymajorticks=true,
ymin=-0.05, ymax=1.05,
ytick style={color=darkslategray38}
]
\addplot [semithick, steelblue31119180, line width=1.3pt]
table {%
0 0.999986050781245
0.0101010101010101 0.999972075801677
0.0202020202020202 0.999957518782439
0.0303030303030303 0.999942351988544
0.0404040404040404 0.999926546070652
0.0505050505050505 0.99991007002168
0.0606060606060606 0.999892890989895
0.0707070707070707 0.999874974199518
0.0808080808080808 0.999856282776528
0.0909090909090909 0.999836777645597
0.101010101010101 0.999816417323896
0.111111111111111 0.999795157780384
0.121212121212121 0.999772952237249
0.131313131313131 0.999749750943253
0.141414141414141 0.999725500977738
0.151515151515152 0.999700146081189
0.161616161616162 0.999673626179052
0.171717171717172 0.99964587727421
0.181818181818182 0.999616831141651
0.191919191919192 0.999586414870289
0.202020202020202 0.999554550595887
0.212121212121212 0.999521154891708
0.222222222222222 0.999486138735034
0.232323232323232 0.999449406471315
0.242424242424242 0.894223136476104
0.252525252525253 0.868948268408355
0.262626262626263 0.848108270949134
0.272727272727273 0.829635557198672
0.282828282828283 0.81270273870777
0.292929292929293 0.796874771983191
0.303030303030303 0.781887395357767
0.313131313131313 0.767564777379789
0.323232323232323 0.753782300516391
0.333333333333333 0.740447445325761
0.343434343434343 0.72748903315633
0.353535353535354 0.714850740126133
0.363636363636364 0.702486970439729
0.373737373737374 0.690360113829832
0.383838383838384 0.678438656543225
0.393939393939394 0.666695840696758
0.404040404040404 0.65510868812399
0.414141414141414 0.643657273926715
0.424242424242424 0.632324175277753
0.434343434343434 0.621094046531216
0.444444444444444 0.60995328645213
0.454545454545455 0.598889774049998
0.464646464646465 0.587892656866303
0.474747474747475 0.576952179543985
0.484848484848485 0.566059542496478
0.494949494949495 0.555206786072229
0.505050505050505 0.544386693855969
0.515151515151515 0.533592710873586
0.525252525252525 0.522818876315494
0.535353535353535 0.512059763972665
0.545454545454546 0.501310432949013
0.555555555555556 0.490566385314292
0.565656565656566 0.479823528490868
0.575757575757576 0.469078143330537
0.585858585858586 0.458326855732857
0.595959595959596 0.447566612848639
0.606060606060606 0.436794660240275
0.616161616161616 0.426008523543312
0.626262626262626 0.415205990685928
0.636363636363636 0.404385096647972
0.646464646464647 0.393544109525769
0.656565656565657 0.382681518100794
0.666666666666667 0.371796020184025
0.676767676767677 0.360886512097841
0.686868686868687 0.349952079075014
0.696969696969697 0.338991985987859
0.707070707070707 0.32800566906106
0.717171717171717 0.316992727820155
0.727272727272727 0.30595291756388
0.737373737373737 0.294886142109808
0.747474747474748 0.283792446967071
0.757575757575758 0.272672012657767
0.767676767676768 0.261525148006285
0.777777777777778 0.250352284334664
0.787878787878788 0.239153968930737
0.797979797979798 0.227930858921051
0.808080808080808 0.216683715578483
0.818181818181818 0.205413398305544
0.828282828282828 0.194120858641086
0.838383838383838 0.182807134712394
0.848484848484849 0.171473345313025
0.858585858585859 0.160120684214141
0.868686868686869 0.148750414164445
0.878787878787879 0.137363861632824
0.888888888888889 0.12596241087748
0.898989898989899 0.114547498202394
0.909090909090909 0.103120606440379
0.919191919191919 0.0916832593332644
0.929292929292929 0.0802370157223887
0.939393939393939 0.0687834642798719
0.94949494949495 0.0573242176629957
0.95959595959596 0.0458609072231645
0.96969696969697 0.0343951773818377
0.97979797979798 0.0229286803632692
0.98989898989899 0.011463070708479
1 8.97672406346137e-28
};
\addlegendentry{$\kappa$}
\addplot [semithick, darkorange25512714, line width=1.3pt]
table {%
0 0.997355559510821
0.0101010101010101 0.997341832601306
0.0202020202020202 0.99732781780306
0.0303030303030303 0.997313505884298
0.0404040404040404 0.997298887216653
0.0505050505050505 0.997283951736115
0.0606060606060606 0.997268688931921
0.0707070707070707 0.997253087828521
0.0808080808080808 0.997237136929899
0.0909090909090909 0.997220824221847
0.101010101010101 0.997204137113788
0.111111111111111 0.997187062426971
0.121212121212121 0.997169586342005
0.131313131313131 0.997151694354087
0.141414141414141 0.997133371230834
0.151515151515152 0.997114601093192
0.161616161616162 0.997095367055247
0.171717171717172 0.997075651494999
0.181818181818182 0.99705543584044
0.191919191919192 0.997034700526549
0.202020202020202 0.997013424959766
0.212121212121212 0.996991587328
0.222222222222222 0.996969164789819
0.232323232323232 0.996946133024028
0.242424242424242 0.9072550202881
0.252525252525253 0.886546559115557
0.262626262626263 0.869897072733732
0.272727272727273 0.855497802238057
0.282828282828283 0.842619018560271
0.292929292929293 0.830874200735896
0.303030303030303 0.820026986426451
0.313131313131313 0.809919150890316
0.323232323232323 0.800437940706309
0.333333333333333 0.791499234095969
0.343434343434343 0.783038030328664
0.353535353535354 0.775002694405169
0.363636363636364 0.767351281839012
0.373737373737374 0.760049087028277
0.383838383838384 0.753066948213186
0.393939393939394 0.746380039838167
0.404040404040404 0.739966989801911
0.414141414141414 0.733809219781274
0.424242424242424 0.727890442454347
0.434343434343434 0.722196272728557
0.444444444444444 0.716713921220614
0.454545454545455 0.711431950294124
0.464646464646465 0.706340077364487
0.474747474747475 0.701429014852275
0.484848484848485 0.696690337607884
0.494949494949495 0.692116373598101
0.505050505050505 0.687700112179662
0.515151515151515 0.683435126094955
0.525252525252525 0.679315506863648
0.535353535353535 0.675335807427787
0.545454545454546 0.671490994212078
0.555555555555556 0.667776406024494
0.565656565656566 0.66418771697981
0.575757575757576 0.660720905278658
0.585858585858586 0.65737222510434
0.595959595959596 0.654138182913042
0.606060606060606 0.651015515677392
0.616161616161616 0.648001171109986
0.626262626262626 0.645092291674462
0.636363636363636 0.642286198861457
0.646464646464647 0.639580379736063
0.656565656565657 0.63697247498887
0.666666666666667 0.634460267785344
0.676767676767677 0.632041673992215
0.686868686868687 0.629714733409131
0.696969696969697 0.627477601459528
0.707070707070707 0.625328542134556
0.717171717171717 0.623265921178516
0.727272727272727 0.621288200162753
0.737373737373737 0.619393930834585
0.747474747474748 0.617581750276404
0.757575757575758 0.615850376260496
0.767676767676768 0.614198602630958
0.777777777777778 0.612625296291733
0.787878787878788 0.611129393285253
0.797979797979798 0.609709894720977
0.808080808080808 0.608365865160245
0.818181818181818 0.607096429246679
0.828282828282828 0.605900768932822
0.838383838383838 0.604778121709619
0.848484848484849 0.603727778430535
0.858585858585859 0.60274908174235
0.868686868686869 0.601841422112685
0.878787878787879 0.601004239574703
0.888888888888889 0.600237020402841
0.898989898989899 0.599539295540514
0.909090909090909 0.598910640248659
0.919191919191919 0.598350672938039
0.929292929292929 0.597859052942037
0.939393939393939 0.597435482699298
0.94949494949495 0.597079702818397
0.95959595959596 0.596791495555679
0.96969696969697 0.596570680205679
0.97979797979798 0.596417117147663
0.98989898989899 0.596330704559501
1 0.596311378307995
};
\addlegendentry{$\eta$}
\end{groupplot}

\end{tikzpicture}

%% file: figs/optimization_of_gamma_alpha_0.5.tex
\begin{tikzpicture}

\definecolor{darkorange25512714}{RGB}{255,127,14}
\definecolor{darkslategray38}{RGB}{38,38,38}
\definecolor{darkviolet1910191}{RGB}{191,0,191}
\definecolor{forestgreen4416044}{RGB}{44,160,44}
\definecolor{lightgray204}{RGB}{204,204,204}
\definecolor{steelblue31119180}{RGB}{31,119,180}

\begin{groupplot}[group style={group size=3 by 1}]
\nextgroupplot[
axis line style={lightgray204},
legend cell align={left},
legend style={
  fill opacity=0.8,
  draw opacity=1,
  text opacity=1,
  at={(0.03,0.03)},
  anchor=south west,
  draw=lightgray204
},
width=0.37\textwidth,
height=0.3\textwidth,
tick align=outside,
title={\(\displaystyle \alpha=0.5\)},
x grid style={lightgray204},
xlabel=\textcolor{darkslategray38}{\(\displaystyle \gamma\)},
xmajorgrids,
xmajorticks=true,
xmin=-0.05, xmax=1.05,
xtick style={color=darkslategray38},
y grid style={lightgray204},
ymajorgrids,
ymajorticks=true,
ymin=-0.05, ymax=1.05,
ytick style={color=darkslategray38}
]
\addplot [semithick, steelblue31119180, line width=1.3pt]
table {%
0 0.906053543616065
0.0101010101010101 0.906043248301294
0.0202020202020202 0.906032522750757
0.0303030303030303 0.90602134647399
0.0404040404040404 0.906009697768346
0.0505050505050505 0.905997553751087
0.0606060606060606 0.905984890149505
0.0707070707070707 0.905971681260714
0.0808080808080808 0.905957899895524
0.0909090909090909 0.905943517220056
0.101010101010101 0.905928502573419
0.111111111111111 0.905912823493222
0.121212121212121 0.905896445403109
0.131313131313131 0.905879331721823
0.141414141414141 0.905861443306761
0.151515151515152 0.905842738636481
0.161616161616162 0.905823173516295
0.171717171717172 0.905802700741525
0.181818181818182 0.905781269953493
0.191919191919192 0.905758827485304
0.202020202020202 0.905735315938641
0.212121212121212 0.905710673949178
0.222222222222222 0.905684835820658
0.232323232323232 0.90565773128563
0.242424242424242 0.905629284752532
0.252525252525253 0.905599415315749
0.262626262626263 0.905568036553758
0.272727272727273 0.905535054502457
0.282828282828283 0.905500368846164
0.292929292929293 0.905463871277716
0.303030303030303 0.90542544503599
0.313131313131313 0.905384964109342
0.323232323232323 0.905342292245635
0.333333333333333 0.905297282110436
0.343434343434343 0.905249774095371
0.353535353535354 0.905199595061177
0.363636363636364 0.905146557027448
0.373737373737374 0.905090455568176
0.383838383838384 0.905031068150778
0.393939393939394 0.904968152102759
0.404040404040404 0.904901442443508
0.414141414141414 0.90483064990303
0.424242424242424 0.904755453627102
0.434343434343434 0.898438476301397
0.444444444444444 0.682445700083004
0.454545454545455 0.635748284418796
0.464646464646465 0.596348404778399
0.474747474747475 0.56136833527422
0.484848484848485 0.529447378061439
0.494949494949495 0.499806873940476
0.505050505050505 0.471948399810147
0.515151515151515 0.445528289892907
0.525252525252525 0.42029669347443
0.535353535353535 0.39606463136545
0.545454545454546 0.372684731524853
0.555555555555556 0.35003925651179
0.565656565656566 0.32803229443877
0.575757575757576 0.306584463810858
0.585858585858586 0.285629201263264
0.595959595959596 0.265110086609909
0.606060606060606 0.244978869039471
0.616161616161616 0.225193981959513
0.626262626262626 0.205719401916424
0.636363636363636 0.186523761448317
0.646464646464647 0.167579646676988
0.656565656565657 0.148863035447227
0.666666666666667 0.130352841842837
0.676767676767677 0.112030542829007
0.686868686868687 0.0938798696982899
0.696969696969697 0.0758865495026825
0.707070707070707 0.0580380872509653
0.717171717171717 0.0403235802814424
0.727272727272727 0.0227335586987092
0.737373737373737 0.00525984715684618
0.747474747474748 0.0121045562872566
0.757575757575758 0.0293655866106784
0.767676767676768 0.0465282070699992
0.777777777777778 0.0635964940164249
0.787878787878788 0.0805737142902287
0.797979797979798 0.0974623950073361
0.808080808080808 0.114264388027208
0.818181818181818 0.130980929647162
0.828282828282828 0.14761269539142
0.838383838383838 0.16415985104349
0.848484848484849 0.180622104705171
0.858585858585859 0.196998746921234
0.868686868686869 0.213288694846656
0.878787878787879 0.229490538277899
0.888888888888889 0.245602572048785
0.898989898989899 0.261622836048322
0.909090909090909 0.277549148584174
0.919191919191919 0.293379139814479
0.929292929292929 0.30911028461761
0.939393939393939 0.324739928379297
0.94949494949495 0.340265319254616
0.95959595959596 0.355683632436008
0.96969696969697 0.370991994616543
0.97979797979798 0.386187508967384
0.98989898989899 0.401267275817765
1 0.416228411798615
};
\addlegendentry{$\theta_{21}$}
\addplot [semithick, darkorange25512714, line width=1.3pt]
table {%
0 0.817263610107352
0.0101010101010101 0.817254323723874
0.0202020202020202 0.817244649237506
0.0303030303030303 0.817234568163267
0.0404040404040404 0.81722406100066
0.0505050505050505 0.817213107038831
0.0606060606060606 0.817201684417876
0.0707070707070707 0.817189769976212
0.0808080808080808 0.817177339162506
0.0909090909090909 0.817164365908945
0.101010101010101 0.81715082266064
0.111111111111111 0.817136680054807
0.121212121212121 0.817121907029462
0.131313131313131 0.81710647035173
0.141414141414141 0.817090334931245
0.151515151515152 0.817073463278257
0.161616161616162 0.817055815448154
0.171717171717172 0.817037348885001
0.181818181818182 0.817018018274155
0.191919191919192 0.81699777509181
0.202020202020202 0.816976567589328
0.212121212121212 0.816954340404296
0.222222222222222 0.81693103425436
0.232323232323232 0.816906585808525
0.242424242424242 0.816880927063125
0.252525252525253 0.81685398510596
0.262626262626263 0.816825680968441
0.272727272727273 0.816795931060874
0.282828282828283 0.816764644454194
0.292929292929293 0.816731723533263
0.303030303030303 0.8166970629158
0.313131313131313 0.816660548937822
0.323232323232323 0.816622058790018
0.333333333333333 0.81658145949822
0.343434343434343 0.816538607084035
0.353535353535354 0.816493345401365
0.363636363636364 0.816445504883069
0.373737373737374 0.816394901185868
0.383838383838384 0.816341333521109
0.393939393939394 0.816284583020047
0.404040404040404 0.816224410681189
0.414141414141414 0.816160554448572
0.424242424242424 0.816092732805362
0.434343434343434 0.82431843316755
0.444444444444444 0.970846225396078
0.454545454545455 0.982991635117386
0.464646464646465 0.989972210633864
0.474747474747475 0.993955054534192
0.484848484848485 0.995935118896549
0.494949494949495 0.996463335055211
0.505050505050505 0.995879864069826
0.515151515151515 0.994410298673058
0.525252525252525 0.992211985197582
0.535353535353535 0.989398832166015
0.545454545454546 0.986055677017034
0.555555555555556 0.982247121100725
0.565656565656566 0.978023229312416
0.575757575757576 0.973423352142178
0.585858585858586 0.968478773982337
0.595959595959596 0.963214598784204
0.606060606060606 0.957651125082613
0.616161616161616 0.951804868878978
0.626262626262626 0.945689339617503
0.636363636363636 0.939315637125754
0.646464646464647 0.932692918510169
0.656565656565657 0.92582876762521
0.666666666666667 0.918729490773854
0.676767676767677 0.911400356619399
0.686868686868687 0.903845791762336
0.696969696969697 0.896069542361362
0.707070707070707 0.888074808216381
0.717171717171717 0.879864354961874
0.727272727272727 0.871440608394796
0.737373737373737 0.862805733985006
0.747474747474748 0.853961705264572
0.757575757575758 0.844910359267964
0.767676767676768 0.835653446985484
0.777777777777778 0.826192673495761
0.787878787878788 0.816529733045943
0.797979797979798 0.806666337930935
0.808080808080808 0.79660424261764
0.818181818181818 0.786345263444192
0.828282828282828 0.775891294190962
0.838383838383838 0.765244317835974
0.848484848484849 0.754406417131905
0.858585858585859 0.743379776653742
0.868686868686869 0.732166689505391
0.878787878787879 0.72076955662543
0.888888888888889 0.709190884786201
0.898989898989899 0.697433282728603
0.909090909090909 0.685499455861352
0.919191919191919 0.67339219830431
0.929292929292929 0.661114384356828
0.939393939393939 0.648668956901682
0.94949494949495 0.636058916273828
0.95959595959596 0.623287306894001
0.96969696969697 0.610357202282618
0.97979797979798 0.597271690908281
0.98989898989899 0.584033858521465
1 0.570646772302553
};
\addlegendentry{$\theta_{22}$}
\addplot [semithick, forestgreen4416044, dashed, line width=1pt]
table {%
0 0.996463335055211
1 0.996463335055211
};
\addlegendentry{$\theta_{22}^{*}=0.996$}
\addplot [semithick, black, dashed, line width=1pt]
table {%
0.5 0
0.5 1
};
\addlegendentry{$\alpha=0.5$}
\addplot [semithick, darkviolet1910191, dashed, line width=1pt]
table {%
0.494949494949495 0
0.494949494949495 1
};
\addlegendentry{$\gamma^*(\theta_{22})=0.49$}

\nextgroupplot[
axis line style={lightgray204},
legend cell align={left},
legend style={
  fill opacity=0.8,
  draw opacity=1,
  text opacity=1,
  at={(0.03,0.03)},
  anchor=south west,
  draw=lightgray204
},
width=0.37\textwidth,
height=0.3\textwidth,
tick align=outside,
title={\(\displaystyle \alpha=0.5\)},
x grid style={lightgray204},
xlabel=\textcolor{darkslategray38}{\(\displaystyle \gamma\)},
xmajorgrids,
xmajorticks=true,
xmin=-0.05, xmax=1.05,
xtick style={color=darkslategray38},
y grid style={lightgray204},
ymajorgrids,
ymajorticks=true,
ymin=-0.05, ymax=1.05,
ytick style={color=darkslategray38}
]
\addplot [thick, steelblue31119180, line width=1.3pt]
table {%
0 0.906051394533005
0.0101010101010101 0.906041203290223
0.0202020202020202 0.906030798950766
0.0303030303030303 0.906020174706119
0.0404040404040404 0.906009323443897
0.0505050505050505 0.90599823776098
0.0606060606060606 0.905986909926397
0.0707070707070707 0.905975331809603
0.0808080808080808 0.905963495004186
0.0909090909090909 0.905951390680072
0.101010101010101 0.905939009537434
0.111111111111111 0.905926341910781
0.121212121212121 0.905913377597629
0.131313131313131 0.905900106045396
0.141414141414141 0.905886515989812
0.151515151515152 0.905872595721602
0.161616161616162 0.905858332967938
0.171717171717172 0.905843714752093
0.181818181818182 0.90582872742567
0.191919191919192 0.905813356693913
0.202020202020202 0.905797587455474
0.212121212121212 0.905781403798179
0.222222222222222 0.905764788925701
0.232323232323232 0.905747725156061
0.242424242424242 0.905730193612115
0.252525252525253 0.905712174511555
0.262626262626263 0.905693647283818
0.272727272727273 0.905674589266062
0.282828282828283 0.905654977090064
0.292929292929293 0.905634785737731
0.303030303030303 0.905613988668466
0.313131313131313 0.905592557654776
0.323232323232323 0.905570462591128
0.333333333333333 0.905547671455753
0.343434343434343 0.905524150046107
0.353535353535354 0.905499861805043
0.363636363636364 0.90547476763387
0.373737373737374 0.905448825624043
0.383838383838384 0.905421990860506
0.393939393939394 0.905394215055216
0.404040404040404 0.905365446256571
0.414141414141414 0.905335628905851
0.424242424242424 0.905304699597373
0.434343434343434 0.899913174579114
0.444444444444444 0.722008311063531
0.454545454545455 0.685853069468037
0.464646464646465 0.656178043198096
0.474747474747475 0.630537735612511
0.484848484848485 0.607769131861585
0.494949494949495 0.587203081086243
0.505050505050505 0.568409054158966
0.515151515151515 0.55108874063986
0.525252525252525 0.535024231776094
0.535353535353535 0.520049933521447
0.545454545454546 0.506036094749036
0.555555555555556 0.492878541050218
0.565656565656566 0.480491958139094
0.575757575757576 0.468805323506428
0.585858585858586 0.457758693108462
0.595959595959596 0.447300878737863
0.606060606060606 0.437387728203205
0.616161616161616 0.427980827155183
0.626262626262626 0.419046498379125
0.636363636363636 0.410555021570192
0.646464646464647 0.402480013648003
0.656565656565657 0.394797932148528
0.666666666666667 0.387487672191978
0.676767676767677 0.380530235599747
0.686868686868687 0.373908457846307
0.696969696969697 0.367606779828289
0.707070707070707 0.361611056363578
0.717171717171717 0.355908394081578
0.727272727272727 0.350487013498377
0.737373737373737 0.345336131092691
0.747474747474748 0.340445856517656
0.757575757575758 0.335807107232659
0.767676767676768 0.331411529687
0.777777777777778 0.32725143470747
0.787878787878788 0.32331973800074
0.797979797979798 0.319609910537968
0.808080808080808 0.316115933092002
0.818181818181818 0.312832256920677
0.828282828282828 0.30975376917846
0.838383838383838 0.306875763343267
0.848484848484849 0.30419390520936
0.858585858585859 0.301704222342412
0.868686868686869 0.299403071114927
0.878787878787879 0.297287120269533
0.888888888888889 0.295353336361078
0.898989898989899 0.293598969740346
0.909090909090909 0.292021539218006
0.919191919191919 0.290618822813294
0.929292929292929 0.28938884530885
0.939393939393939 0.288329874150122
0.94949494949495 0.287440408584861
0.95959595959596 0.286719174241806
0.96969696969697 0.286165120754224
0.97979797979798 0.285777414143965
0.98989898989899 0.285555437571488
1 0.285498786024296
};
\addlegendentry{$\rho_{21}$}
\addplot [thick, darkorange25512714, line width=1.3pt]
table {%
0 0.817261671627573
0.0101010101010101 0.817252479116999
0.0202020202020202 0.817243094360041
0.0303030303030303 0.817233511235212
0.0404040404040404 0.817223723342662
0.0505050505050505 0.817213724020881
0.0606060606060606 0.817203506242256
0.0707070707070707 0.817193062785204
0.0808080808080808 0.817182385968799
0.0909090909090909 0.817171467798678
0.101010101010101 0.81716029997527
0.111111111111111 0.817148873716868
0.121212121212121 0.81713717992266
0.131313131313131 0.817125208872363
0.141414141414141 0.817112950590415
0.151515151515152 0.817100394481179
0.161616161616162 0.817087529401872
0.171717171717172 0.817074343682563
0.181818181818182 0.817060825088809
0.191919191919192 0.817046960637463
0.202020202020202 0.81703273672507
0.212121212121212 0.817018138993174
0.222222222222222 0.817003152250702
0.232323232323232 0.816987760605026
0.242424242424242 0.816971947223773
0.252525252525253 0.816955694264162
0.262626262626263 0.816938982305199
0.272727272727273 0.816921791925723
0.282828282828283 0.816904101655781
0.292929292929293 0.816885888984541
0.303030303030303 0.816867129959656
0.313131313131313 0.816847799078372
0.323232323232323 0.816827869276253
0.333333333333333 0.816807311599371
0.343434343434343 0.816786095193191
0.353535353535354 0.816764187102754
0.363636363636364 0.816741552052081
0.373737373737374 0.816718152285336
0.383838383838384 0.816693947233391
0.393939393939394 0.816668893358552
0.404040404040404 0.816642943812651
0.414141414141414 0.816616047515015
0.424242424242424 0.816588154015382
0.434343434343434 0.823715146138867
0.444444444444444 0.957440646512842
0.454545454545455 0.970414448750262
0.464646464646465 0.978814152357043
0.474747474747475 0.984613451764736
0.484848484848485 0.98872263903377
0.494949494949495 0.991648287812013
0.505050505050505 0.993704954661171
0.515151515151515 0.995102634330594
0.525252525252525 0.995988932815284
0.535353535353535 0.996471683739218
0.545454545454546 0.996632060560265
0.555555555555556 0.996532647690351
0.565656565656566 0.996222650615484
0.575757575757576 0.995741390971724
0.585858585858586 0.995120727936274
0.595959595959596 0.99438678131306
0.606060606060606 0.99356118677888
0.616161616161616 0.992662027519945
0.626262626262626 0.991704539490557
0.636363636363636 0.990701651934252
0.646464646464647 0.98966440819937
0.656565656565657 0.988602296890683
0.666666666666667 0.987523515561414
0.676767676767677 0.986435182398693
0.686868686868687 0.985343507978463
0.696969696969697 0.984253935422838
0.707070707070707 0.983171255784529
0.717171717171717 0.982099703520797
0.727272727272727 0.981043035980724
0.737373737373737 0.980004599860117
0.747474747474748 0.978987387493127
0.757575757575758 0.977994083155526
0.767676767676768 0.977027103840434
0.777777777777778 0.976088632631356
0.787878787878788 0.975180647657061
0.797979797979798 0.974304946509546
0.808080808080808 0.973463167151797
0.818181818181818 0.97265680569138
0.828282828282828 0.971887231525213
0.838383838383838 0.971155700144798
0.848484848484849 0.970463364381379
0.858585858585859 0.969811282761017
0.868686868686869 0.96920042843012
0.878787878787879 0.96863169482011
0.888888888888889 0.968105901420328
0.898989898989899 0.967623797973714
0.909090909090909 0.967186068390257
0.919191919191919 0.966793333503582
0.929292929292929 0.966446153361283
0.939393939393939 0.966145029235801
0.94949494949495 0.965890404650977
0.95959595959596 0.965682666651033
0.96969696969697 0.965522146091108
0.97979797979798 0.965409118358336
0.98989898989899 0.965343803152178
1 0.965326364621182
};
\addlegendentry{$\rho_{22}$}
\addplot [semithick, forestgreen4416044, dashed, line width=1pt]
table {%
0 0.996632060560265
1 0.996632060560265
};
\addlegendentry{$\rho_{22}^{*}=0.997$}
\addplot [semithick, black, dashed, line width=1pt]
table {%
0.5 0
0.5 1
};
\addlegendentry{$\alpha=0.5$}
\addplot [semithick, darkviolet1910191, dashed, line width=1pt]
table {%
0.545454545454546 0
0.545454545454546 1
};
\addlegendentry{$\gamma^*(\rho_{22})=0.55$}

\nextgroupplot[
axis line style={lightgray204},
legend cell align={left},
legend style={
  fill opacity=0.8,
  draw opacity=1,
  text opacity=1,
  at={(0.03,0.03)},
  anchor=south west,
  draw=lightgray204
},
width=0.37\textwidth,
height=0.3\textwidth,
tick align=outside,
title={\(\displaystyle \alpha=0.5\)},
x grid style={lightgray204},
xlabel=\textcolor{darkslategray38}{\(\displaystyle \gamma\)},
xmajorgrids,
xmajorticks=true,
xmin=-0.05, xmax=1.05,
xtick style={color=darkslategray38},
y grid style={lightgray204},
ymajorgrids,
ymajorticks=true,
ymin=-0.05, ymax=1.05,
ytick style={color=darkslategray38}
]
\addplot [semithick, steelblue31119180, line width=1.3pt]
table {%
0 0.999990532976802
0.0101010101010101 0.999979073147373
0.0202020202020202 0.999967136468508
0.0303030303030303 0.999954700198677
0.0404040404040404 0.999941740380079
0.0505050505050505 0.999928231678325
0.0606060606060606 0.99991414722871
0.0707070707070707 0.999899458601655
0.0808080808080808 0.999884135692853
0.0909090909090909 0.99986814657314
0.101010101010101 0.999851457368855
0.111111111111111 0.99983403209956
0.121212121212121 0.999815832599835
0.131313131313131 0.999796818249149
0.141414141414141 0.999776945880283
0.151515151515152 0.999756169519254
0.161616161616162 0.999734440237854
0.171717171717172 0.999711705825637
0.181818181818182 0.999687910649153
0.191919191919192 0.999662995264025
0.202020202020202 0.99963689620523
0.212121212121212 0.999609545589212
0.222222222222222 0.999580870748993
0.232323232323232 0.999550793973711
0.242424242424242 0.999519231758923
0.252525252525253 0.999486094699427
0.262626262626263 0.999451286631913
0.272727272727273 0.999414704303463
0.282828282828283 0.999376236532428
0.292929292929293 0.999335763606721
0.303030303030303 0.999293156369044
0.313131313131313 0.999248275485397
0.323232323232323 0.999200970382097
0.333333333333333 0.999151078051662
0.343434343434343 0.999098422142639
0.353535353535354 0.99904281114267
0.363636363636364 0.9989840372642
0.373737373737374 0.998921874548102
0.383838383838384 0.998856076949467
0.393939393939394 0.998786376227517
0.404040404040404 0.998712479474288
0.414141414141414 0.998634066379665
0.424242424242424 0.99855078621933
0.434343434343434 0.99835506582706
0.444444444444444 0.925975822561196
0.454545454545455 0.90080377842319
0.464646464646465 0.877972969730727
0.474747474747475 0.85662383387719
0.484848484848485 0.836335446434139
0.494949494949495 0.816857954175254
0.505050505050505 0.798025339654132
0.515151515151515 0.779718773441581
0.525252525252525 0.761848620306069
0.535353535353535 0.74434459769605
0.545454545454546 0.727149954235247
0.555555555555556 0.710217804463908
0.565656565656566 0.693508708242036
0.575757575757576 0.676989008205716
0.585858585858586 0.660629649336207
0.595959595959596 0.644405316881519
0.606060606060606 0.628293792734239
0.616161616161616 0.612275465073823
0.626262626262626 0.596332946884381
0.636363636363636 0.580450776243216
0.646464646464647 0.564615176296669
0.656565656565657 0.548813861441931
0.666666666666667 0.53303587753699
0.676767676767677 0.517271471109229
0.686868686868687 0.501511978488989
0.696969696969697 0.485749732907236
0.707070707070707 0.469977984545973
0.717171717171717 0.45419083171855
0.727272727272727 0.438383160618679
0.737373737373737 0.42255059180934
0.747474747474748 0.406689433384892
0.757575757575758 0.390796636357973
0.767676767676768 0.374869756461622
0.777777777777778 0.358906916973395
0.787878787878788 0.342906774239319
0.797979797979798 0.326868485302177
0.808080808080808 0.310791676776879
0.818181818181818 0.294676414698326
0.828282828282828 0.278523175508909
0.838383838383838 0.262332817700222
0.848484848484849 0.246106552951084
0.858585858585859 0.229845920389772
0.868686868686869 0.213552757969323
0.878787878787879 0.19722917577687
0.888888888888889 0.180877530083457
0.898989898989899 0.164500396578373
0.909090909090909 0.148100544641771
0.919191919191919 0.131680911650286
0.929292929292929 0.115244577481636
0.939393939393939 0.0987947400517486
0.94949494949495 0.0823346904212347
0.95959595959596 0.0658677890250969
0.96969696969697 0.04939744214028
0.97979797979798 0.0329270788171799
0.98989898989899 0.0164601287057967
1 2.57291072323897e-26
};
\addlegendentry{$\kappa$}
\addplot [semithick, darkorange25512714, line width=1.3pt]
table {%
0 0.997821966655984
0.0101010101010101 0.997810694986307
0.0202020202020202 0.997799187667116
0.0303030303030303 0.997787437103152
0.0404040404040404 0.997775435466407
0.0505050505050505 0.997763174583655
0.0606060606060606 0.997750645895171
0.0707070707070707 0.997737840427993
0.0808080808080808 0.997724748863462
0.0909090909090909 0.997711361424468
0.101010101010101 0.997697667866463
0.111111111111111 0.997683657453159
0.121212121212121 0.997669318975625
0.131313131313131 0.997654640644659
0.141414141414141 0.997639610122532
0.151515151515152 0.99762421442206
0.161616161616162 0.99760843994205
0.171717171717172 0.997592272346275
0.181818181818182 0.99757569659881
0.191919191919192 0.997558696831399
0.202020202020202 0.997541256363741
0.212121212121212 0.997523357597258
0.222222222222222 0.997504981931737
0.232323232323232 0.997486109871382
0.242424242424242 0.997466720581784
0.252525252525253 0.997446792265024
0.262626262626263 0.9974263016893
0.272727272727273 0.997405224335707
0.282828282828283 0.997383534165075
0.292929292929293 0.997361203549933
0.303030303030303 0.997338203131653
0.313131313131313 0.997314501683052
0.323232323232323 0.997290065981565
0.333333333333333 0.997264860557225
0.343434343434343 0.997238847655453
0.353535353535354 0.997211986819527
0.363636363636364 0.997184234844221
0.373737373737374 0.997155545421357
0.383838383838384 0.99712586887695
0.393939393939394 0.997095151844157
0.404040404040404 0.99706333690142
0.414141414141414 0.997030362168032
0.424242424242424 0.99699616087791
0.434343434343434 0.996881416602703
0.444444444444444 0.943833740925481
0.454545454545455 0.926155279741139
0.464646464646465 0.910553268906176
0.474747474747475 0.896364311773046
0.484848484848485 0.883259147878634
0.494949494949495 0.871039995246934
0.505050505050505 0.859574079389945
0.515151515151515 0.848765536384946
0.525252525252525 0.838541528252847
0.535353535353535 0.828844597533215
0.545454545454546 0.819628111702393
0.555555555555556 0.810853373552752
0.565656565656566 0.80248769696884
0.575757575757576 0.794503074397503
0.585858585858586 0.786875221792057
0.595959595959596 0.779582874802291
0.606060606060606 0.772607257900915
0.616161616161616 0.76593167578578
0.626262626262626 0.759541192165368
0.636363636363636 0.753422374849351
0.646464646464647 0.747563089029282
0.656565656565657 0.741952329137101
0.666666666666667 0.736580078735559
0.676767676767677 0.73143719487459
0.686868686868687 0.726515309703392
0.696969696969697 0.721806747528744
0.707070707070707 0.717304453686009
0.717171717171717 0.713001933441927
0.727272727272727 0.708893199146942
0.737373737373737 0.70497272415316
0.747474747474748 0.701235402328559
0.757575757575758 0.697676513411413
0.767676767676768 0.694291690972624
0.777777777777778 0.691076895565805
0.787878787878788 0.68802838944108
0.797979797979798 0.685142715315435
0.808080808080808 0.682416676638363
0.818181818181818 0.679847319874942
0.828282828282828 0.677431919123556
0.838383838383838 0.675167962617827
0.848484848484849 0.67305313566542
0.858585858585859 0.6710853176897
0.868686868686869 0.669262564746857
0.878787878787879 0.667583101251983
0.888888888888889 0.66604531350528
0.898989898989899 0.664647741747484
0.909090909090909 0.663389072571365
0.919191919191919 0.662268133748288
0.929292929292929 0.661283887218119
0.939393939393939 0.660435427313063
0.94949494949495 0.659721973974171
0.95959595959596 0.659142869984284
0.96969696969697 0.658697579299442
0.97979797979798 0.658385682517296
0.98989898989899 0.658206877104751
1 0.65816097413601
};
\addlegendentry{$\eta$}
\end{groupplot}

\end{tikzpicture}

%% file: figs/model_estimation_mo.tex
\begin{tikzpicture}

\definecolor{crimson2143940}{RGB}{214,39,40}
\definecolor{darkorange25512714}{RGB}{255,127,14}
\definecolor{darkslategray38}{RGB}{38,38,38}
\definecolor{forestgreen4416044}{RGB}{44,160,44}
\definecolor{lightgray204}{RGB}{204,204,204}
\definecolor{mediumpurple148103189}{RGB}{148,103,189}
\definecolor{orchid227119194}{RGB}{227,119,194}
\definecolor{sienna1408675}{RGB}{140,86,75}
\definecolor{steelblue31119180}{RGB}{31,119,180}

\begin{axis}[
axis line style={lightgray204},
legend cell align={left},
legend style={
  fill opacity=0.8,
  draw opacity=1,
  text opacity=1,
  at={(1.2,0.5)},
  anchor=east,
  draw=lightgray204
},
width=0.8\textwidth,
height=0.5\textwidth,
tick align=outside,
x grid style={lightgray204},
xlabel=\textcolor{darkslategray38}{\(\displaystyle \beta_1\)},
xmajorgrids,
xmajorticks=true,
xmin=-0.75, xmax=15.75,
xtick style={color=darkslategray38},
y grid style={lightgray204},
ylabel=\textcolor{darkslategray38}{\(\displaystyle \hat\beta_1, \hat\beta_2\)},
ymajorgrids,
ymajorticks=true,
ymin=-0.8, ymax=16.8,
ytick style={color=darkslategray38}
]
\addplot [thick, steelblue31119180, opacity=0.7, dashed, line width=1.3pt]
table {%
0 0
0.306122448979592 0.306122448979592
0.612244897959184 0.612244897959184
0.918367346938776 0.918367346938776
1.22448979591837 1.22448979591837
1.53061224489796 1.53061224489796
1.83673469387755 1.83673469387755
2.14285714285714 2.14285714285714
2.44897959183673 2.44897959183673
2.75510204081633 2.75510204081633
3.06122448979592 3.06122448979592
3.36734693877551 3.36734693877551
3.6734693877551 3.6734693877551
3.97959183673469 3.97959183673469
4.28571428571429 4.28571428571429
4.59183673469388 4.59183673469388
4.89795918367347 4.89795918367347
5.20408163265306 5.20408163265306
5.51020408163265 5.51020408163265
5.81632653061224 5.81632653061224
6.12244897959184 6.12244897959184
6.42857142857143 6.42857142857143
6.73469387755102 6.73469387755102
7.04081632653061 7.04081632653061
7.3469387755102 7.3469387755102
7.6530612244898 7.6530612244898
7.95918367346939 7.95918367346939
8.26530612244898 8.26530612244898
8.57142857142857 8.57142857142857
8.87755102040816 8.87755102040816
9.18367346938776 9.18367346938776
9.48979591836735 9.48979591836735
9.79591836734694 9.79591836734694
10.1020408163265 10.1020408163265
10.4081632653061 10.4081632653061
10.7142857142857 10.7142857142857
11.0204081632653 11.0204081632653
11.3265306122449 11.3265306122449
11.6326530612245 11.6326530612245
11.9387755102041 11.9387755102041
12.2448979591837 12.2448979591837
12.5510204081633 12.5510204081633
12.8571428571429 12.8571428571429
13.1632653061224 13.1632653061224
13.469387755102 13.469387755102
13.7755102040816 13.7755102040816
14.0816326530612 14.0816326530612
14.3877551020408 14.3877551020408
14.6938775510204 14.6938775510204
15 15
};
\addlegendentry{$y=x$}
\addplot [thick, darkorange25512714, opacity=0.7, dashed, line width=1.3pt]
table {%
0 5
0.306122448979592 5
0.612244897959184 5
0.918367346938776 5
1.22448979591837 5
1.53061224489796 5
1.83673469387755 5
2.14285714285714 5
2.44897959183673 5
2.75510204081633 5
3.06122448979592 5
3.36734693877551 5
3.6734693877551 5
3.97959183673469 5
4.28571428571429 5
4.59183673469388 5
4.89795918367347 5
5.20408163265306 5
5.51020408163265 5
5.81632653061224 5
6.12244897959184 5
6.42857142857143 5
6.73469387755102 5
7.04081632653061 5
7.3469387755102 5
7.6530612244898 5
7.95918367346939 5
8.26530612244898 5
8.57142857142857 5
8.87755102040816 5
9.18367346938776 5
9.48979591836735 5
9.79591836734694 5
10.1020408163265 5
10.4081632653061 5
10.7142857142857 5
11.0204081632653 5
11.3265306122449 5
11.6326530612245 5
11.9387755102041 5
12.2448979591837 5
12.5510204081633 5
12.8571428571429 5
13.1632653061224 5
13.469387755102 5
13.7755102040816 5
14.0816326530612 5
14.3877551020408 5
14.6938775510204 5
15 5
};
\addlegendentry{$y=\beta_2$}
\addplot [thick, forestgreen4416044, opacity=0.7, dashed, line width=1.3pt]
table {%
5 0
5 16
};
\addlegendentry{$\beta_1=\beta_2$}
\path [draw=red]
(axis cs:0,0.221840860419565)
--(axis cs:0,0.221842917703733);

\path [draw=red]
(axis cs:0.306122448979592,0)
--(axis cs:0.306122448979592,0);

\path [draw=red]
(axis cs:0.612244897959184,0)
--(axis cs:0.612244897959184,0);

\path [draw=red]
(axis cs:0.918367346938776,0.224430224012768)
--(axis cs:0.918367346938776,0.224432305513624);

\path [draw=red]
(axis cs:1.22448979591837,0)
--(axis cs:1.22448979591837,0);

\path [draw=red]
(axis cs:1.53061224489796,0.116120486551379)
--(axis cs:1.53061224489796,0.116122027073935);

\path [draw=red]
(axis cs:1.83673469387755,0.21037667245166)
--(axis cs:1.83673469387755,0.210378625260733);

\path [draw=red]
(axis cs:2.14285714285714,1.77047126003246)
--(axis cs:2.14285714285714,1.77047602850433);

\path [draw=red]
(axis cs:2.44897959183673,4.11041235399799)
--(axis cs:2.44897959183673,4.11041618034167);

\path [draw=red]
(axis cs:2.75510204081633,4.99466293862407)
--(axis cs:2.75510204081633,4.99466319551138);

\path [draw=red]
(axis cs:3.06122448979592,5.01650980006589)
--(axis cs:3.06122448979592,5.0165101079961);

\path [draw=red]
(axis cs:3.36734693877551,4.96343799422825)
--(axis cs:3.36734693877551,4.96343828282697);

\path [draw=red]
(axis cs:3.6734693877551,5.03245924215296)
--(axis cs:3.6734693877551,5.03245953123508);

\path [draw=red]
(axis cs:3.97959183673469,5.02546618263812)
--(axis cs:3.97959183673469,5.02546650517665);

\path [draw=red]
(axis cs:4.28571428571429,5.08964751675322)
--(axis cs:4.28571428571429,5.08964787376688);

\path [draw=red]
(axis cs:4.59183673469388,5.44431746375809)
--(axis cs:4.59183673469388,5.44431984581147);

\path [draw=red]
(axis cs:4.89795918367347,7.38180939630741)
--(axis cs:4.89795918367347,7.38181600132952);

\path [draw=red]
(axis cs:5.20408163265306,6.67747045158551)
--(axis cs:5.20408163265306,6.67747543338561);

\path [draw=red]
(axis cs:5.51020408163265,5.69217828410857)
--(axis cs:5.51020408163265,5.69217867611964);

\path [draw=red]
(axis cs:5.81632653061224,5.80540638716016)
--(axis cs:5.81632653061224,5.80540666336088);

\path [draw=red]
(axis cs:6.12244897959184,6.1680019842508)
--(axis cs:6.12244897959184,6.16800229473894);

\path [draw=red]
(axis cs:6.42857142857143,6.48187325083529)
--(axis cs:6.42857142857143,6.48187355159599);

\path [draw=red]
(axis cs:6.73469387755102,6.77699312139552)
--(axis cs:6.73469387755102,6.7769934062761);

\path [draw=red]
(axis cs:7.04081632653061,7.12190191326235)
--(axis cs:7.04081632653061,7.12190294825087);

\path [draw=red]
(axis cs:7.3469387755102,7.58230647828955)
--(axis cs:7.3469387755102,7.58230909438427);

\path [draw=red]
(axis cs:7.6530612244898,7.93254911490278)
--(axis cs:7.6530612244898,7.93255253178769);

\path [draw=red]
(axis cs:7.95918367346939,7.97125725633624)
--(axis cs:7.95918367346939,7.97125752777896);

\path [draw=red]
(axis cs:8.26530612244898,8.29144990396901)
--(axis cs:8.26530612244898,8.29145017807116);

\path [draw=red]
(axis cs:8.57142857142857,8.59082522081743)
--(axis cs:8.57142857142857,8.59082539616864);

\path [draw=red]
(axis cs:8.87755102040816,8.90577291767132)
--(axis cs:8.87755102040816,8.90577312598745);

\path [draw=red]
(axis cs:9.18367346938776,9.21138149370919)
--(axis cs:9.18367346938776,9.21138172810323);

\path [draw=red]
(axis cs:9.48979591836735,9.50137829672423)
--(axis cs:9.48979591836735,9.50137851649834);

\path [draw=red]
(axis cs:9.79591836734694,9.83127758938217)
--(axis cs:9.79591836734694,9.83127777053966);

\path [draw=red]
(axis cs:10.1020408163265,10.1383715748404)
--(axis cs:10.1020408163265,10.1383717736711);

\path [draw=red]
(axis cs:10.4081632653061,10.4021002342359)
--(axis cs:10.4081632653061,10.4021004684403);

\path [draw=red]
(axis cs:10.7142857142857,10.7346189602738)
--(axis cs:10.7142857142857,10.7346191795954);

\path [draw=red]
(axis cs:11.0204081632653,11.0239887857368)
--(axis cs:11.0204081632653,11.0239890178453);

\path [draw=red]
(axis cs:11.3265306122449,11.3399069308239)
--(axis cs:11.3265306122449,11.3399071301554);

\path [draw=red]
(axis cs:11.6326530612245,11.6602203438051)
--(axis cs:11.6326530612245,11.6602205255902);

\path [draw=red]
(axis cs:11.9387755102041,11.9387277153874)
--(axis cs:11.9387755102041,11.9387279081294);

\path [draw=red]
(axis cs:12.2448979591837,12.2569303315497)
--(axis cs:12.2448979591837,12.2569305431179);

\path [draw=red]
(axis cs:12.5510204081633,12.5646354454692)
--(axis cs:12.5510204081633,12.5646356692108);

\path [draw=red]
(axis cs:12.8571428571429,12.8650098647192)
--(axis cs:12.8571428571429,12.8650100452266);

\path [draw=red]
(axis cs:13.1632653061224,13.1884312962012)
--(axis cs:13.1632653061224,13.1884314882522);

\path [draw=red]
(axis cs:13.469387755102,13.4774865183893)
--(axis cs:13.469387755102,13.4774867253377);

\path [draw=red]
(axis cs:13.7755102040816,13.8053722403984)
--(axis cs:13.7755102040816,13.8053724493429);

\path [draw=red]
(axis cs:14.0816326530612,14.1028462493778)
--(axis cs:14.0816326530612,14.1028464612421);

\path [draw=red]
(axis cs:14.3877551020408,14.4082845953679)
--(axis cs:14.3877551020408,14.4082847845235);

\path [draw=red]
(axis cs:14.6938775510204,14.7036742017178)
--(axis cs:14.6938775510204,14.7036744134718);

\path [draw=red]
(axis cs:15,14.9909298813896)
--(axis cs:15,14.9909300797312);

\path [draw=red]
(axis cs:0,0.0521720105917036)
--(axis cs:0,0.0521740678758711);

\path [draw=red]
(axis cs:0.306122448979592,0)
--(axis cs:0.306122448979592,0);

\path [draw=red]
(axis cs:0.612244897959184,0)
--(axis cs:0.612244897959184,0);

\path [draw=red]
(axis cs:0.918367346938776,0.0521364128306636)
--(axis cs:0.918367346938776,0.0521384943315193);

\path [draw=red]
(axis cs:1.22448979591837,0)
--(axis cs:1.22448979591837,0);

\path [draw=red]
(axis cs:1.53061224489796,0.0263139747204987)
--(axis cs:1.53061224489796,0.0263155152430548);

\path [draw=red]
(axis cs:1.83673469387755,0.0962850672753496)
--(axis cs:1.83673469387755,0.0962870200844221);

\path [draw=red]
(axis cs:2.14285714285714,0.748194384607574)
--(axis cs:2.14285714285714,0.748199153079446);

\path [draw=red]
(axis cs:2.44897959183673,1.9455639119538)
--(axis cs:2.44897959183673,1.94556773829748);

\path [draw=red]
(axis cs:2.75510204081633,2.68516302415008)
--(axis cs:2.75510204081633,2.6851632810374);

\path [draw=red]
(axis cs:3.06122448979592,2.97542873890466)
--(axis cs:3.06122448979592,2.97542904683487);

\path [draw=red]
(axis cs:3.36734693877551,3.34917239385182)
--(axis cs:3.36734693877551,3.34917268245054);

\path [draw=red]
(axis cs:3.6734693877551,3.52722674928054)
--(axis cs:3.6734693877551,3.52722703836266);

\path [draw=red]
(axis cs:3.97959183673469,3.87315153045115)
--(axis cs:3.97959183673469,3.87315185298968);

\path [draw=red]
(axis cs:4.28571428571429,4.09205059944493)
--(axis cs:4.28571428571429,4.09205095645859);

\path [draw=red]
(axis cs:4.59183673469388,4.41113002566858)
--(axis cs:4.59183673469388,4.41113240772196);

\path [draw=red]
(axis cs:4.89795918367347,3.79800490048143)
--(axis cs:4.89795918367347,3.79801150550354);

\path [draw=red]
(axis cs:5.20408163265306,4.74601747862156)
--(axis cs:5.20408163265306,4.74602246042166);

\path [draw=red]
(axis cs:5.51020408163265,4.63526036621542)
--(axis cs:5.51020408163265,4.63526075822649);

\path [draw=red]
(axis cs:5.81632653061224,5.03950905988161)
--(axis cs:5.81632653061224,5.03950933608233);

\path [draw=red]
(axis cs:6.12244897959184,4.87598492349258)
--(axis cs:6.12244897959184,4.87598523398072);

\path [draw=red]
(axis cs:6.42857142857143,4.85492424210284)
--(axis cs:6.42857142857143,4.85492454286355);

\path [draw=red]
(axis cs:6.73469387755102,4.86680342231657)
--(axis cs:6.73469387755102,4.86680370719716);

\path [draw=red]
(axis cs:7.04081632653061,4.97911364039571)
--(axis cs:7.04081632653061,4.97911467538423);

\path [draw=red]
(axis cs:7.3469387755102,4.98965206501983)
--(axis cs:7.3469387755102,4.98965468111455);

\path [draw=red]
(axis cs:7.6530612244898,5.07206898899597)
--(axis cs:7.6530612244898,5.07207240588088);

\path [draw=red]
(axis cs:7.95918367346939,4.93435701094568)
--(axis cs:7.95918367346939,4.9343572823884);

\path [draw=red]
(axis cs:8.26530612244898,4.92327705845549)
--(axis cs:8.26530612244898,4.92327733255763);

\path [draw=red]
(axis cs:8.57142857142857,4.93409304658513)
--(axis cs:8.57142857142857,4.93409322193633);

\path [draw=red]
(axis cs:8.87755102040816,4.90945736573485)
--(axis cs:8.87755102040816,4.90945757405097);

\path [draw=red]
(axis cs:9.18367346938776,4.91782318736881)
--(axis cs:9.18367346938776,4.91782342176285);

\path [draw=red]
(axis cs:9.48979591836735,4.92794768095605)
--(axis cs:9.48979591836735,4.92794790073016);

\path [draw=red]
(axis cs:9.79591836734694,4.88518102165156)
--(axis cs:9.79591836734694,4.88518120280904);

\path [draw=red]
(axis cs:10.1020408163265,4.88050939665501)
--(axis cs:10.1020408163265,4.88050959548579);

\path [draw=red]
(axis cs:10.4081632653061,4.94376075825755)
--(axis cs:10.4081632653061,4.943760992462);

\path [draw=red]
(axis cs:10.7142857142857,4.91587275925201)
--(axis cs:10.7142857142857,4.91587297857354);

\path [draw=red]
(axis cs:11.0204081632653,4.9489314138837)
--(axis cs:11.0204081632653,4.94893164599222);

\path [draw=red]
(axis cs:11.3265306122449,4.95811664720943)
--(axis cs:11.3265306122449,4.95811684654098);

\path [draw=red]
(axis cs:11.6326530612245,4.91986112424865)
--(axis cs:11.6326530612245,4.91986130603382);

\path [draw=red]
(axis cs:11.9387755102041,4.96741426203882)
--(axis cs:11.9387755102041,4.96741445478082);

\path [draw=red]
(axis cs:12.2448979591837,4.91277428551195)
--(axis cs:12.2448979591837,4.91277449708021);

\path [draw=red]
(axis cs:12.5510204081633,4.94045656488701)
--(axis cs:12.5510204081633,4.94045678862863);

\path [draw=red]
(axis cs:12.8571428571429,4.91763906896349)
--(axis cs:12.8571428571429,4.91763924947093);

\path [draw=red]
(axis cs:13.1632653061224,4.92082763361799)
--(axis cs:13.1632653061224,4.92082782566898);

\path [draw=red]
(axis cs:13.469387755102,4.95554204262433)
--(axis cs:13.469387755102,4.95554224957266);

\path [draw=red]
(axis cs:13.7755102040816,4.93291698386802)
--(axis cs:13.7755102040816,4.93291719281253);

\path [draw=red]
(axis cs:14.0816326530612,4.93312406522306)
--(axis cs:14.0816326530612,4.93312427708727);

\path [draw=red]
(axis cs:14.3877551020408,4.92917622356019)
--(axis cs:14.3877551020408,4.92917641271576);

\path [draw=red]
(axis cs:14.6938775510204,4.95587292213729)
--(axis cs:14.6938775510204,4.95587313389125);

\path [draw=red]
(axis cs:15,4.97048679020895)
--(axis cs:15,4.97048698855052);

\addplot [thick, sienna1408675, solid, line width=1.3pt]
table {%
0 5.06647467666484
0.306122448979592 5.10278964975031
0.612244897959184 5.14560812123766
0.918367346938776 5.18305912161837
1.22448979591837 5.23948204723228
1.53061224489796 5.2915421043858
1.83673469387755 5.34816292765925
2.14285714285714 5.40570828023512
2.44897959183673 5.47535559814974
2.75510204081633 5.56085262212359
3.06122448979592 5.665009859769
3.36734693877551 5.7455549447
3.6734693877551 5.87265764708556
3.97959183673469 6.00024221612281
4.28571428571429 6.14291461146017
4.59183673469388 6.30893658299779
4.89795918367347 6.49186775645032
5.20408163265306 6.69900501567406
5.51020408163265 6.91198219592226
5.81632653061224 7.14660296706443
6.12244897959184 7.38911967933519
6.42857142857143 7.64620884498624
6.73469387755102 7.90698103022739
7.04081632653061 8.17371305066138
7.3469387755102 8.44182624043333
7.6530612244898 8.71461997945448
7.95918367346939 8.99897341857198
8.26530612244898 9.285708879819
8.57142857142857 9.56668152051986
8.87755102040816 9.85629373152639
9.18367346938776 10.1501260684038
9.48979591836735 10.439566051285
9.79591836734694 10.7303196954708
10.1020408163265 11.0241904308807
10.4081632653061 11.3215032789525
10.7142857142857 11.6117645651473
11.0204081632653 11.9036972107509
11.3265306122449 12.2120729779879
11.6326530612245 12.5094704013659
11.9387755102041 12.8065122959782
12.2448979591837 13.0944537267743
12.5510204081633 13.4001305564664
12.8571428571429 13.6833139563718
13.1632653061224 13.9926656420673
13.469387755102 14.2968610110194
13.7755102040816 14.5990628158262
14.0816326530612 14.8996947726108
14.3877551020408 15.2008587859052
14.6938775510204 15.5062009278818
15 15.8015310349504
};
\addlegendentry{$\hat \lambda_1$}
\addplot [thick, orchid227119194, solid, line width=1.3pt]
table {%
0 1.603181025163
0.306122448979592 1.6012926855471
0.612244897959184 1.60198123594557
0.918367346938776 1.6009147241492
1.22448979591837 1.60022335320157
1.53061224489796 1.60331421368191
1.83673469387755 1.61290788010497
2.14285714285714 1.71806429070862
2.44897959183673 1.9606605102745
2.75510204081633 2.16534642235687
3.06122448979592 2.2946734767666
3.36734693877551 2.38693727335485
3.6734693877551 2.45543887540264
3.97959183673469 2.49727911518369
4.28571428571429 2.48441265059826
4.59183673469388 2.44432102540538
4.89795918367347 2.4035012087649
5.20408163265306 2.52369220319297
5.51020408163265 2.68759261653053
5.81632653061224 2.83983262567699
6.12244897959184 2.98335081104084
6.42857142857143 3.09418847095677
6.73469387755102 3.18430413399509
7.04081632653061 3.26175769886034
7.3469387755102 3.34039161042436
7.6530612244898 3.40416414781952
7.95918367346939 3.46345136525489
8.26530612244898 3.5026295974057
8.57142857142857 3.54967161067059
8.87755102040816 3.59049189993698
9.18367346938776 3.62001778628431
9.48979591836735 3.64447197570698
9.79591836734694 3.67081164143591
10.1020408163265 3.68971131650415
10.4081632653061 3.71894394678334
10.7142857142857 3.73659190409448
11.0204081632653 3.75986789284657
11.3265306122449 3.77924794857568
11.6326530612245 3.7927801451125
11.9387755102041 3.80926538571121
12.2448979591837 3.82020468886263
12.5510204081633 3.83895885315377
12.8571428571429 3.84248093888553
13.1632653061224 3.85899687941379
13.469387755102 3.86527391599012
13.7755102040816 3.87682492960743
14.0816326530612 3.8865483155596
14.3877551020408 3.89455975400259
14.6938775510204 3.89739917173276
15 3.90998985752617
};
\addlegendentry{$\hat \lambda_2$}
\addplot [crimson2143940, mark=o, mark size=3, mark options={solid}]
table {%
0 0.221841889061649
0.306122448979592 0
0.612244897959184 0
0.918367346938776 0.224431264763196
1.22448979591837 0
1.53061224489796 0.116121256812657
1.83673469387755 0.210377648856197
2.14285714285714 1.77047364426839
2.44897959183673 4.11041426716983
2.75510204081633 4.99466306706772
3.06122448979592 5.016509954031
3.36734693877551 4.96343813852761
3.6734693877551 5.03245938669402
3.97959183673469 5.02546634390738
4.28571428571429 5.08964769526005
4.59183673469388 5.44431865478478
4.89795918367347 7.38181269881846
5.20408163265306 6.67747294248556
5.51020408163265 5.6921784801141
5.81632653061224 5.80540652526052
6.12244897959184 6.16800213949487
6.42857142857143 6.48187340121564
6.73469387755102 6.77699326383581
7.04081632653061 7.12190243075661
7.3469387755102 7.58230778633691
7.6530612244898 7.93255082334523
7.95918367346939 7.9712573920576
8.26530612244898 8.29145004102009
8.57142857142857 8.59082530849303
8.87755102040816 8.90577302182938
9.18367346938776 9.21138161090621
9.48979591836735 9.50137840661129
9.79591836734694 9.83127767996091
10.1020408163265 10.1383716742558
10.4081632653061 10.4021003513381
10.7142857142857 10.7346190699346
11.0204081632653 11.023988901791
11.3265306122449 11.3399070304897
11.6326530612245 11.6602204346976
11.9387755102041 11.9387278117584
12.2448979591837 12.2569304373338
12.5510204081633 12.56463555734
12.8571428571429 12.8650099549729
13.1632653061224 13.1884313922267
13.469387755102 13.4774866218635
13.7755102040816 13.8053723448707
14.0816326530612 14.10284635531
14.3877551020408 14.4082846899457
14.6938775510204 14.7036743075948
15 14.9909299805604
};
\addlegendentry{$\min\{\hat \beta_1, \hat \beta_2\}$}
\addplot [mediumpurple148103189, mark=asterisk, mark size=3, mark options={solid}]
table {%
0 0.0521730392337873
0.306122448979592 0
0.612244897959184 0
0.918367346938776 0.0521374535810914
1.22448979591837 0
1.53061224489796 0.0263147449817767
1.83673469387755 0.0962860436798858
2.14285714285714 0.74819676884351
2.44897959183673 1.94556582512564
2.75510204081633 2.68516315259374
3.06122448979592 2.97542889286977
3.36734693877551 3.34917253815118
3.6734693877551 3.5272268938216
3.97959183673469 3.87315169172042
4.28571428571429 4.09205077795176
4.59183673469388 4.41113121669527
4.89795918367347 3.79800820299248
5.20408163265306 4.74601996952161
5.51020408163265 4.63526056222096
5.81632653061224 5.03950919798197
6.12244897959184 4.87598507873665
6.42857142857143 4.85492439248319
6.73469387755102 4.86680356475687
7.04081632653061 4.97911415788997
7.3469387755102 4.98965337306719
7.6530612244898 5.07207069743842
7.95918367346939 4.93435714666704
8.26530612244898 4.92327719550656
8.57142857142857 4.93409313426073
8.87755102040816 4.90945746989291
9.18367346938776 4.91782330456583
9.48979591836735 4.9279477908431
9.79591836734694 4.8851811122303
10.1020408163265 4.8805094960704
10.4081632653061 4.94376087535978
10.7142857142857 4.91587286891277
11.0204081632653 4.94893152993796
11.3265306122449 4.9581167468752
11.6326530612245 4.91986121514123
11.9387755102041 4.96741435840982
12.2448979591837 4.91277439129608
12.5510204081633 4.94045667675782
12.8571428571429 4.91763915921721
13.1632653061224 4.92082772964349
13.469387755102 4.9555421460985
13.7755102040816 4.93291708834027
14.0816326530612 4.93312417115517
14.3877551020408 4.92917631813798
14.6938775510204 4.95587302801427
15 4.97048688937973
};
\addlegendentry{$\min\{\hat \beta_1, \hat \beta_2\}$}
\end{axis}

\end{tikzpicture}

%% file: figs/alignments_estimation_gamma_1.tex
\begin{tikzpicture}

\definecolor{crimson2143940}{RGB}{214,39,40}
\definecolor{darkorange25512714}{RGB}{255,127,14}
\definecolor{darkslategray38}{RGB}{38,38,38}
\definecolor{forestgreen4416044}{RGB}{44,160,44}
\definecolor{lightgray204}{RGB}{204,204,204}
\definecolor{steelblue31119180}{RGB}{31,119,180}

\begin{groupplot}[group style={group size=2 by 2}]
\nextgroupplot[
axis line style={lightgray204},
legend cell align={left},
legend style={
  fill opacity=0.8,
  draw opacity=1,
  text opacity=1,
  at={(0.03,0.97)},
  anchor=north west,
  draw=lightgray204
},
width=0.5\textwidth,
height=0.35\textwidth,
tick align=outside,
x grid style={lightgray204},
xlabel=\textcolor{darkslategray38}{\(\displaystyle \alpha\)},
xmajorgrids,
xmajorticks=true,
xmin=0.065, xmax=0.835,
xtick style={color=darkslategray38},
y grid style={lightgray204},
ymajorgrids,
ymajorticks=true,
ymin=-0.091981197516902, ymax=0.972904750289262,
ytick style={color=darkslategray38}
]
\addplot [semithick, steelblue31119180, smooth, line width=1.3pt]
table {%
0.1 0.10027153134992
0.136842105263158 0.139745977970676
0.173684210526316 0.179409167951517
0.210526315789474 0.2139030902167
0.247368421052632 0.253786202644498
0.284210526315789 0.288314418310631
0.321052631578947 0.322988660831939
0.357894736842105 0.35774383327909
0.394736842105263 0.386209004525722
0.431578947368421 0.414303198387906
0.468421052631579 0.446730494000377
0.505263157894737 0.46858168869093
0.542105263157895 0.493308301860547
0.578947368421053 0.516903191530471
0.615789473684211 0.52994774073487
0.652631578947368 0.544838710009435
0.689473684210526 0.56159246731926
0.726315789473684 0.564216122453865
0.763157894736842 0.566641267104785
0.8 0.513553935603982
};
\addlegendentry{$\hat \eta$}
\path [draw=darkorange25512714, semithick]
(axis cs:0.1,-0.04357729079844)
--(axis cs:0.1,0.310784723008754);

\path [draw=darkorange25512714, semithick]
(axis cs:0.136842105263158,0.00227225977191714)
--(axis cs:0.136842105263158,0.371325877946248);

\path [draw=darkorange25512714, semithick]
(axis cs:0.173684210526316,0.0619428400453502)
--(axis cs:0.173684210526316,0.349293244251841);

\path [draw=darkorange25512714, semithick]
(axis cs:0.210526315789474,0.187153888047059)
--(axis cs:0.210526315789474,0.235880398898851);

\path [draw=darkorange25512714, semithick]
(axis cs:0.247368421052632,0.235761046297957)
--(axis cs:0.247368421052632,0.263941350721692);

\path [draw=darkorange25512714, semithick]
(axis cs:0.284210526315789,0.20732144121018)
--(axis cs:0.284210526315789,0.419981535772949);

\path [draw=darkorange25512714, semithick]
(axis cs:0.321052631578947,0.232876499920779)
--(axis cs:0.321052631578947,0.476625447365601);

\path [draw=darkorange25512714, semithick]
(axis cs:0.357894736842105,0.34710184381578)
--(axis cs:0.357894736842105,0.373810500902193);

\path [draw=darkorange25512714, semithick]
(axis cs:0.394736842105263,0.379004826480767)
--(axis cs:0.394736842105263,0.409003799580632);

\path [draw=darkorange25512714, semithick]
(axis cs:0.431578947368421,0.41181806425702)
--(axis cs:0.431578947368421,0.446175915369295);

\path [draw=darkorange25512714, semithick]
(axis cs:0.468421052631579,0.454634243446998)
--(axis cs:0.468421052631579,0.487743888120245);

\path [draw=darkorange25512714, semithick]
(axis cs:0.505263157894737,0.484346723742315)
--(axis cs:0.505263157894737,0.521708434965205);

\path [draw=darkorange25512714, semithick]
(axis cs:0.542105263157895,0.520526678102836)
--(axis cs:0.542105263157895,0.562961519571422);

\path [draw=darkorange25512714, semithick]
(axis cs:0.578947368421053,0.557372425218972)
--(axis cs:0.578947368421053,0.604166495426742);

\path [draw=darkorange25512714, semithick]
(axis cs:0.615789473684211,0.586666046222413)
--(axis cs:0.615789473684211,0.63375788600034);

\path [draw=darkorange25512714, semithick]
(axis cs:0.652631578947368,0.611069047360322)
--(axis cs:0.652631578947368,0.676899924846815);

\path [draw=darkorange25512714, semithick]
(axis cs:0.689473684210526,0.640494456609239)
--(axis cs:0.689473684210526,0.723809409741505);

\path [draw=darkorange25512714, semithick]
(axis cs:0.726315789473684,0.655048170927191)
--(axis cs:0.726315789473684,0.759967484395606);

\path [draw=darkorange25512714, semithick]
(axis cs:0.763157894736842,0.67364055101179)
--(axis cs:0.763157894736842,0.79817776580114);

\path [draw=darkorange25512714, semithick]
(axis cs:0.8,0.459282877828107)
--(axis cs:0.8,0.9245008435708);

\addplot [semithick, forestgreen4416044, smooth, line width=1.3pt]
table {%
0.1 0.1
0.136842105263158 0.136842105263158
0.173684210526316 0.173684210526316
0.210526315789474 0.210526315789474
0.247368421052632 0.247368421052632
0.284210526315789 0.284210526315789
0.321052631578947 0.321052631578947
0.357894736842105 0.357894736842105
0.394736842105263 0.394736842105263
0.431578947368421 0.431578947368421
0.468421052631579 0.468421052631579
0.505263157894737 0.505263157894737
0.542105263157895 0.542105263157895
0.578947368421053 0.578947368421053
0.615789473684211 0.615789473684211
0.652631578947368 0.652631578947368
0.689473684210526 0.689473684210526
0.726315789473684 0.726315789473684
0.763157894736842 0.763157894736842
0.8 0.8
};
\addlegendentry{$y=x$}
\addplot [semithick, darkorange25512714, smooth, line width=1.3pt]
table {%
0.1 0.133603716105157
0.136842105263158 0.186799068859082
0.173684210526316 0.205618042148595
0.210526315789474 0.211517143472955
0.247368421052632 0.249851198509825
0.284210526315789 0.313651488491565
0.321052631578947 0.35475097364319
0.357894736842105 0.360456172358986
0.394736842105263 0.3940043130307
0.431578947368421 0.428996989813157
0.468421052631579 0.471189065783622
0.505263157894737 0.50302757935376
0.542105263157895 0.541744098837129
0.578947368421053 0.580769460322857
0.615789473684211 0.610211966111376
0.652631578947368 0.643984486103569
0.689473684210526 0.682151933175372
0.726315789473684 0.707507827661399
0.763157894736842 0.735909158406465
0.8 0.691891860699454
};
\addlegendentry{$\tilde \alpha$}

\nextgroupplot[
axis line style={lightgray204},
legend cell align={left},
legend style={
  fill opacity=0.8,
  draw opacity=1,
  text opacity=1,
  at={(0.97,0.03)},
  anchor=south east,
  draw=lightgray204
},
width=0.5\textwidth,
height=0.35\textwidth,
tick align=outside,
x grid style={lightgray204},
xlabel=\textcolor{darkslategray38}{\(\displaystyle \alpha\)},
xmajorgrids,
xmajorticks=true,
xmin=0.065, xmax=0.835,
xtick style={color=darkslategray38},
y grid style={lightgray204},
ymajorgrids,
ymajorticks=true,
ymin=-0.107472090665145, ymax=1.24561669482423,
ytick style={color=darkslategray38}
]
\addplot [semithick, steelblue31119180, smooth, line width=1.3pt]
table {%
0.1 0.999252409976537
0.136842105263158 0.999247548411571
0.173684210526316 0.999222813769688
0.210526315789474 0.999129599163843
0.247368421052632 0.999009408484409
0.284210526315789 0.998795074842407
0.321052631578947 0.998482317569235
0.357894736842105 0.998065284091785
0.394736842105263 0.997502076683936
0.431578947368421 0.996794014529048
0.468421052631579 0.995946488192129
0.505263157894737 0.995122808933782
0.542105263157895 0.99414186464854
0.578947368421053 0.993171222388146
0.615789473684211 0.992337396290942
0.652631578947368 0.991750034931563
0.689473684210526 0.991268246475382
0.726315789473684 0.991059923561809
0.763157894736842 0.991141743687298
0.8 0.991543691601075
};
\addlegendentry{$\hat\rho_{11}$}
\path [draw=darkorange25512714, semithick]
(axis cs:0.1,0.824882841637159)
--(axis cs:0.1,1.17924485544435);

\path [draw=darkorange25512714, semithick]
(axis cs:0.136842105263158,0.815059040945838)
--(axis cs:0.136842105263158,1.18411265912017);

\path [draw=darkorange25512714, semithick]
(axis cs:0.173684210526316,0.854095003444597)
--(axis cs:0.173684210526316,1.14144540765109);

\path [draw=darkorange25512714, semithick]
(axis cs:0.210526315789474,0.974638311706498)
--(axis cs:0.210526315789474,1.02336482255829);

\path [draw=darkorange25512714, semithick]
(axis cs:0.247368421052632,0.984896675682424)
--(axis cs:0.247368421052632,1.01307698010616);

\path [draw=darkorange25512714, semithick]
(axis cs:0.284210526315789,0.878183919202442)
--(axis cs:0.284210526315789,1.09084401376521);

\path [draw=darkorange25512714, semithick]
(axis cs:0.321052631578947,0.831342689231174)
--(axis cs:0.321052631578947,1.075091636676);

\path [draw=darkorange25512714, semithick]
(axis cs:0.357894736842105,0.984674470484511)
--(axis cs:0.357894736842105,1.01138312757092);

\path [draw=darkorange25512714, semithick]
(axis cs:0.394736842105263,0.982535328896434)
--(axis cs:0.394736842105263,1.0125343019963);

\path [draw=darkorange25512714, semithick]
(axis cs:0.431578947368421,0.979727008213314)
--(axis cs:0.431578947368421,1.01408485932559);

\path [draw=darkorange25512714, semithick]
(axis cs:0.468421052631579,0.979369643588744)
--(axis cs:0.468421052631579,1.01247928826199);

\path [draw=darkorange25512714, semithick]
(axis cs:0.505263157894737,0.976499020345939)
--(axis cs:0.505263157894737,1.01386073156883);

\path [draw=darkorange25512714, semithick]
(axis cs:0.542105263157895,0.972966812112528)
--(axis cs:0.542105263157895,1.01540165358111);

\path [draw=darkorange25512714, semithick]
(axis cs:0.578947368421053,0.969588832236118)
--(axis cs:0.578947368421053,1.01638290244389);

\path [draw=darkorange25512714, semithick]
(axis cs:0.615789473684211,0.969164717543579)
--(axis cs:0.615789473684211,1.01625655732151);

\path [draw=darkorange25512714, semithick]
(axis cs:0.652631578947368,0.95897451890562)
--(axis cs:0.652631578947368,1.02480539639211);

\path [draw=darkorange25512714, semithick]
(axis cs:0.689473684210526,0.948767543064232)
--(axis cs:0.689473684210526,1.0320824961965);

\path [draw=darkorange25512714, semithick]
(axis cs:0.726315789473684,0.937942623025718)
--(axis cs:0.726315789473684,1.04286193649413);

\path [draw=darkorange25512714, semithick]
(axis cs:0.763157894736842,0.927725540978699)
--(axis cs:0.763157894736842,1.05226275576805);

\path [draw=darkorange25512714, semithick]
(axis cs:0.8,0.6681069998593)
--(axis cs:0.8,1.13332496560199);

\addplot [semithick, forestgreen4416044, smooth, line width=1.3pt]
table {%
0.1 0.104308335800437
0.136842105263158 0.143884384096891
0.173684210526316 0.18374536521917
0.210526315789474 0.226456080044226
0.247368421052632 0.269037756742441
0.284210526315789 0.313278998900984
0.321052631578947 0.358429672937184
0.357894736842105 0.402759684676888
0.394736842105263 0.448567012985099
0.431578947368421 0.493730759917255
0.468421052631579 0.538493352604931
0.505263157894737 0.581499212488217
0.542105263157895 0.624060654924658
0.578947368421053 0.665142495208459
0.615789473684211 0.703923399705656
0.652631578947368 0.740382995140746
0.689473684210526 0.7753391642961
0.726315789473684 0.808167899367985
0.763157894736842 0.839189052382974
0.8 0.868209961956073
};
\addlegendentry{$\hat\rho_{12}$}
\path [draw=crimson2143940, semithick]
(axis cs:0.1,-0.0459680549610828)
--(axis cs:0.1,0.308393958846111);

\path [draw=crimson2143940, semithick]
(axis cs:0.136842105263158,0.00567863572192298)
--(axis cs:0.136842105263158,0.374732253896253);

\path [draw=crimson2143940, semithick]
(axis cs:0.173684210526316,0.0692065692771836)
--(axis cs:0.173684210526316,0.356556973483674);

\path [draw=crimson2143940, semithick]
(axis cs:0.210526315789474,0.202112132389633)
--(axis cs:0.210526315789474,0.250838643241425);

\path [draw=crimson2143940, semithick]
(axis cs:0.247368421052632,0.257740055476167)
--(axis cs:0.247368421052632,0.285920359899902);

\path [draw=crimson2143940, semithick]
(axis cs:0.284210526315789,0.218129607327118)
--(axis cs:0.284210526315789,0.430789701889887);

\path [draw=crimson2143940, semithick]
(axis cs:0.321052631578947,0.231474278608319)
--(axis cs:0.321052631578947,0.475223226053141);

\path [draw=crimson2143940, semithick]
(axis cs:0.357894736842105,0.392879620355579)
--(axis cs:0.357894736842105,0.419588277441992);

\path [draw=crimson2143940, semithick]
(axis cs:0.394736842105263,0.432208460076495)
--(axis cs:0.394736842105263,0.462207433176359);

\path [draw=crimson2143940, semithick]
(axis cs:0.431578947368421,0.472602142976399)
--(axis cs:0.431578947368421,0.506959994088673);

\path [draw=crimson2143940, semithick]
(axis cs:0.468421052631579,0.524867332249224)
--(axis cs:0.468421052631579,0.557976976922472);

\path [draw=crimson2143940, semithick]
(axis cs:0.505263157894737,0.559968341756171)
--(axis cs:0.505263157894737,0.597330052979061);

\path [draw=crimson2143940, semithick]
(axis cs:0.542105263157895,0.601784344014593)
--(axis cs:0.542105263157895,0.644219185483179);

\path [draw=crimson2143940, semithick]
(axis cs:0.578947368421053,0.643901911915877)
--(axis cs:0.578947368421053,0.690695982123647);

\path [draw=crimson2143940, semithick]
(axis cs:0.615789473684211,0.672261524508521)
--(axis cs:0.615789473684211,0.719353364286448);

\path [draw=crimson2143940, semithick]
(axis cs:0.652631578947368,0.696429263998875)
--(axis cs:0.652631578947368,0.762260141485368);

\path [draw=crimson2143940, semithick]
(axis cs:0.689473684210526,0.726433662041273)
--(axis cs:0.689473684210526,0.809748615173538);

\path [draw=crimson2143940, semithick]
(axis cs:0.726315789473684,0.73406781364523)
--(axis cs:0.726315789473684,0.838987127113644);

\path [draw=crimson2143940, semithick]
(axis cs:0.763157894736842,0.746852158089598)
--(axis cs:0.763157894736842,0.871389372878947);

\path [draw=crimson2143940, semithick]
(axis cs:0.8,0.518651340233639)
--(axis cs:0.8,0.983869305976331);

\addplot [semithick, darkorange25512714, smooth, line width=1.3pt]
table {%
0.1 1.00206384854076
0.136842105263158 0.999585850033004
0.173684210526316 0.997770205547842
0.210526315789474 0.999001567132394
0.247368421052632 0.998986827894291
0.284210526315789 0.984513966483826
0.321052631578947 0.953217162953585
0.357894736842105 0.998028799027718
0.394736842105263 0.997534815446366
0.431578947368421 0.996905933769452
0.468421052631579 0.995924465925368
0.505263157894737 0.995179875957383
0.542105263157895 0.994184232846821
0.578947368421053 0.992985867340004
0.615789473684211 0.992710637432543
0.652631578947368 0.991889957648866
0.689473684210526 0.990425019630365
0.726315789473684 0.990402279759925
0.763157894736842 0.989994148373373
0.8 0.900715982730646
};
\addlegendentry{$\tilde\rho_{12}$}
\addplot [semithick, crimson2143940, smooth, line width=1.3pt]
table {%
0.1 0.131212951942514
0.136842105263158 0.190205444809088
0.173684210526316 0.212881771380429
0.210526315789474 0.226475387815529
0.247368421052632 0.271830207688034
0.284210526315789 0.324459654608503
0.321052631578947 0.35334875233073
0.357894736842105 0.406233948898786
0.394736842105263 0.447207946626427
0.431578947368421 0.489781068532536
0.468421052631579 0.541422154585848
0.505263157894737 0.578649197367616
0.542105263157895 0.623001764748886
0.578947368421053 0.667298947019762
0.615789473684211 0.695807444397484
0.652631578947368 0.729344702742122
0.689473684210526 0.768091138607406
0.726315789473684 0.786527470379437
0.763157894736842 0.809120765484273
0.8 0.751260323104985
};
\addlegendentry{$\tilde\rho_{12}$}

\nextgroupplot[
axis line style={lightgray204},
legend cell align={left},
legend style={
  fill opacity=0.8,
  draw opacity=1,
  text opacity=1,
  at={(0.09,0.5)},
  anchor=west,
  draw=lightgray204
},
width=0.5\textwidth,
height=0.35\textwidth,
tick align=outside,
x grid style={lightgray204},
xlabel=\textcolor{darkslategray38}{\(\displaystyle \alpha\)},
xmajorgrids,
xmajorticks=true,
xmin=0.065, xmax=0.835,
xtick style={color=darkslategray38},
y grid style={lightgray204},
ymajorgrids,
ymajorticks=true,
ymin=-0.231543859051179, ymax=1.19391955758113,
ytick style={color=darkslategray38}
]
\addplot [semithick, steelblue31119180, smooth, line width=1.3pt]
table {%
0.1 0.00491383249295881
0.136842105263158 0.00700139688355016
0.173684210526316 0.0101379310748254
0.210526315789474 0.0161800211491053
0.247368421052632 0.0222448715012218
0.284210526315789 0.0302332541501851
0.321052631578947 0.0393917700271265
0.357894736842105 0.0481338725354586
0.394736842105263 0.0589607473954401
0.431578947368421 0.069722394305932
0.468421052631579 0.0805120684519464
0.505263157894737 0.0900655795151716
0.542105263157895 0.100100983754206
0.578947368421053 0.108885207119868
0.615789473684211 0.115943074116625
0.652631578947368 0.12033720131295
0.689473684210526 0.123685768211242
0.726315789473684 0.124151501443299
0.763157894736842 0.122180108015839
0.8 0.10747676874026
};
\addlegendentry{$\hat\theta_{21}$}
\path [draw=darkorange25512714, semithick]
(axis cs:0.1,-0.166750067386074)
--(axis cs:0.1,0.18761194642112);

\path [draw=darkorange25512714, semithick]
(axis cs:0.136842105263158,-0.160639671815719)
--(axis cs:0.136842105263158,0.208413946358612);

\path [draw=darkorange25512714, semithick]
(axis cs:0.173684210526316,-0.119186136746029)
--(axis cs:0.173684210526316,0.168164267460461);

\path [draw=darkorange25512714, semithick]
(axis cs:0.210526315789474,-0.00682059124160768)
--(axis cs:0.210526315789474,0.0419059196101842);

\path [draw=darkorange25512714, semithick]
(axis cs:0.247368421052632,0.0088283060382285)
--(axis cs:0.247368421052632,0.0370086104619629);

\path [draw=darkorange25512714, semithick]
(axis cs:0.284210526315789,-0.0668902537815321)
--(axis cs:0.284210526315789,0.145769840781237);

\path [draw=darkorange25512714, semithick]
(axis cs:0.321052631578947,-0.0856956935464371)
--(axis cs:0.321052631578947,0.158053253898385);

\path [draw=darkorange25512714, semithick]
(axis cs:0.357894736842105,0.0356322203023958)
--(axis cs:0.357894736842105,0.0623408773888084);

\path [draw=darkorange25512714, semithick]
(axis cs:0.394736842105263,0.0429722338387728)
--(axis cs:0.394736842105263,0.0729712069386372);

\path [draw=darkorange25512714, semithick]
(axis cs:0.431578947368421,0.0504754235862782)
--(axis cs:0.431578947368421,0.0848332746985526);

\path [draw=darkorange25512714, semithick]
(axis cs:0.468421052631579,0.0639255812458599)
--(axis cs:0.468421052631579,0.0970352259191074);

\path [draw=darkorange25512714, semithick]
(axis cs:0.505263157894737,0.0701408973896371)
--(axis cs:0.505263157894737,0.107502608612527);

\path [draw=darkorange25512714, semithick]
(axis cs:0.542105263157895,0.0775067751837484)
--(axis cs:0.542105263157895,0.119941616652334);

\path [draw=darkorange25512714, semithick]
(axis cs:0.578947368421053,0.0859366301476863)
--(axis cs:0.578947368421053,0.132730700355457);

\path [draw=darkorange25512714, semithick]
(axis cs:0.615789473684211,0.0878371715906571)
--(axis cs:0.615789473684211,0.134929011368584);

\path [draw=darkorange25512714, semithick]
(axis cs:0.652631578947368,0.0833214457675564)
--(axis cs:0.652631578947368,0.149152323254049);

\path [draw=darkorange25512714, semithick]
(axis cs:0.689473684210526,0.0827983861160361)
--(axis cs:0.689473684210526,0.166113339248302);

\path [draw=darkorange25512714, semithick]
(axis cs:0.726315789473684,0.0677908154129787)
--(axis cs:0.726315789473684,0.172710128881393);

\path [draw=darkorange25512714, semithick]
(axis cs:0.763157894736842,0.0557495855174046)
--(axis cs:0.763157894736842,0.180286800306754);

\path [draw=darkorange25512714, semithick]
(axis cs:0.8,-0.13178935690299)
--(axis cs:0.8,0.333428608839703);

\addplot [semithick, forestgreen4416044, smooth, line width=1.3pt]
table {%
0.1 0.983233617942522
0.136842105263158 0.976201072692608
0.173684210526316 0.970767733810125
0.210526315789474 0.962437271997007
0.247368421052632 0.950694872148222
0.284210526315789 0.93765949825717
0.321052631578947 0.922385258232744
0.357894736842105 0.902709184470267
0.394736842105263 0.881670267355331
0.431578947368421 0.855277946926053
0.468421052631579 0.829306534759345
0.505263157894737 0.799404491667667
0.542105263157895 0.765749188613071
0.578947368421053 0.731227081914361
0.615789473684211 0.693108916267805
0.652631578947368 0.653692594385614
0.689473684210526 0.610793809019115
0.726315789473684 0.565325722020385
0.763157894736842 0.513722959659877
0.8 0.421262470699916
};
\addlegendentry{$\hat\theta_{22}$}
\path [draw=crimson2143940, semithick]
(axis cs:0.1,0.774763752108832)
--(axis cs:0.1,1.12912576591603);

\path [draw=crimson2143940, semithick]
(axis cs:0.136842105263158,0.738504703840645)
--(axis cs:0.136842105263158,1.10755832201498);

\path [draw=crimson2143940, semithick]
(axis cs:0.173684210526316,0.775716883855496)
--(axis cs:0.173684210526316,1.06306728806199);

\path [draw=crimson2143940, semithick]
(axis cs:0.210526315789474,0.935476018899144)
--(axis cs:0.210526315789474,0.984202529750936);

\path [draw=crimson2143940, semithick]
(axis cs:0.247368421052632,0.940713673987314)
--(axis cs:0.247368421052632,0.968893978411049);

\path [draw=crimson2143940, semithick]
(axis cs:0.284210526315789,0.785120685109118)
--(axis cs:0.284210526315789,0.997780779671886);

\path [draw=crimson2143940, semithick]
(axis cs:0.321052631578947,0.740541396884935)
--(axis cs:0.321052631578947,0.984290344329758);

\path [draw=crimson2143940, semithick]
(axis cs:0.357894736842105,0.892154687647567)
--(axis cs:0.357894736842105,0.918863344733979);

\path [draw=crimson2143940, semithick]
(axis cs:0.394736842105263,0.870700484865538)
--(axis cs:0.394736842105263,0.900699457965402);

\path [draw=crimson2143940, semithick]
(axis cs:0.431578947368421,0.845316961511668)
--(axis cs:0.431578947368421,0.879674812623943);

\path [draw=crimson2143940, semithick]
(axis cs:0.468421052631579,0.814223406605661)
--(axis cs:0.468421052631579,0.847333051278908);

\path [draw=crimson2143940, semithick]
(axis cs:0.505263157894737,0.785984622691087)
--(axis cs:0.505263157894737,0.823346333913976);

\path [draw=crimson2143940, semithick]
(axis cs:0.542105263157895,0.748711488229567)
--(axis cs:0.542105263157895,0.791146329698153);

\path [draw=crimson2143940, semithick]
(axis cs:0.578947368421053,0.707590426972178)
--(axis cs:0.578947368421053,0.754384497179949);

\path [draw=crimson2143940, semithick]
(axis cs:0.615789473684211,0.678884301040658)
--(axis cs:0.615789473684211,0.725976140818585);

\path [draw=crimson2143940, semithick]
(axis cs:0.652631578947368,0.631255738632242)
--(axis cs:0.652631578947368,0.697086616118735);

\path [draw=crimson2143940, semithick]
(axis cs:0.689473684210526,0.573019366580194)
--(axis cs:0.689473684210526,0.65633431971246);

\path [draw=crimson2143940, semithick]
(axis cs:0.726315789473684,0.52926528419667)
--(axis cs:0.726315789473684,0.634184597665084);

\path [draw=crimson2143940, semithick]
(axis cs:0.763157894736842,0.476737769685132)
--(axis cs:0.763157894736842,0.601274984474482);

\path [draw=crimson2143940, semithick]
(axis cs:0.8,0.208571512137289)
--(axis cs:0.8,0.673789477879981);

\addplot [semithick, darkorange25512714, smooth, line width=1.3pt]
table {%
0.1 0.0104309395175231
0.136842105263158 0.0238871372714464
0.173684210526316 0.0244890653572157
0.210526315789474 0.0175426641842883
0.247368421052632 0.0229184582500957
0.284210526315789 0.0394397934998524
0.321052631578947 0.036178780175974
0.357894736842105 0.0489865488456021
0.394736842105263 0.057971720388705
0.431578947368421 0.0676543491424154
0.468421052631579 0.0804804035824837
0.505263157894737 0.0888217530010819
0.542105263157895 0.0987241959180414
0.578947368421053 0.109333665251572
0.615789473684211 0.111383091479621
0.652631578947368 0.116236884510803
0.689473684210526 0.124455862682169
0.726315789473684 0.120250472147186
0.763157894736842 0.118018192912079
0.8 0.100819625968356
};
\addlegendentry{$\tilde\theta_{22}$}
\addplot [semithick, crimson2143940, smooth, line width=1.3pt]
table {%
0.1 0.951944759012429
0.136842105263158 0.923031512927811
0.173684210526316 0.919392085958741
0.210526315789474 0.95983927432504
0.247368421052632 0.954803826199182
0.284210526315789 0.891450732390502
0.321052631578947 0.862415870607347
0.357894736842105 0.905509016190773
0.394736842105263 0.88569997141547
0.431578947368421 0.862495887067805
0.468421052631579 0.830778228942284
0.505263157894737 0.804665478302531
0.542105263157895 0.76992890896386
0.578947368421053 0.730987462076064
0.615789473684211 0.702430220929621
0.652631578947368 0.664171177375489
0.689473684210526 0.614676843146327
0.726315789473684 0.581724940930877
0.763157894736842 0.539006377079807
0.8 0.441180495008635
};
\addlegendentry{$\tilde\theta_{22}$}

\nextgroupplot[
axis line style={lightgray204},
legend cell align={left},
legend style={
  fill opacity=0.8,
  draw opacity=1,
  text opacity=1,
  at={(0.09,0.5)},
  anchor=west,
  draw=lightgray204
},
width=0.5\textwidth,
height=0.35\textwidth,
tick align=outside,
x grid style={lightgray204},
xlabel=\textcolor{darkslategray38}{\(\displaystyle \alpha\)},
xmajorgrids,
xmajorticks=true,
xmin=0.065, xmax=0.835,
xtick style={color=darkslategray38},
y grid style={lightgray204},
ymajorgrids,
ymajorticks=true,
ymin=-0.122741632112168, ymax=1.24696931898863,
ytick style={color=darkslategray38}
]
\addplot [semithick, steelblue31119180, smooth, line width=1.3pt]
table {%
0.1 0.0969177679326179
0.136842105263158 0.133669494458943
0.173684210526316 0.168725510333743
0.210526315789474 0.19756200156338
0.247368421052632 0.231912548050361
0.284210526315789 0.25854472109996
0.321052631578947 0.285122808185985
0.357894736842105 0.311895347442129
0.394736842105263 0.331073751648401
0.431578947368421 0.349911165185676
0.468421052631579 0.373456346550088
0.505263157894737 0.386699150360483
0.542105263157895 0.404102842434072
0.578947368421053 0.421513695980498
0.615789473684211 0.429017809908892
0.652631578947368 0.440938446867181
0.689473684210526 0.456304397220574
0.726315789473684 0.459047122019642
0.763157894736842 0.464749084141395
0.8 0.426606224278178
};
\addlegendentry{$\hat\rho_{21}$}
\path [draw=darkorange25512714, semithick]
(axis cs:0.1,-0.0604820434257683)
--(axis cs:0.1,0.293879970381426);

\path [draw=darkorange25512714, semithick]
(axis cs:0.136842105263158,-0.0243273685163588)
--(axis cs:0.136842105263158,0.344726249657972);

\path [draw=darkorange25512714, semithick]
(axis cs:0.173684210526316,0.0389057291291293)
--(axis cs:0.173684210526316,0.32625613333562);

\path [draw=darkorange25512714, semithick]
(axis cs:0.210526315789474,0.174231790767706)
--(axis cs:0.210526315789474,0.222958301619497);

\path [draw=darkorange25512714, semithick]
(axis cs:0.247368421052632,0.217774565367405)
--(axis cs:0.247368421052632,0.24595486979114);

\path [draw=darkorange25512714, semithick]
(axis cs:0.284210526315789,0.169801947100295)
--(axis cs:0.284210526315789,0.382462041663064);

\path [draw=darkorange25512714, semithick]
(axis cs:0.321052631578947,0.199912531234142)
--(axis cs:0.321052631578947,0.443661478678964);

\path [draw=darkorange25512714, semithick]
(axis cs:0.357894736842105,0.298084641472291)
--(axis cs:0.357894736842105,0.324793298558703);

\path [draw=darkorange25512714, semithick]
(axis cs:0.394736842105263,0.31698259007987)
--(axis cs:0.394736842105263,0.346981563179734);

\path [draw=darkorange25512714, semithick]
(axis cs:0.431578947368421,0.334539642425948)
--(axis cs:0.431578947368421,0.368897493538222);

\path [draw=darkorange25512714, semithick]
(axis cs:0.468421052631579,0.356721748044897)
--(axis cs:0.468421052631579,0.389831392718145);

\path [draw=darkorange25512714, semithick]
(axis cs:0.505263157894737,0.369659442534954)
--(axis cs:0.505263157894737,0.407021153757844);

\path [draw=darkorange25512714, semithick]
(axis cs:0.542105263157895,0.384011035092163)
--(axis cs:0.542105263157895,0.426445876560749);

\path [draw=darkorange25512714, semithick]
(axis cs:0.578947368421053,0.397186890699225)
--(axis cs:0.578947368421053,0.443980960906995);

\path [draw=darkorange25512714, semithick]
(axis cs:0.615789473684211,0.409171227897344)
--(axis cs:0.615789473684211,0.456263067675271);

\path [draw=darkorange25512714, semithick]
(axis cs:0.652631578947368,0.411648510764609)
--(axis cs:0.652631578947368,0.477479388251102);

\path [draw=darkorange25512714, semithick]
(axis cs:0.689473684210526,0.414026187995512)
--(axis cs:0.689473684210526,0.497341141127777);

\path [draw=darkorange25512714, semithick]
(axis cs:0.726315789473684,0.410500243156127)
--(axis cs:0.726315789473684,0.515419556624542);

\path [draw=darkorange25512714, semithick]
(axis cs:0.763157894736842,0.404956966088036)
--(axis cs:0.763157894736842,0.529494180877386);

\path [draw=darkorange25512714, semithick]
(axis cs:0.8,0.187684865029668)
--(axis cs:0.8,0.65290283077236);

\addplot [semithick, forestgreen4416044, smooth, line width=1.3pt]
table {%
0.1 0.987084357976018
0.136842105263158 0.987934810419507
0.173684210526316 0.987929584560806
0.210526315789474 0.98724265166284
0.247368421052632 0.987609925703204
0.284210526315789 0.987562533679518
0.321052631578947 0.986053892258918
0.357894736842105 0.984932149929858
0.394736842105263 0.984293147769385
0.431578947368421 0.982252577101195
0.468421052631579 0.97946232756358
0.505263157894737 0.976213501994736
0.542105263157895 0.971440948767351
0.578947368421053 0.965177409312064
0.615789473684211 0.956045023705129
0.652631578947368 0.944617722811642
0.689473684210526 0.931235107331221
0.726315789473684 0.909781709377661
0.763157894736842 0.882525465710265
0.8 0.773344874437623
};
\addlegendentry{$\hat\rho_{22}$}
\path [draw=crimson2143940, semithick]
(axis cs:0.1,0.818931039868236)
--(axis cs:0.1,1.17329305367543);

\path [draw=crimson2143940, semithick]
(axis cs:0.136842105263158,0.815656112127897)
--(axis cs:0.136842105263158,1.18470973030223);

\path [draw=crimson2143940, semithick]
(axis cs:0.173684210526316,0.855054249956039)
--(axis cs:0.173684210526316,1.14240465416253);

\path [draw=crimson2143940, semithick]
(axis cs:0.210526315789474,0.970397472597623)
--(axis cs:0.210526315789474,1.01912398344942);

\path [draw=crimson2143940, semithick]
(axis cs:0.247368421052632,0.978107648209076)
--(axis cs:0.247368421052632,1.00628795263281);

\path [draw=crimson2143940, semithick]
(axis cs:0.284210526315789,0.885904209938745)
--(axis cs:0.284210526315789,1.09856430450151);

\path [draw=crimson2143940, semithick]
(axis cs:0.321052631578947,0.869295166297538)
--(axis cs:0.321052631578947,1.11304411374236);

\path [draw=crimson2143940, semithick]
(axis cs:0.357894736842105,0.976718594212588)
--(axis cs:0.357894736842105,1.003427251299);

\path [draw=crimson2143940, semithick]
(axis cs:0.394736842105263,0.973676540239967)
--(axis cs:0.394736842105263,1.00367551333983);

\path [draw=crimson2143940, semithick]
(axis cs:0.431578947368421,0.969498138111264)
--(axis cs:0.431578947368421,1.00385598922354);

\path [draw=crimson2143940, semithick]
(axis cs:0.468421052631579,0.966963510151341)
--(axis cs:0.468421052631579,1.00007315482459);

\path [draw=crimson2143940, semithick]
(axis cs:0.505263157894737,0.961412365519383)
--(axis cs:0.505263157894737,0.998774076742273);

\path [draw=crimson2143940, semithick]
(axis cs:0.542105263157895,0.953522064258418)
--(axis cs:0.542105263157895,0.995956905727004);

\path [draw=crimson2143940, semithick]
(axis cs:0.578947368421053,0.944475700202528)
--(axis cs:0.578947368421053,0.991269770410298);

\path [draw=crimson2143940, semithick]
(axis cs:0.615789473684211,0.937304498991616)
--(axis cs:0.615789473684211,0.984396338769543);

\path [draw=crimson2143940, semithick]
(axis cs:0.652631578947368,0.917630690080276)
--(axis cs:0.652631578947368,0.983461567566769);

\path [draw=crimson2143940, semithick]
(axis cs:0.689473684210526,0.894302573834734)
--(axis cs:0.689473684210526,0.977617526966999);

\path [draw=crimson2143940, semithick]
(axis cs:0.726315789473684,0.86735885914039)
--(axis cs:0.726315789473684,0.972278172608805);

\path [draw=crimson2143940, semithick]
(axis cs:0.763157894736842,0.834285462096437)
--(axis cs:0.763157894736842,0.958822676885786);

\path [draw=crimson2143940, semithick]
(axis cs:0.8,0.554957893704132)
--(axis cs:0.8,1.02017585944682);

\addplot [semithick, darkorange25512714, smooth, line width=1.3pt]
table {%
0.1 0.116698963477829
0.136842105263158 0.160199440570806
0.173684210526316 0.182580931232375
0.210526315789474 0.198595046193602
0.247368421052632 0.231864717579272
0.284210526315789 0.27613199438168
0.321052631578947 0.321787004956553
0.357894736842105 0.311438970015497
0.394736842105263 0.331982076629802
0.431578947368421 0.351718567982085
0.468421052631579 0.373276570381521
0.505263157894737 0.388340298146399
0.542105263157895 0.405228455826456
0.578947368421053 0.42058392580311
0.615789473684211 0.432717147786308
0.652631578947368 0.444563949507855
0.689473684210526 0.455683664561644
0.726315789473684 0.462959899890334
0.763157894736842 0.467225573482711
0.8 0.420293847901014
};
\addlegendentry{$\tilde\rho_{22}$}
\addplot [semithick, crimson2143940, smooth, line width=1.3pt]
table {%
0.1 0.996112046771833
0.136842105263158 1.00018292121506
0.173684210526316 0.998729452059284
0.210526315789474 0.994760728023519
0.247368421052632 0.992197800420943
0.284210526315789 0.99223425722013
0.321052631578947 0.991169640019949
0.357894736842105 0.990072922755794
0.394736842105263 0.988676026789899
0.431578947368421 0.986677063667401
0.468421052631579 0.983518332487965
0.505263157894737 0.980093221130828
0.542105263157895 0.974739484992711
0.578947368421053 0.967872735306413
0.615789473684211 0.96085041888058
0.652631578947368 0.950546128823523
0.689473684210526 0.935960050400866
0.726315789473684 0.919818515874597
0.763157894736842 0.896554069491112
0.8 0.787566876575478
};
\addlegendentry{$\tilde\rho_{22}$}
\end{groupplot}

\end{tikzpicture}